\documentclass{article}

\usepackage{arxiv}

\usepackage[utf8]{inputenc} % allow utf-8 input
\usepackage[T1]{fontenc}    % use 8-bit T1 fonts
\usepackage{hyperref}       % hyperlinks
\usepackage{url}            % simple URL typesetting
\usepackage{booktabs}       % professional-quality tables
\usepackage{amsfonts}       % blackboard math symbols
\usepackage{nicefrac}       % compact symbols for 1/2, etc.
\usepackage{microtype}      % microtypography
\usepackage{xcolor}         % colors

\usepackage{amsmath}
\usepackage{amssymb}
\usepackage{mathtools}
\usepackage{amsthm}
\usepackage{pifont}
\usepackage{multirow}
\usepackage{enumitem}
\usepackage{dashrule}
\usepackage{mathdots}
\usepackage{array}
\newcommand{\cmark}{\ding{51}} % ✓
\newcommand{\xmark}{\ding{55}} % ✗
\usepackage{wrapfig}
\usepackage{booktabs}
\usepackage[table]{xcolor}
\usepackage[numbers,sort&compress]{natbib}
\usepackage{amsmath}\allowdisplaybreaks
\usepackage{amsfonts,bm}
\usepackage{amssymb}
\usepackage{algorithm}
\usepackage{algorithmic}
\def\eqref#1{equation~(\ref{#1})}
\def\1{\bf{1}}

\def\vg{{\bf{g}}}

\def\vx{{\bf{x}}}
\def\vy{{\bf{y}}}
\def\vz{{\bf{z}}}

\def\BR{{\mathbb{R}}}

\usepackage{amsthm}
\theoremstyle{plain}
\makeatletter
\def\Ddots{\mathinner{\mkern1mu\raise\p@
\vbox{\kern7\p@\hbox{.}}\mkern2mu
\raise4\p@\hbox{.}\mkern2mu\raise7\p@\hbox{.}\mkern1mu}}
\makeatother

\makeatletter
\newcommand*{\rom}[1]{\expandafter\@slowromancap\romannumeral #1@}
\makeatother

\theoremstyle{plain}
\newtheorem{theorem}{Theorem}[section]

\newtheorem{lemma}[theorem]{Lemma}
\newtheorem{corollary}[theorem]{Corollary}
\theoremstyle{definition}

\newtheorem{assumption}[theorem]{Assumption}
\theoremstyle{remark}
\newtheorem{remark}[theorem]{Remark}

\title{Fully First-Order Algorithms for Online Non-Convex Bilevel Optimization}

\author{%
 Tingkai Jia \\
  East China Normal University\\
  Shanghai China \\
  \texttt{51275902086@stu.ecnu.edu.cn} \\
  \And
 Cheng Chen \\
  East China Normal University\\
  Shanghai China \\
  \texttt{chchen@sei.ecnu.edu.cn} \\
}

\begin{document}

\maketitle

\begin{abstract}
In this work, we study nonconvex–strongly convex online bilevel optimization (OBO) using only first-order oracle.
Existing OBO algorithms are mainly based on hypergradient descent, which requires access to a Hessian-vector product (HVP) oracle and potentially incurs high computational costs.
By reformulating the original OBO problem as a single-level online problem with inequality constraints and constructing a sequence of Lagrangian function, we eliminate the need for HVPs arising from implicit differentiation.
Specifically, we propose a fully first-order algorithm for OBO, and provide theoretical guarantees showing that it achieves regret of $O(1 + V_T + H_{2,T})$ with a total of $O(T\log T)$ iterations, where $V_T$ measures the variation in function values and $H_{2,T}$ characterizes the drift variation of the inner-level optimal solution.
We also establish a sublinear regret bound under the single-loop structure by introducing additional gradient-variation terms. 
Furthermore, we develop an improved variant with an adaptive inner-iteration scheme, which removes the dependence on $H_{2,T}$ and achieves regret of $O(\log T + V_T)$. 
Finally, under the stochastic OBO setting, we establish the regret bound for the fully first-order algorithm, i.e., $O(T^{2/3}(1 + \sigma^2) + V_T + H_{2,T})$. 
Numerical experiments demonstrate the feasibility of our algorithm and support our theoretical findings.
\end{abstract}

\section{Introduction}
Online bilevel optimization (OBO) has attracted growing attention in recent years. It extends bilevel learning to streaming settings, where the underlying data distribution may drift over time.
The goal is to generate a sequence of decisions $\{\vx_t\}_{t=1}^T$ that tracks the time-varying optimal solutions $\{\vx_t^*\}_{t=1}^T$ over the horizon $t \in [T]$ by solving 
\begin{align}
    &\min_{\vx \in \mathcal{X}}  F_t(\vx) := f_t \left(\vx, \vy_t^*(\vx)\right) \quad \text{s.t.} \quad \vy_t^*(\vx) = \mathop{\rm argmin}_{\vy \in \mathbb{R}^{d_2}} g_t (\vx, \vy), \tag{P1}\label{P1}
\end{align}
where the upper-level function $f_t:\mathcal{X} \times \mathbb{R}^{d_2} \to \mathbb{R}$ is smooth but possibly non-convex, the inner-level function $g_t:\mathcal{X} \times \mathbb{R}^{d_2} \to \mathbb{R}$ is smooth and $\mu_g$-strongly convex which ensures the existence of a unique minimizer $\vy_t^*(\vx)$ for all given $\vx\in\mathcal{X}\subset\mathbb{R}^{d_1}$.

Existing OBO methods \cite{tarzanagh2024online,lin2023non,bohne2024online,jia2026achieving} build on hypergradient descent, which requires an inner iterative procedure to approximately solve the inner-level problem, and then update the upper decision $\vx$ using the hypergradient $\nabla F_t(\vx)$. Achieving a sufficiently accurate hypergradient typically requires many Hessian-vector products (HVP) to control the approximation error \cite{pmlr-v48-pedregosa16,gould2016differentiating,ghadimi2018approximation,maclaurin2015gradient,franceschi2017forward,shaban2019truncated,mackay2019self}. These algorithms may become impractical or even inapplicable when HVP queries are unavailable or prohibitively expensive \cite{sow2022convergence,song2019maml,nichol2018first}. This work instead studies OBO that relies solely on first-order gradients of the inner- and upper-level objectives to solve the above problem.

\begin{table*}[t]
% \captionsetup{skip=2pt}
\caption{Comparison of algorithms on bilevel local regret without window averaging as defined in Eq.(\ref{eq:reg_F}) in terms of gradient queries, required Hessian-vector product (HVP) queries, fully single-loop structure, local regret bounds, and corresponding theorems. 
$V_T$ and $H_{2,T}$ are defined in Eq.(\ref{pathV}). $E_{2,T}$ and $P_T$ are defined in Eq.(\ref{eq:E_yy_g}) and Theorem~\ref{thm:SF$^2$OBO}, respectively. 
% Theorems C.16 and C.19 are in the appendix of \citet{jia2026achieving}.
}
\label{tab:1}
\vskip -0.2in
\begin{center}
\begin{small}

\begin{tabular}{ccccc}
\toprule
\textbf{Algorithm}
& \begin{tabular}{c} \textbf{Grad.} \\ \textbf{Query} \end{tabular}
& \begin{tabular}{c} \textbf{HVP} \\ \textbf{Query} \end{tabular}
& \begin{tabular}{c} \textbf{Fully} \\ \textbf{S-Loop} \end{tabular}
& \begin{tabular}{c} \textbf{Local Regret} \\ \textbf{Bound} \end{tabular} \\
\midrule
SOBOW~\cite{lin2023non} 
& $O(T)$ 
& $O(T^2)$ 
& \xmark
& $O(1 + V_T + H_{2,T})$ \\
\midrule
OBBO~\cite{bohne2024online} 
& $O(T\log T)$ 
& $O(T\log T)$ 
& \xmark
& $O(1 + V_T + H_{2,T})$ \\
\midrule
FSOBO~\cite{jia2026achieving} 
& $O(T)$ 
& $O(T)$ 
& \cmark
& $O(1 + V_T + H_{2,T} + E_{2,T})$ \\
\midrule
AOBO~\cite{jia2026achieving} 
& $O(T\log T)$ 
& $O(T\log T)$ 
& \xmark
& $O(1 + V_T)$ \\
\midrule
\rowcolor{blue!8}
Algorithm~\ref{alg:F$^2$OBO} 
& $O(T\log T)$ 
& None 
& \xmark
& $O(1 + V_T + H_{2,T})$ \\
\midrule
\rowcolor{blue!8}
Algorithm~\ref{alg:F$^2$OBO} 
& $O(T)$ 
& None 
& \cmark
& $O(T^{1/3}(1+V_T) + T^{2/3}P_T)$ \\
\midrule
\rowcolor{blue!8}
Algorithm~\ref{alg:AF$^2$OBO} 
& $O(T^2 + TH_{2,T})$ 
& None 
& \xmark
& $O(\log T + V_T)$ \\
\bottomrule
\end{tabular}
\end{small}
\end{center}
\vskip -0.2in
\end{table*}

\begin{table*}[t]
% \captionsetup{skip=2pt}
\caption{Comparison of stochastic algorithms on bilevel local regret in terms of sample complexities, required HVP term numbers, window-averaged regret form, local regret bounds, and corresponding theorems. In SOGD, $\Delta_T = O(1) + V_T$, $\Psi_T = P_T + E_{\vx,T}^f + E_{\vx\vy,T}^g + E_{\vy\vy,T}^g$, where $E_{\vx,T}^f:=\sum_{t=2}^T\sup_\vz\|\nabla_\vx f_t(\vz) {-} \nabla_\vx f_{t-1}(\vz)\|^2$, $E_{\vx\vy,T}^g:=\sum_{t=2}^T\sup_\vz\|\nabla_{\vx\vy}^2g_t(\vz) {-} \nabla_{\vx\vy}^2g_{t-1}(\vz)\|^2$.}
\label{tab:2}
\vskip -0.2in
\begin{center}
\begin{small}

\begin{tabular}{ccccc}
\toprule
\textbf{Algorithm} 
& \begin{tabular}{c} \textbf{Sample} \\ \textbf{Comp.} \end{tabular}
& \begin{tabular}{c} \textbf{HVP} \\ \textbf{Query} \end{tabular}
& \textbf{WinAvg.}
& \begin{tabular}{c} \textbf{Local Regret} \\ \textbf{Bound} \end{tabular} \\
\midrule
SOBBO~\cite{bohne2024online} 
& $O(wT)$ 
& $O(T\log w)$ 
& \cmark
& $O(\frac{T}{w}(1+\sigma^2) + V_T + H_{2,T})$ \\
\midrule
SOGD~\cite{nazari2025stochastic} 
& $O(T)$ 
& $O(T)$ 
& \xmark
& $O(T^{1/3}(\sigma^2+\Delta_T) + T^{2/3}\Psi_T)$ \\
\midrule
\rowcolor{blue!8}
Algorithm~\ref{alg:StochasF$^2$OBO} 
& $O(T^{5/3}\log T)$ 
& None 
& \xmark
& $O(T^{2/3}(1+\sigma^2) + V_T + H_{2,T})$ \\
\bottomrule
\end{tabular}
\end{small}
\end{center}
\vskip -0.2in
\end{table*}

\subsection{Our Contributions}

We propose a fully first-order online bilevel optimizer (F$^2$OBO), presented in Algorithm~\ref{alg:F$^2$OBO}. Theoretical analysis shows that it can achieve an $O(1 + V_T + H_{2,T})$ upper bound on the standard bilevel local regret, without requiring any second-order information. Moreover, our method supports fully single-loop structure, which performs only one inner-step update per $t$. By introducing additional regularities of environmental variation, we can still guarantee sublinear regret, i.e., $O(T^{1/3}(1+V_T) + T^{2/3}P_T)$.

Building on the adaptive-iteration strategy in AOBO \cite{jia2026achieving}, we further enhance F$^2$OBO in a similar spirit to propose adaptive-iteration fully first-order online bilevel optimizer (AF$^2$OBO), presented in Algorithm~\ref{alg:AF$^2$OBO}. 
Theoretical analysis shows that it eliminates the dependence on the drift variation $H_{2,T}$ of the inner-level optimal solution. Although this comes at the cost of a larger iteration overhead and a weaker regret guarantee, i.e., $O(\log T + V_T)$ with an additional fixed $O(\log T)$ term, our method still ensures sublinear regret even when $\vy_t^*(x)$ drifts drastically (e.g., $H_{2,T}=O(T)$ or larger). This robustness is significant compared with static inner-iteration methods.

Finally, we extend F$^2$OBO to the stochastic OBO problem and propose stochastic fully first-order online bilevel optimizer (SF$^2$OBO), presented in Algorithm~\ref{alg:StochasF$^2$OBO}. Compared with SOBBO, we reduce the variance of the biased gradient by appropriately increasing the batch size, thereby eliminating the reliance on window-averaged regret and obtaining $O(T^{2/3}(1+\sigma^2) + V_T + H_{2,T})$. Compared with SOGD, our result imposes a milder requirement on environmental variation, i.e., when $V_T=H_{2,T}=O(T^{2/3})$, SF$^2$OBO still achieves sublinear regret, whereas SOGD suffers at least $O(T^{4/3})$ superlinear regret.

\subsection{Related Work}

\textbf{Bilevel Optimization} Bilevel optimization (BO) was originally introduced by \citet{bracken1973mathematical}. Existing BO methods mainly include hypergradient-based methods, penalty-based methods, and Hessian- or Jacobian-free methods. 
Hypergradient-based methods are typically divided into approximate implicit differentiation (AID) \cite{pmlr-v48-pedregosa16,gould2016differentiating,ghadimi2018approximation,grazzi2020iteration,ji2021bilevel,pedregosa2016hyperparameter} and iterative differentiation (ITD) approaches \cite{maclaurin2015gradient,franceschi2017forward,shaban2019truncated,mackay2019self,ji2021bilevel}. 
Penalty-based methods reformulate BO as a constrained single-level problem using inner-level optimality conditions, thereby avoiding implicit-gradient computation and allowing extensions to non-strongly convex inner problems under suitable constraint qualification conditions \cite{shen2023penalty,liu2022bome,huang2023momentum,arbel2022non,liu2021towards,chen2024finding}. 
Hessian- or Jacobian-free methods further avoid explicit second-order derivatives by estimating hypergradients with zeroth-order approximations \cite{vuorio2019multimodal, ji2022will, yang2023achieving}.

\textbf{Online Bilevel Optimization.}
Online bilevel optimization (OBO) has attracted growing attention in recent years.
Early studies by \citet{tarzanagh2024online} and \citet{lin2023non} initiated the investigation of OBO and developed a window-averaged hypergradient framework. 
Building on this line of work, \citet{bohne2024online} employed iterative differentiation (ITD) to compute hypergradients and proposed the OBBO algorithm, together with its stochastic counterpart SOBBO. 
Later, \citet{nazari2025stochastic} introduced a momentum-based search direction and developed SOGD for the stochastic OBO setting. 
More recently, \citet{jia2026achieving} revisited bilevel local regret without window averaging, proposed the first adaptive-inner-iteration algorithm, and established optimal guarantees. 
They further introduced a new window-averaging analysis framework that removes the need for sublinear variation assumptions. 
In addition, \citet{bohne2026non} studied non-stationary functional bilevel optimization, which can also be interpreted within the existing OBO framework.

\section{Preliminaries}\label{sec:2}

\subsection{Notations and Assumptions}\label{sec:2.1}

% Let $\|\cdot\|$ denote the $\ell_2$-norm of the vector. 
% For the inner-level function in problem (\ref{P1}), we denote the partial derivatives of $g_t(\vx,\vy)$ with respect to $\vx$ and $\vy$ by $\nabla_\vx g_t(\vx, \vy)$ and $\nabla_\vy g_t(\vx, \vy)$, respectively. 
% We also use $\nabla_{\vx\vy}^2 g_t(\vx, \vy)$ and $\nabla_{\vy\vy}^2 g_t(\vx, \vy)$ to denote the Jacobian of $\nabla_{\vy} g_t(\vx, \vy)$ with respect to $\vx$ and the Hessian of $g_t(\vx,\vy)$ with respect to $\vy$. 
% For the upper-level function in problem (\ref{P1}), we denote the total derivative of $f_t(\vx, \vy_t^*(\vx))$ by $\nabla f_t(\vx, \vy_t^*(\vx))$. $\kappa_g := L_{g,1}/\mu_g$ is used to denote the condition number of $g_t(\vx, \vy)$ with respect to $\vy$.
We use $\|\cdot\|$ to denote the Euclidean norm. 
Given the inner-level function $g_t$ in problem~(\ref{P1}), its partial gradients with respect to the upper- and inner-level variables are denoted by 
$\nabla_{\vx} g_t(\vx,\vy)$ and $\nabla_{\vy} g_t(\vx,\vy)$. 
The mixed derivative $\nabla_{\vx\vy}^2 g_t(\vx,\vy)$ is defined as the Jacobian of $\nabla_{\vy} g_t(\vx,\vy)$ with respect to $\vx$, while 
$\nabla_{\vy\vy}^2 g_t(\vx,\vy)$ denotes the Hessian with respect to $\vy$. 
For the composition $f_t(\vx,\vy_t^*(\vx))$, we use $\nabla f_t(\vx,\vy_t^*(\vx))$ to denote its total derivative with respect to $\vx$. 
Finally, the condition number of the inner-level objective with respect to $\vy$ is denoted by 
$\kappa_g := L_{g,1}/\mu_g$.
% We impose the following Lipschitz continuous and convex assumption for problem (\ref{P1}) as follows:
% \begin{assumption}\label{asm1}
%     For all $t\in[T]$ and given $\vx\in\mathcal{X}$, we suppose $g_t(\vx, \vy)$ is $\mu_g$-strongly convex in $\vy$.
% \end{assumption}

% \begin{assumption}\label{asm2}
%     For all $t\in[T]$, we suppose: (i) $f_t(\vx,\vy)$ is $L_{f,0}$-Lipschitz continuous and $\nabla f_t(\vx, \vy)$ is $L_{f,1}$-Lipschitz continuous; (ii) $\nabla g_t(\vx, \vy)$ is $L_{g,1}$-Lipschitz continuous; (iii) $\nabla_{\vx\vy}^2 g_t(\vx,\vy)$ and $\nabla_{\vy\vy}^2 g_t(\vx,\vy)$ are $L_{g,2}$-Lipschitz continuous.
% \end{assumption}

% We also impose the lower bound assumption on the upper-level objective function value, as commonly used in online non-convex optimization.
% \begin{assumption}\label{asm3}
%     For all $t\in[T]$, we suppose $f_t(\vx, \vy)$ is lower bounded, i.e., $\inf_{\vx\in\mathcal{X},\vy\in\mathbb{R}^{d_2}}f_t(\vx, \vy) > -\infty$.
% \end{assumption}
% \begin{assumption}\label{asm3}
%     For all $t\in[T]$, there exists a positive constant $M$ such that $\left| f_t(\vx, \vy_t^*(\vx)) \right| \leq M$ for any $\vx\in\mathcal{X}$.
% \end{assumption}
We make the following assumptions on the upper- and inner-level functions. 
These assumptions are also commonly used in existing OBO studies \cite{lin2023non,nazari2025stochastic,bohne2024online} and are standard in this setting. 
% For the upper-level function $f_t$, we impose the following assumption.
\begin{assumption}\label{asm1}
For all $t\in[T]$, $f_t$ is $L_{f,0}$-Lipschitz continuous, and $\nabla f_t(\vx,\vy)$ is $L_{f,1}$-Lipschitz continuous. 
Moreover, there exists a positive constant $M$ such that $\left| f_t(\vx,\vy_t^*(\vx)) \right| \leq M, \quad \forall \vx\in\mathcal{X}$.
\end{assumption}
% For the inner-level function $g_t$, we impose the following assumption.
\begin{assumption}\label{asm2}
For all $t\in[T]$, $g_t$ satisfies the following conditions. 
For any given $\vx\in\mathcal{X}$, $g_t(\vx,\vy)$ is $\mu_g$-strongly convex in $\vy$. 
In addition, $\nabla g_t(\vx,\vy)$ is $L_{g,1}$-Lipschitz continuous, and both $\nabla_{\vx\vy}^2 g_t(\vx,\vy)$ and $\nabla_{\vy\vy}^2 g_t(\vx,\vy)$ are $L_{g,2}$-Lipschitz continuous.
\end{assumption}

For stochastic OBO, we additionally give the following assumption to bound the variance of inner- and upper-level first-order gradients, which is common for stochastic regret analysis.
\begin{assumption}\label{asm:stochas}
For all $t\in[T]$, suppose there exist $\sigma>0$ such that
\begin{align*}
    \mathbb{E}\left[\left\|\nabla f_t(\vx,\vy;\xi) - \nabla f_t(\vx,\vy)\right\|^2\right] \leq \sigma^2, \quad
    \mathbb{E}\left[\left\|\nabla g_t(\vx,\vy;\zeta) - \nabla g_t(\vx,\vy)\right\|^2\right] \leq \sigma^2.
\end{align*}
\end{assumption}

% \subsection{Online Bilevel Optimization}
\subsection{Single-Level Reformulation of OBO Problem}\label{sec:2.3}

For non-convex-strongly-convex OBO, the uniqueness of the inner-level solution $\vy_t^*(\vx)$ for any $\vx\in\mathcal{X}$ yields a unique closed-form expression for the hypergradient of $F_t$, given by
\begin{align}
    \nabla F_t(\vx) =& \nabla_\vx f_t(\vx, \vy_t^*(\vx)) + \nabla\vy_t^*(\vx)\nabla_\vy f_t(\vx, \vy_t^*(\vx)) \nonumber\\
    =& \nabla_\vx f_t(\vx, \vy_t^*(\vx)) - \nabla_{\vx\vy}^2g_t(\vx, \vy_t^*(\vx))\left[\nabla_{\vy\vy}^2g_t(\vx, \vy_t^*(\vx))\right]^{-1}\nabla_\vy f_t(\vx, \vy_t^*(\vx)). \label{eq:hypergrad}
\end{align}
% We consider the \textit{bilevel local regret} without window averaging \cite{nazari2025stochastic} as the evaluation criterion for the algorithm, which can accurately reflect the system performance when functions change rapidly:
We consider the following local regret notion as the evaluation criterion for the algorithms:
\begin{align}
    \mathrm{BLReg} (T) = \sum_{t=1}^T \left\| \mathcal{G}_\mathcal{X}\left( \vx_t, \nabla F_t(\vx_t), \gamma \right) \right\|^2. \label{eq:reg_F}
\end{align}
Here, $\mathcal{G}_{\mathcal{X}}$ denotes the gradient mapping with respect to the constraint set $\mathcal{X}$, defined as
\begin{align}
    &\mathcal{G}_\mathcal{X}\left( \vx, \vg, \gamma \right) := \frac{1}{\gamma}\left(\vx - \mathop{\rm argmin}_{\mathbf{u}\in\mathcal{X}}\left\{\left\langle \vg, \mathbf{u} \right\rangle + \frac{1}{2\gamma}\|\mathbf{u} - \vx\|^2\right\}\right),
\end{align}
for given $\vg \in \mathbb{R}^{d_1}$ and $\gamma>0$. 
We mainly use the following two variation regularities to establish our regret bound, which are also used in \citet{lin2023non,bohne2024online} and \citet{nazari2025stochastic}:
\begin{align}\label{pathV}
V_T := \sum_{t=2}^{T} \sup_{\vx\in\mathcal{X}}\!\left| F_{t-1}(\vx) - F_t(\vx) \right| \quad \text{and} \quad H_{2,T} := \sum_{t=2}^T \sup_{\vx\in\mathcal{X}} \left\|\vy_{t-1}^*(\vx) - \vy_t^*(\vx)\right\|^2.
\end{align}

Following previous works~\cite{kwon2023fully,sow2022primal,liu2022bome,chen2025near}, the original problem (\ref{P1}) can be equivalently written as an inequality constrained single-level problem, shown as:
\begin{align}
    \min_{\vx\in\mathcal{X}, \vy\in\mathbb{R}^{d_2}}f_t(\vx, \vy) \quad &\text{s.t.} \quad g_t(\vx, \vy) - g_t(\vx, \vy_t^*(\vx)) \leq 0. \tag{P2}\label{P2}
\end{align}
By treating the constraint as a penalty term, we construct the following Lagrangian function $\mathcal{L}_{t}$ for some $\lambda_t > 0$:
\begin{align}
    \mathcal{L}_{t}(\vx, \vy, \lambda_t) = f_t(\vx, \vy) + \lambda_t\left(g_t(\vx, \vy) - g_t(\vx, \vy_t^*(\vx))\right). \label{eq:L_t}
\end{align}
We can finally approximately transform problem (\ref{P2}) into an min-max optimization problem for $\mathcal{L}_t^*(\vx)$ as follows:
\begin{align}
    \min_{\vx\in\mathcal{X}} \mathcal{L}_t^*(\vx) := \mathcal{L}_{t}(\vx, \vy_{\lambda,t}^*(\vx), \lambda_t) \quad \text{s.t.} \quad \vy_{\lambda_t,t}^*(\vx) = \rm argmin_{\vy\in\mathbb{R}^{d_2}}\mathcal{L}_{t}(\vx, \vy, \lambda_t). \tag{P3}\label{P3}
\end{align}
Here the multiplier $\lambda_t$ represents the degree to which the solution to (\ref{P3}) conforms to the constraint, and we have $\lim_{\lambda_t\rightarrow\infty} \mathcal{L}^*_t(\vx) = f_t(\vx, \vy_t^*(\vx))$.

Based on the above, we approximate the OBO problem by an online min-max optimization problem. The ingenuity of this approach lies in its elimination of the requirement for a second-order gradient oracle arising from the implicit term $\nabla\vy_t^*(\vx)$ in Eq.(\ref{eq:hypergrad}). Specifically, we have
\begin{align}
    \nabla\mathcal{L}_t^*(\vx) =& \nabla_\vx\mathcal{L}_t(\vx, \vy_{\lambda_t,t}^*(\vx), \lambda_t) + \nabla\vy_{\lambda_t,t}^*(\vx)\nabla_\vy\mathcal{L}_t(\vx, \vy_{\lambda_t,t}^*(\vx), \lambda_t) \nonumber\\
    =& \nabla_\vx f_t(\vx, \vy_{\lambda_t,t}^*(\vx)) + \lambda_t\left(\nabla_\vx g_t(\vx, \vy_{\lambda_t,t}^*(\vx)) - \nabla_\vx g_t(\vx, \vy_t^*(\vx))\right), \label{eq:nabla_L}
\end{align}
note that $\nabla_\vy\mathcal{L}_t(\vx, \vy_{\lambda_t,t}^*(\vx), \lambda_t) = 0$ holds for all $t\in[T]$ and any $\lambda_t > 0$, thus the hypergradient does not require the calculation of $\nabla\vy_{\lambda_t,t}^*(\vx)$.

\section{A Fully First-Order Algorithm for Online Bilevel Optimization}\label{sec:3}

This section develops and analyzes a fully first-order algorithm for determistic OBO. 
% In Section~\ref{sec:3.1}, we design an algorithm for the approximate problem (\ref{P3}) that avoids implicit gradient computation and thus any second-order information. In Section~\ref{sec:3.2}, we provide a regret analysis showing that, with an appropriate multiplier sequence $\{\lambda_t\}_{t=1}^T$, the algorithm achieves sublinear regret for both the original problem (\ref{P1}) and (\ref{P3}).

\subsection{Fully First-order Online Bilevel Optimizer (F$^2$OBO)}\label{sec:3.1}

Similar to existing OBO algorithms \cite{tarzanagh2024online,lin2023non,bohne2024online}, we still attempt to apply online gradient descent (OGD) to problem (\ref{P3}). Consider a sequence $\{\lambda_t\}_{t=1}^T$ that governs our approximation, the hypergradient $\nabla\mathcal{L}_t^*(\vx_t)$ is calculated at the current iterate $\vx_t$ and used to update $\vx_{t+1}$ at each time step $t$. Therefore, according to Eq.(\ref{eq:nabla_L}), we need to solve two subproblems to obtain $\vy_t^*(\vx_t)$ and $\vy_{\lambda_t,t}^*(\vx_t)$.
% which arise from
% \begin{align*}
%     \vy_t^*(\vx_t) = \mathop{\rm argmin}_{\vy\in\mathbb{R}^{d_2}}g_t(\vx_t, \vy) \quad \text{and} \quad \vy_{\lambda_t,t}^*(\vx_t) = \mathop{\rm argmin}_{\vy\in\mathbb{R}^{d_2}}\mathcal{L}_t(\vx_t, \vy, \lambda_t).
% \end{align*}
We can solve both subproblems simultaneously via a nested inner loop of gradient descent, starting from $\vz_t^1 = \vz_t$ and $\vy_t^1 = \vy_t$ for $K$ times, as shown below:
\begin{align}
    \vz_t^{k+1} \leftarrow \vz_t^k - \alpha\nabla_\vy g_t(\vx_t, \vz_t^k) \quad \text{and} \quad \vy_t^{k+1} \leftarrow \vy_t^k - \beta\nabla_\vy \mathcal{L}_t(\vx_t, \vy_t^k, \lambda_t). \label{eq:y_update}
\end{align}
This simple strategy yields excellent theoretical performance because both $g_t$ and $\mathcal{L}_t$ is strongly convex and smooth with respect to $\vy$ under Assumptions~\ref{asm1}, \ref{asm2} and if set $\lambda_t\geq2L_{f,1} / \mu_g$ for all $t\in[T]$, which can also be seen in several offline works \cite{kwon2023fully,chen2025near}. Finally with $\vz_{t+1} = \vz_t^{K+1}$ and $\vy_{t+1} = \vy_t^{K+1}$, the approximated hypergradient $\widetilde{\nabla}\mathcal{L}_t^*(\vx_t)$ can be calculated by
% to update for step size $\gamma>0$ by equation (\ref{eq:x_update}):
% \begin{align*}
%     &\vx_{t+1} \leftarrow \mathop{\rm argmin}_{\mathbf{u}\in\mathcal{X}}\left\{\left\langle\widetilde{\nabla}\mathcal{L}_t^*(\vx_t), \mathbf{u}\right\rangle + \frac{1}{2\gamma}\left\|\mathbf{u} - \vx_t\right\|^2\right\}, \quad \text{where} \\
%     &\widetilde\nabla \mathcal{L}_t^*(\vx_t) := \nabla_\vx f_t(\vx_t, \vy_{t+1}) + \lambda_t\left(\nabla_\vx g_t(\vx_t, \vy_{t+1}) - \nabla_\vx g_t(\vx_t, \vz_{t+1})\right).
% \end{align*}
\begin{align*}
    \widetilde\nabla \mathcal{L}_t^*(\vx_t) := \nabla_\vx f_t(\vx_t, \vy_{t+1}) + \lambda_t\left(\nabla_\vx g_t(\vx_t, \vy_{t+1}) - \nabla_\vx g_t(\vx_t, \vz_{t+1})\right).
\end{align*}
% The multiplier $\lambda_t$ at each $t$ needs to be chosen carefully: if it is too small, the solution of (\ref{P3}) may deviate substantially from those of (\ref{P1}), and the resulting accumulated bias can easily destroy the sublinear regret guarantee; if it is too large, the gradients in (\ref{eq:y_update}) can become overly large, making the optimization unstable and difficult to control \cite{kwon2023fully}. To address this issue, we set $\lambda_{t+1} = \Lambda(t, \lambda_t)$ with specifically designed function $\Lambda$ to generate a non-decreasing sequence, and our final Fully First-order Online Bilevel Optimizer (F$^2$OBO) is presented in Algorithm~\ref{alg:F$^2$OBO}.
The multiplier $\lambda_t$ must be carefully chosen: a small $\lambda_t$ may cause the solution of (\ref{P3}) to deviate from that of (\ref{P1}), leading to accumulated bias, while a large $\lambda_t$ may induce overly large gradients in Eq.(\ref{eq:y_update}) and destabilize optimization \cite{kwon2023fully}. To balance this trade-off, we update $\lambda_{t+1}=\Lambda(t,\lambda_t)$ via a designed non-decreasing schedule, yielding the final Fully First-order Online Bilevel Optimizer (F$^2$OBO) in Algorithm~\ref{alg:F$^2$OBO}. The following theorem gives a regret upper bound of F$^2$OBO.

\begin{algorithm}[tb]
    % \caption{First-order Penalty-based Online Bilevel Optimizer (F$^2$OBO)}
    \caption{Fully First-order Online Bilevel Optimizer (F$^2$OBO)}
    \label{alg:F$^2$OBO}
    \textbf{Input}: $\vx_1$, $\vy_1$, $\vz_1$, $\{\lambda_t, \beta_t, \gamma_t, K_t\}_{t=1}^T$, $\alpha$
    
    \begin{algorithmic}[1] %[1] enables line numbers
        \FOR{\(t = 1\) to \(T\)}
        \STATE Output $\vx_t$ and $\vy_t$, receive $f_t$ and $g_t$.
        \STATE warm start $\vy_t^1 \leftarrow \vy_t$, $\vz_t^1 \leftarrow \vz_t$
        \FOR{\(k = 1\) to \(K_t\)}
        \STATE \(
        \vz_t^{k+1} \leftarrow \vz_t^k - \alpha\nabla_\vy g_t(\vx_t, \vz_t^k)
        \)
        \STATE \( \vy_t^{k+1} \leftarrow \vy_t^k - \beta_t\left(\nabla_\vy f_t(\vx_t, \vy_t^k) + \lambda_t\nabla_\vy g_t(\vx_t, \vy_t^k)\right) \)
        \ENDFOR
        \STATE update $\vy_{t+1} \leftarrow \vy_t^{K_t+1}$, $\vz_{t+1} \leftarrow \vz_t^{K_t+1}$
        \STATE calculate \( 
        \widetilde\nabla \mathcal{L}_t^*(\vx_t) = \nabla_\vx f_t(\vx_t, \vy_{t+1}) + \lambda_t\left(\nabla_\vx g_t(\vx_t, \vy_{t+1}) - \nabla_\vx g_t(\vx_t, \vz_{t+1})\right)
        \)
        \STATE \( 
        \vx_{t+1} \leftarrow \mathop{\rm argmin}_{\mathbf{u}\in\mathcal{X}}\left\{\left\langle\widetilde{\nabla}\mathcal{L}_t^*(\vx_t), \mathbf{u}\right\rangle + \frac{1}{2\gamma_t}\left\|\mathbf{u} - \vx_t\right\|^2\right\}
        \)
        \STATE \(
        \lambda_{t+1} \leftarrow \Lambda(t, \lambda_t)
        \)
        \ENDFOR
    \end{algorithmic}
\end{algorithm}

\begin{theorem}\label{thm:F$^2$OBO}
    Under Assumptions~\ref{asm1},~\ref{asm2}, let $\Lambda(t, \lambda_t) = \left(1+\frac{1}{t}\right)^{\tau}\lambda_t$ for some $\tau>0$, $\lambda_1 > \frac{2L_{f,1}}{\mu_g}$, $\alpha\leq \frac{1}{L_{g,1}}$, $\beta_t \leq \frac{1}{2\lambda_tL_{g,1}}$, $\gamma_t \equiv \gamma \leq \min\{\frac{1}{2L_F}, \frac{1}{16\kappa_g}\sqrt{\frac{1-\rho}{C_\lambda}}\}$ and $K_t \equiv K \geq \max\{1 - \frac{\ln2}{\ln\rho}, \frac{\ln c - 2\tau\ln T}{\ln\rho}\}$ for some constants $\rho\in(0,1)$, $C_\lambda$ and $c > 0$, Algorithm~\ref{alg:F$^2$OBO} can guarantee
    \begin{align*}
        \mathrm{BLReg}(T) \leq O\Big(\sum_{t=1}^T \frac{1}{t^{2\tau}} + V_T + H_{2,T}\Big).
    \end{align*}
    % \begin{equation}\label{pathV}
    % \begin{gathered}
    %     \mathrm{BLReg}(T) 
    %     \leq O\Big(\sum_{t=1}^T \frac{1}{t^{2\tau}} + V_T + H_{2,T}\Big), \quad \text{where}\\
    %     V_T := \sum_{t=2}^{T} \sup_{\vx\in\mathcal{X}}
    %     \left| F_{t-1}(\vx) - F_t(\vx) \right|,
    %     \quad
    %     H_{2,T} := \sum_{t=2}^T \sup_{\vx\in\mathcal{X}} 
    %     \left\|\vy_{t-1}^*(\vx) - \vy_t^*(\vx)\right\|^2.
    % \end{gathered}
    % \end{equation}
\end{theorem}
Theorem~\ref{thm:F$^2$OBO} shows that our regret bound mainly depends on two variation regularities, 
$V_T$ and $H_{2,T}$. 
% These quantities have also been used in existing OBO studies, such as \citet{lin2023non,bohne2024online} and \citet{nazari2025stochastic}.
Specifically, if we set $\tau > 1/2$, then we achieve a regret of $O(1 + V_T + H_{2,T})$, which is comparable to SOBOW \cite{lin2023non} and OBBO \cite{bohne2024online}. Our algorithm requires a time complexity of $O(T\log T)$, which is on par with OBBO and better than that of SOBOW. 
% Moreover, if the algorithm operates in a static environment with $V_T=H_{2,T}=0$, we obtain $\mathrm{BLReg}(T) = O(1)$, leading to the following result: 
% \begin{align*}
%     \min_{t\in[1,T]}\left\|\mathcal{G}_\mathcal{X}(\vx_t, \nabla F_t(\vx_t),\gamma)\right\| \leq O\left(\frac{1}{\sqrt{T}}\right).
% \end{align*}
% This implies that we achieve an $O\left(\epsilon^{-2}\right)$ complexity, which matches the optimal first-order convergence rate for single-loop nonconvex optimization \cite{CarmonDuchiHinderSidford2017,CarmonDuchiHinderSidford2020}, and aligns with the best-known oracle complexity of the first-order method for offline deterministic bilevel optimization as shown in \cite{chen2025near}.

\subsection{Single-Loop F$^2$OBO}

Although sufficiently many inner-loop iterations can yield favorable theoretical guarantees, it may be difficult to deploy in practice due to the substantial computational overhead.
% which is a particularly critical issue in online data-stream settings where rapid adaptation to change is required.
Here we proof that F$^2$OBO still achieves sublinear regret under the single-loop structure, i.e., $K_t=1$. 
To control the accumulated error aused by inexactly solving the inner subproblems, we introduce the following two variation regularities to further restrict the first-order gradient changes:
\begin{align*}
    E_{\vy,T}^f {=} \sum_{t=2}^T\sup_{\vx,\vy}\left\|\nabla_\vy f_t(\vx, \vy) {-} \nabla_\vy f_{t-1}(\vx, \vy)\right\|^2, \,\,\,\, E_{\vy,T}^g {=} \sum_{t=2}^T\sup_{\vx, \vy}\left\|\nabla_\vy g_t(\vx, \vy) {-} \nabla_\vy g_{t-1}(\vx, \vy)\right\|^2.
\end{align*}
We believe that these additional terms are necessary in this case. Similar ideas have been adopted in the analyses of existing second-order single-loop methods, such as FSOBO \cite{jia2026achieving} and SOGD \cite{nazari2025stochastic}.

\begin{theorem}\label{thm:SF$^2$OBO}
    Under Assumptions~\ref{asm1},~\ref{asm2}, let $\Lambda(t, \lambda_t) = \left(1+\frac{1}{t}\right)^{\tau}\lambda_t$ for some $\tau\in(0,\frac{1}{2})$, $\lambda_1 > \frac{2L_{f,1}}{\mu_g}$, $\alpha\leq \frac{1}{L_{g,1}}$, $\beta_t = \frac{1}{2\lambda_tL_{g,1}}$, $\gamma_t < \min\{\frac{1}{2L_F}, \frac{1}{8\lambda_1t^\tau}\sqrt{\frac{1-\rho}{3C_g}}\}$ and $K_t\equiv1$ for some constant $C_g$, with $P_T := H_{2,T} + E_{\vy,T}^f + E_{\vy,T}^g$, Algorithm~\ref{alg:F$^2$OBO} can guarantee
    \begin{align*}
        \mathrm{BLReg}(T) 
        \leq O\left(T^\tau(1 + V_T) + T^{1-2\tau} + T^{2\tau}P_T\right).
    \end{align*}
\end{theorem}
Thoerem~\ref{thm:SF$^2$OBO} shows that SF$^2$OBO enjoys a sublinear regret guarantee with a certain degree of robustness: namely, under the assumption that $V_T$ and $P_T$ are all $o(T)$, there must exist some $\tau$ such that the regret satisfies $\mathrm{BLReg}(T) = o(T)$. In particular, by setting $\tau = 1/3$, we obtain
\begin{align}
    \mathrm{BLReg}(T) \leq O\left(T^{1/3}(1+V_T) + T^{2/3}P_T\right), \label{eq:sl_bound}
\end{align}
a sublinear regret guarantee holds when $V_T = o(T^{2/3})$ and $P_T = o(T^{1/3})$.

% \subsection{Discussion}\label{sec:4.2}

\begin{remark}
Compared with the regret bound of FSOBO as shown in Table~\ref{tab:1}, Eq.(\ref{eq:sl_bound}) incurs an additional $T$-dependent factor, leading to a weaker guarantee. This degradation stems from the cross-problem error induced by fully first-order oracle feedback. In the single-loop setting, this error is excessively amplified, even with additional variation regularities to control it, we believe it is difficult to eliminate entirely. However, existing second-order single-loop algorithms typically require assuming sublinear variation of the Hessian or Jacobian of $g_t$. For example, FSOBO assumes
\begin{align}
    E_{2,T} = E_{\vy,T}^f + E_{\mathbf{yy},T}^g, \quad \text{where} \quad
    E_{\mathbf{yy},T}^g = \sum_{t=2}^T\sup_{\vx, \vy}\left\|\nabla_{\mathbf{yy}}^2g_t(\vx, \vy) - \nabla_{\mathbf{yy}}^2g_{t-1}(\vx, \vy)\right\|^2, \label{eq:E_yy_g}
\end{align}
is sublinear. Such assumptions is harder to justify and guarantee. In contrast, our first-order algorithm only requires a sublinear first-order gradient changes, which is easier to verify and more practical.
\end{remark}

\section{An Adaptive Variant of F$^2$OBO}\label{sec:5}

In this section, we consider improving F$^2$OBO by applying an adaptive inner iteration approach \cite{jia2026achieving} to improve the regret upper bound. 
% Its follow-the-leader-based online optimization approach \cite{hazan2017efficient,huang2023online} has also been used in  to obtain the optimal regret upper bound in the hypergradient-based OBO algorithms.

\subsection{Adaptive-iteration Fully First-order Online Bilevel Optimizer (AF$^2$OBO)}\label{sec:5.1}

To summarize, F$^2$OBO solves the two inner subproblems using a fixed number of iterations at each time step $t$. This static strategy is inevitably sensitive to changes, which is described by $H_{2,T}$. Considering (\ref{P3}), $\mathcal{L}_t$ remains $(\lambda_t\mu_g/2)$-strongly convex and $(2\lambda_tL_{g,1})$-smooth if $\lambda_t\geq 2L_{f,1}/\mu_g$ for all $t\in[T]$. 
% With the update rule in (\ref{eq:y_update}), it naturally holds that
% \begin{align*}
%     \frac{\lambda_t\mu_g}{2}\left\|\vy_t^{k+1} - \vy_{\lambda_t,t}^*(\vx_t)\right\| \leq \left\|\nabla_\vy\mathcal{L}_t(\vx_t, \vy_t^{k+1}, \lambda_t)\right\|.
% \end{align*}
The distance between the current iterate and the optimum can be upper bounded by the norm of a computable gradient. Therefore, the tolerance parameters $\delta_\vy$ and $\delta_\vz$ can be used as stopping criteria for obtaining sufficiently accurate approximate solutions $\vy_{t+1}$ and $\vz_{t+1}$:
% , and following the same way, the final $\|\vy_{t+1} - \vy_{\lambda_t,t}^*(\vx_t)\|$ and $\|\vz_{t+1} - \vy_t^*(\vx_t)\|$ can be controlled by
\begin{align}
    \frac{\lambda_t\mu_g}{2}\left\|\vy_{t+1} - \vy_{\lambda_t,t}^*(\vx_t)\right\| \leq& \left\|\nabla_\vy\mathcal{L}_t(\vx_t, \vy_{t+1}, \lambda_t)\right\| \leq \delta_\vy \quad \text{and} \label{eq:delta_y}\\
    \mu_g\left\|\vz_{t+1} - \vy_t^*(\vx_t)\right\| \leq& \left\|\nabla_\vy g_t(\vx_t, \vy_{t+1})\right\| \leq \delta_\vz. \label{eq:delta_z}
\end{align}
We use two adaptive iterative processes to solve for $\vy_t^*(\vx_t)$ and $\vy_{\lambda_t,t}^*(\vx_t)$ respectively to constitute our Adaptive-iteration Fully First-order Online Bilevel Optimizer (AF$^2$OBO), as shown in Algorithm~\ref{alg:AF$^2$OBO}. The following theorem provides the regret and computational cost guarantees of AF$^2$OBO.

% Note that as $t$ increases, $\mathcal{L}_t$ has a smoothness coefficient much larger than $g_t$, which makes the computational cost of solving the problem much higher. Therefore, solving the two subproblems independently is more efficient.

% \subsection{Regret Analysis}\label{sec:5.2}

% Note that this variation does not involve any linear upper bound dependent on the Lipschitz assumption, therefore we consider this improvement significant. 
% Based on the error tolerance parameters that control the solution approximation errors in (\ref{eq:delta_y}) and (\ref{eq:delta_z}), we can bound the resulting hypergradient error via the following lemma.
% \begin{lemma}\label{lem:nabla_L_3}
% Under Assumptions~\ref{asm1} and~\ref{asm2}, let $\Lambda(t, \lambda_t) = \left(1+\frac{1}{t}\right)^{\tau}\lambda_t$ for some $\tau>0$, $\lambda_1 > \frac{2L_{f,1}}{\mu_g}$, $\alpha\leq \frac{1}{L_{g,1}}$ and $\beta\leq \frac{1}{2\lambda_TL_{g,1}}$, the approximation error of hypergradient in problem (\ref{P3}) can be bounded as follows::
% \begin{align}
%     \left\|\widetilde\nabla \mathcal{L}_t^*(\vx_t) - \nabla \mathcal{L}_t^*(\vx_t)\right\|^2 
%     \leq& 6\lambda_t^2L_{g,1}^2\left\|\vy_{t+1} - \vy_{\lambda_t,t}^*(\vx_t)\right\|^2 + 3\lambda_t^2L_{g,1}^2\left\|\vz_{t+1} - \vy_t^*(\vx_t)\right\|^2 \nonumber\\
%     \leq& 24\kappa_g^2\delta_\vy^2 + 3\lambda_t^2\kappa_g^2\delta_\vz^2. \nonumber
% \end{align}
% \end{lemma}
% Building on the above lemma and following the same analysis framework as for F$^2$OBO, we obtain the following final results.

\begin{algorithm}[tb]
    % \caption{Adaptive-iteration First-order Penalty-based Online Bilevel Optimizer (AF$^2$OBO)}
    \caption{Adaptive-iteration Fully First-order Online Bilevel Optimizer (AF$^2$OBO)}
    \label{alg:AF$^2$OBO}
    \textbf{Input}: $\vx_1$, $\vy_1$, $\vz_1$, $\{\lambda_t, \beta_t\}_{t=1}^T$, $\alpha$, $\gamma$, $\delta_\vy$, $\delta_\vz$ 
    
    \begin{algorithmic}[1] %[1] enables line numbers
        \FOR{\(t = 1\) to \(T\)}
        \STATE Output $\vx_t$ and $\vy_t$, receive $f_t$ and $g_t$.
        \STATE Set $\vy_{t+1} \leftarrow \vy_t$, $\vz_{t+1} \leftarrow \vz_t$
        \WHILE{\(\left\|\nabla_\vy g_t(\vx_t, \vz_{t+1})\right\| > \delta_\vz\)}
        \STATE \(
        \vz_{t+1} \leftarrow \vz_{t+1} - \alpha\nabla_\vy g_t(\vx_t, \vz_{t+1})
        \)
        \ENDWHILE
        \WHILE{\(\left\|\nabla_\vy\mathcal{L}_t(\vx_t, \vy_{t+1}, \lambda_t)\right\| > \delta_\vy\)}
        \STATE \( \vy_{t+1} \leftarrow \vy_{t+1} - \beta\left(\nabla_\vy f_t(\vx_t, \vy_{t+1}) + \lambda_t\nabla_\vy g_t(\vx_t, \vy_{t+1})\right) \)
        \ENDWHILE
        \STATE calculate \( \widetilde\nabla \mathcal{L}_t^*(\vx_t) = \nabla_\vx f_t(\vx_t, \vy_{t+1}) + \lambda_t\left(\nabla_\vx g_t(\vx_t, \vy_{t+1}) - \nabla_\vx g_t(\vx_t, \vz_{t+1})\right)
        \)
        \STATE \( 
        \vx_{t+1} \leftarrow \mathop{\rm argmin}_{\mathbf{u}\in\mathcal{X}}\left\{\left\langle\widetilde{\nabla}\mathcal{L}_t^*(\vx_t), \mathbf{u}\right\rangle + \frac{1}{2\gamma}\left\|\mathbf{u} - \vx_t\right\|^2\right\}
        \)
        \STATE \(
        \lambda_{t+1} \leftarrow \Lambda(t, \lambda_t)
        \)
        \ENDFOR
    \end{algorithmic}
\end{algorithm}

\begin{theorem}\label{thm:AF$^2$OBO}
    Under Assumptions~\ref{asm1},~\ref{asm2}, let $\Lambda(t, \lambda_t) = \left(1+\frac{1}{t}\right)^{\tau}\lambda_t$ for some $\tau>0$, $\lambda_1 > \frac{2L_{f,1}}{\mu_g}$, $\alpha\leq \frac{1}{L_{g,1}}$, $\beta_t \leq \frac{1}{2\lambda_tL_{g,1}}$, $\gamma \leq \frac{1}{2L_F}$ for all $t\in[T]$, $\delta_\vy = \frac{1}{\sqrt{T}}$ and $\delta_\vz = \frac{1}{\sqrt{T^{1+2\tau}}}$, with $\mathcal{I}_T$ representing the total number of inner-loop iterations, Algorithm~\ref{alg:AF$^2$OBO} can guarantee
    \begin{align*}
        \mathrm{BLReg}(T) \leq O\left(\sum_{t=1}^T\frac{1}{t^{2\tau}} + V_T\right) \quad \text{and} \quad \mathcal{I}_T \leq O\left(T^{1+2\tau} + T^{2\tau}H_{2,T}\right).
    \end{align*}
\end{theorem}
As shown in Theorem~\ref{thm:AF$^2$OBO}, 
the main improvement of our method is that it eliminates the dependency of the regret upper bound on $H_{2,T}$, which describes the drift of the inner optimal solution $\|\vy_t^*(\vx) - \vy_{t-1}^*(\vx)\|$ for any $\vx\in\mathcal{X}$. 
% The upper bound on $\mathcal{I}_T$ is dominated by solving the subproblem $\min_{\vy\in\mathbb{R}^{d_2}}\mathcal{L}_t(\vx_t, \vy, \lambda_t)$, since it is generally harder to optimize than $g_t$. 

\subsection{Discussion on the Improved Regret Bound}

Note that AOBO \cite{jia2026achieving} achieves the optimal bilevel local regret $\Omega(1 + V_T)$ with $O(T\log T)$ query of gradients and HVPs. Our method achieves the same bound if we set, i.e., $\tau = 1$, which requires $O(T^3)$ first-order gradient evaluations. If we choose $\tau = 1/2$, We can guarantee a sublinear regret of $O(\log T + V_T)$, which does not rely on $H_{2,T}$. In this setting, the total number of gradient quiries is at most $O(T^2 + TH_{2,T})$.

The upper bound in Theorem~\ref{thm:AF$^2$OBO} have the advantage when the drift of $\vy_t^*(\vx)$ is more severe. We provide a $1$-dimensional example for illustration. Consider a set of $\{f_t,g_t\}_{t=1}^T$ within  constraints $\mathcal{X} = \{x\,|\,x\in[-1,1]\}$ and $y\in\mathbb{R}$ that satisfies $V_T=0$ and $H_{2,T}=O(T)$:
\begin{align*}
    f_t(x, y) \equiv \widetilde{f}(y) = ce^{-y^2} \quad \text{and} \quad g_t(x, y) = \frac{\mu_g}{2}\left(y - (-1)^tx\right)^2.
\end{align*}
Our construction satisfies Assumptions~\ref{asm1},~\ref{asm2} when constant $c$ is set appropriately. With $y_t^*(x) = (-1)^tx$, $f_t(x, y_t^*(x)) = \widetilde{f}(y_t^*(x)) \equiv ce^{-x^2}$, it holds that
{\small
\begin{align*}
    V_T = \sum_{t=2}^T\sup_{x\in[-1,1]}\left|\widetilde{f}(y_t^*(x)) - \widetilde{f}(y_{t-1}^*(x))\right| = 0, \,\,\,\,
    H_{2,T} = \sum_{t=2}^T\sup_{x\in[-1,1]}\left|(-1)^tx - (-1)^{t-1}x\right|^2 = 4T-4.
\end{align*}}
% \begin{align*}
%     V_T {=} \sum_{t=2}^T\sup_{x\in[-1,1]}\left|\widetilde{f}(y_t^*(x)) {-} \widetilde{f}(y_{t-1}^*(x))\right| {=} 0, \,\, H_{2,T} {=} \sum_{t=2}^T\sup_{x\in[-1,1]}\left|(-1)^tx {-} (-1)^{t-1}x\right|^2 {=} 4T{-}4.
% \end{align*}
In the case of oscillatory drift in the inner-level optimal solution, our AF$^2$OBO still guarantees a sublinear regret bound of $O(\log T)$ if $\tau=1/2$, while the regret guarantees of static inner-iteration methods, such as OBBO and F$^2$OBO, will suffer from linear regret.

\section{F$^2$OBO with Stochastic First-Order Gradient}\label{sec:6}

In this section, we study the stochastic OBO problem with our fully first-order algorithm, where the objective functions are formulated as:
\begin{align*}
    f_t(\vx, \vy_t^*(\vx)) = \mathbb{E}_{\xi\sim\mathcal{D}_t^f}\left[f_t(\vx, \vy_t^*(\vx);\xi)\right], \quad g_t(\vx, \vy) = \mathbb{E}_{\zeta\sim\mathcal{D}_t^g}\left[g_t(\vx, \vy;\zeta)\right],
\end{align*}
where $\{\mathcal{D}_t^f, \mathcal{D}_t^g\}_{t=1}^T$ are the time-varying upper- and inner-level distributions, respectively. 

% \begin{theorem}
% Under Assumption~\ref{asm1},~\ref{asm2} and \ref{asm:stochas}, Let $\Lambda(t, \lambda_t) = \left(1+\frac{1}{t}\right)^{\tau}\lambda_t$ for some $\tau>0$, $\lambda_1 > \frac{2L_{f,1}}{\mu_g}$, $\alpha_t = \frac{1}{L_{g,1}t^a}$, $\beta_t = \frac{1}{2\lambda_1L_{g,1}t^b}$, $\gamma_t \equiv \gamma \leq \min\{\frac{1}{2L_F}, \frac{1}{48\kappa_gL_{g,1}}\sqrt{\frac{1-\rho}{7c_\lambda}}\}$, $K_t \geq 4\kappa_g\max\{t^{b-\tau}, t^a\}\log t$, $\widehat{B}_t=t^{B_1}$ and $B_t=t^{B_2}$, Algorithm~\ref{alg:StochasF$^2$OBO} can guarantee
% \begin{align*}
%     \mathbb{E}\left[\mathrm{BLReg}(T)\right] \leq O\left(T^{1-2\tau} + \sigma^2\max\{T^{1+2\tau-B_1}, T^{1+3\tau-b-B_2}, T^{1+2\tau-a-B_2}\} + V_T + H_{2,T}\right).
% \end{align*}
% \end{theorem}
% \begin{corollary}
% Specifically, set
% \begin{align*}
%     \tau=\frac{1}{6}, \quad a=0, \quad b=\frac{1}{6}, \quad B_1=B_2=\frac{2}{3}, \quad K_t=4\kappa_g\log t
% \end{align*}
% we have
% \begin{align*}
%     \mathbb{E}\left[\mathrm{BLReg}(T)\right] \leq O\left(T^{2/3}(1+\sigma^2) + V_T + H_{2,T}\right)
% \end{align*}
% \end{corollary}
We replace the exact gradients in F$^2$OBO with stochastic gradients, and control the accumulated variance in the regret through the inner- and upper-level batch sizes $B_t$ and $\widehat{B}_t$, respectively, for all $t\in[T]$. This leads to Algorithm~\ref{alg:StochasF$^2$OBO}. The following theorem establishes a regret bound.

\begin{algorithm}[tb]
    % \caption{First-order Penalty-based Online Bilevel Optimizer (F$^2$OBO)}
    \caption{Stochastic Fully First-order Online Bilevel Optimizer (SF$^2$OBO)}
    \label{alg:StochasF$^2$OBO}

    \textbf{Input}: $\mathbf{x}_1$, $\mathbf{y}_1$, $\{\lambda_t, \alpha_t, \beta_t, \gamma_t, K_t, B_t, \widehat{B}_t\}_{t=1}^T$
    
    \begin{algorithmic}[1] %[1] enables line numbers
        \FOR{\(t = 1\) to \(T\)}
        \STATE Output $\vx_t$ and $\vy_t$, receive $f_t$ and $g_t$.
        \STATE warm start $\vy_t^1 \leftarrow \vy_t$, $\mathbf{z}_t^1 \leftarrow \mathbf{z}_t$
        \FOR{\(k = 1\) to \(K_t\)}
        \STATE draw samples $\xi_t^k\sim\mathcal{D}_t^f$ and $\zeta_t^k\sim\mathcal{D}_t^g$ with batch size $B_t$
        \STATE \(
        \mathbf{z}_t^{k+1} \leftarrow \mathbf{z}_t^k - \alpha_t\nabla_\mathbf{y} g_t(\vx_t, \mathbf{z}_t^k;\zeta_t^k)
        \)
        \STATE \( \vy_t^{k+1} \leftarrow \vy_t^k - \beta_t\left(\nabla_\mathbf{y} f_t(\vx_t, \vy_t^k;\xi_t^k) + \lambda_t\nabla_\mathbf{y} g_t(\vx_t, \vy_t^k;\zeta_t^k)\right) \)
        \ENDFOR
        \STATE update $\mathbf{y}_{t+1} \leftarrow \vy_t^{K_t+1}$, $\mathbf{z}_{t+1} \leftarrow \mathbf{z}_t^{K_t+1}$, draw samples $\xi_t\sim\mathcal{D}_t^f$, $\zeta_t\sim\mathcal{D}_t^g$ with batch size $\widehat{B}_t$
        \STATE calculate \( 
        \widetilde{\nabla} \mathcal{L}_t^*(\vx_t;\mathcal{S}_t) {=} \nabla_\mathbf{x} f_t(\vx_t, \mathbf{y}_{t+1};\xi_t) {+} \lambda_t\left(\nabla_\mathbf{x} g_t(\vx_t, \mathbf{y}_{t+1};\zeta_t) {-} \nabla_\mathbf{x} g_t(\vx_t, \mathbf{z}_{t+1};\zeta_t)\right)
        \)
        \STATE \( 
        \vx_{t+1} \leftarrow \mathop{\rm argmin}_{\mathbf{u}\in\mathcal{X}}\left\{\left\langle\widetilde{\nabla}\mathcal{L}_t^*(\vx_t;\mathcal{S}_t), \mathbf{u}\right\rangle + \frac{1}{2\gamma_t}\left\|\mathbf{u} - \vx_t\right\|^2\right\}
        \)
        \STATE \(
        \lambda_{t+1} \leftarrow \Lambda(t, \lambda_t)
        \)
        \ENDFOR
    \end{algorithmic}
\end{algorithm}

\begin{theorem}\label{thm:StochasF$^2$OBO}
    Under Assumption~\ref{asm1},~\ref{asm2} and \ref{asm:stochas}, Let $\Lambda(t, \lambda_t) = \left(1+\frac{1}{t}\right)^{\tau}\lambda_t$ for some $\tau \in (0,\frac{1}{2})$, $\lambda_1 > \frac{2L_{f,1}}{\mu_g}$, $\alpha_t = \frac{1}{L_{g,1}t^a}$, $\beta_t = \frac{1}{2\lambda_1L_{g,1}t^b}$, $\gamma_t \equiv \gamma \leq \min\{\frac{1}{2L_F}, \frac{1}{48\kappa_gL_{g,1}}\sqrt{\frac{1-\rho}{7c_\lambda}}\}$ for some constant $c_\lambda>0$, $K_t = 4\kappa_g\max\{t^{b-\tau}, t^a\}\log t$, $\widehat{B}_t=t^{B_1}$ and $B_t=t^{B_2}$, Algorithm~\ref{alg:StochasF$^2$OBO} can guarantee
    \begin{align*}
        \mathbb{E}\left[\mathrm{BLReg}(T)\right] \leq O\left(T^{1-2\tau} + \sigma^2\max\{T^{1+2\tau-B_1}, T^{1+3\tau-b-B_2}, T^{1+2\tau-a-B_2}\} + V_T + H_{2,T}\right).
    \end{align*}
\end{theorem}
As shown in Theorem~\ref{thm:StochasF$^2$OBO}, the three variance terms arise from the updates of $\vx$, $\vy$, and $\vz$, respectively. The first $T^{1+2\tau-B_1}$ term is the main bottleneck, due to the bias of the gradient mapping and the dependence of $\nabla \mathcal{L}_t^*(\vx_t;\mathcal{S}_t)$ variance on $\lambda_t$. This makes the parameter selection more challenging, shown as $\mathbb{E} [ \|\nabla\mathcal{L}_t^*(\vx;\mathcal{S}_t) - \nabla\mathcal{L}_t^*(\vx) \|^2 ] \leq (1+2\lambda_t^2)\sigma^2$. 
\begin{corollary}\label{cor:SF2OBO}
    Under Assumption~\ref{asm1},~\ref{asm2} and \ref{asm:stochas}, with the same parameter setting in Theorem~\ref{thm:StochasF$^2$OBO}, let $\tau=\frac{1}{6}, a=0, b=\frac{1}{6}, B_1=B_2=\frac{2}{3}$, which means $K_t = 4\kappa_g\log t$, Algorithm~\ref{alg:StochasF$^2$OBO} can guarantee
    \begin{align*}
        \mathbb{E}\left[\mathrm{BLReg}(T)\right] \leq O\left(T^{2/3}(1+\sigma^2) + V_T + H_{2,T}\right).
    \end{align*}
\end{corollary}
Corollary~\ref{cor:SF2OBO} gives a specific result of SF$^2$OBO.
SOBBO \cite{bohne2024online} reduces the upper-level variance by using a window-averaged hypergradient, which yields a variance term of $O(T\sigma^2/w)$. However, this approach inherently relies on the analysis of window-averaged regret, which may not accurately characterize the algorithm's performance in dynamic environments. 
SOGD \cite{nazari2025stochastic} establishes results under the standard local regret by employing variance reduction techniques and achieve a $O(T^{1/3}\sigma^2)$. However, its guarantees crucially depend on sublinear environmental variation assumptions as shown in Table~\ref{tab:2}, and additionally require a bounded variance assumption on stochastic second-order gradients, i.e., $\mathbb{E}\|\nabla_{\vx\vy}^2g_t(\vz;\zeta) - \nabla_{\vx\vy}^2g_t(\vz)\|^2 \leq \sigma^2$ and $\mathbb{E}\|\nabla_{\vy\vy}^2g_t(\vz;\zeta) - \nabla_{\vy\vy}^2g_t(\vz)\|^2 \leq \sigma^2$.

% \subsection{Single-Loop}

% \begin{theorem}
% Under Assumption~\ref{asm1},~\ref{asm2} and \ref{asm:stochas}, Let $\Lambda(t, \lambda_t) = \left(1+\frac{1}{t}\right)^{\tau}\lambda_t$ for some $\tau>0$, $\lambda_1 > \frac{2L_{f,1}}{\mu_g}$, $\alpha_t = \frac{1}{L_{g,1}t^a}$, $\beta_t = \frac{1}{2\lambda_1L_{g,1}t^b}$, $\gamma_t = \min\{\frac{1}{2L_F}, \frac{1}{16\kappa_g}\sqrt{\frac{\mu_g}{51c_\nu c_\mu}}\}\cdot\frac{1}{t^c}$, $K_t \equiv 1$, $\widehat{B}_t=t^{B_1}$ and $B_t=t^{B_2}$, with specifically set
% \begin{align*}
%     \tau=\frac{1}{12}, \quad \nu=\frac{3}{2}, \quad \mu=\frac{7}{4}, \quad a=\frac{1}{4}, \quad b=\frac{7}{36}, \quad c=\frac{4}{5}, \quad B_1=\frac{1}{3}, \quad B_2=\frac{2}{3},
% \end{align*}
% Algorithm~\ref{alg:StochasF$^2$OBO} can gaurantee
% \begin{align*}
%     \mathbb{E}\left[\mathrm{BLReg}(T)\right] \leq O\left( T^{5/6}(1+\sigma^2) + T^{4/5}V_T + T^{3/4}H_{2,T} + T^{1/2}E_T \right).
% \end{align*}
% \end{theorem}

\section{Experiments}\label{sec:experitments}

In this section, we provide the empirical studies on the applications of parametric loss tuning for imbalanced data tasks and online hyperparameter optimization.

\subsection{Parametric Loss Tuning for Imbalanced Data}

Tuning for imbalanced data addresses skewed class distributions, where models, especially high-capacity deep networks, often bias toward the majority class and perform poorly on minority classes. Parametric training loss can be used to address this issue by balancing accuracy and fairness while preventing overfitting, and we use it to optimize the model's performance on the training set $\mathcal{D}_t^{\mathrm{tr}}$ with $\lambda = (\gamma_j, \Delta_j)_{j=1}^J$, which means that the logits adjustment is treated as an upper variable and optimized by validation set $\mathcal{D}_t^{\mathrm{val}}$.
\begin{align}
    \min_{\lambda\in\mathbb{R}^{d_1}} F_t(\lambda) &= \frac{1}{\left| \mathcal{D}_t^{\mathrm{val}} \right|}\sum_{(x_i,y_i)\in \mathcal{D}_t^{\mathrm{val}} } u_{y_i}\mathcal{L}(f(\theta_t^*(\lambda);x_i), y_i) \quad \mathrm{s.t.} \quad \theta_t^*(\lambda) = \mathop{\rm argmin}_{\theta\in\mathbb{R}^{d_2}} g_t(\lambda,\theta), \nonumber\\
    \text{where}& \quad g_t(\lambda,\theta) = -\frac{1}{\left|\mathcal{D}_t^{\mathrm{tr}}\right|}\sum_{(x_i,y_i)\in \mathcal{D}_t^{\mathrm{tr}}} \log\frac{\exp(\gamma_{y_i}[f(\theta;x_i)]_{y_i} + \Delta_{y_i})}{\sum_{j=1}^J \exp(\gamma_j[f(\theta;x_i)]_j + \Delta_j)}. \nonumber
\end{align}

\begin{figure*}   
\begin{tabular}{ccc}
\includegraphics[width=0.3\linewidth]{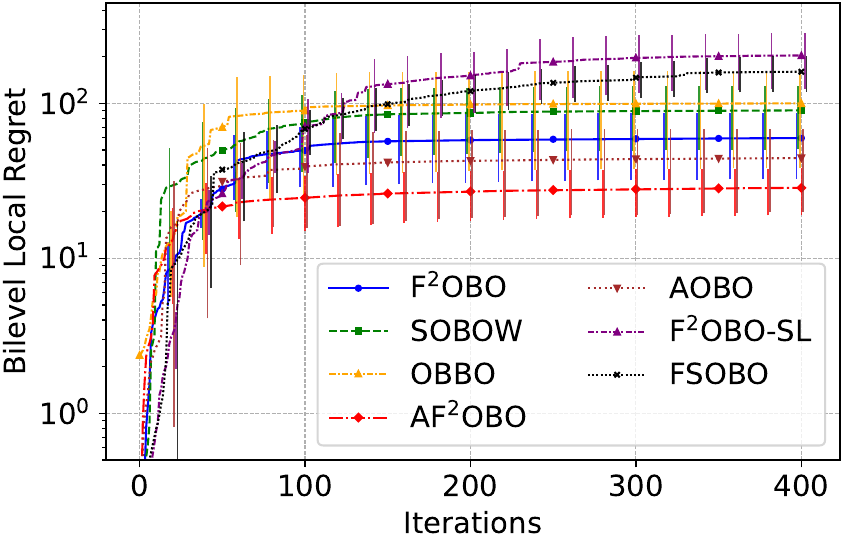} &
\includegraphics[width=0.3\linewidth]{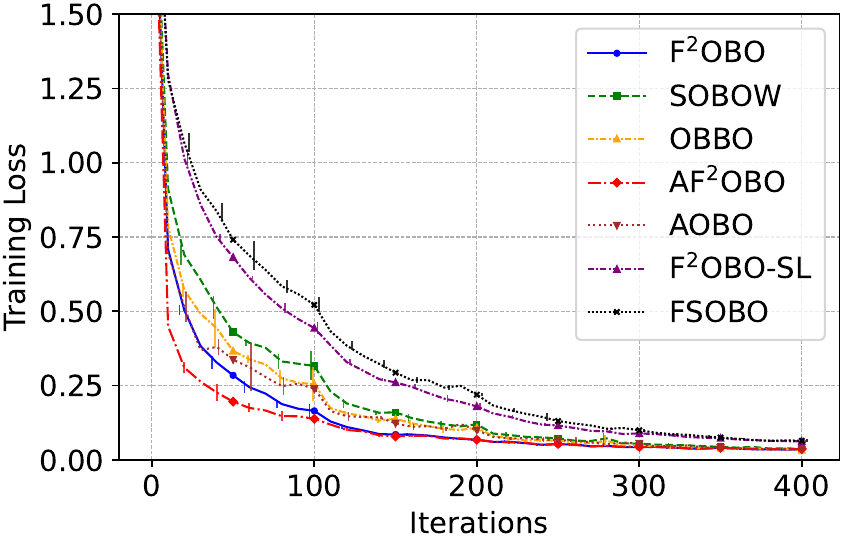} &
\includegraphics[width=0.3\linewidth]{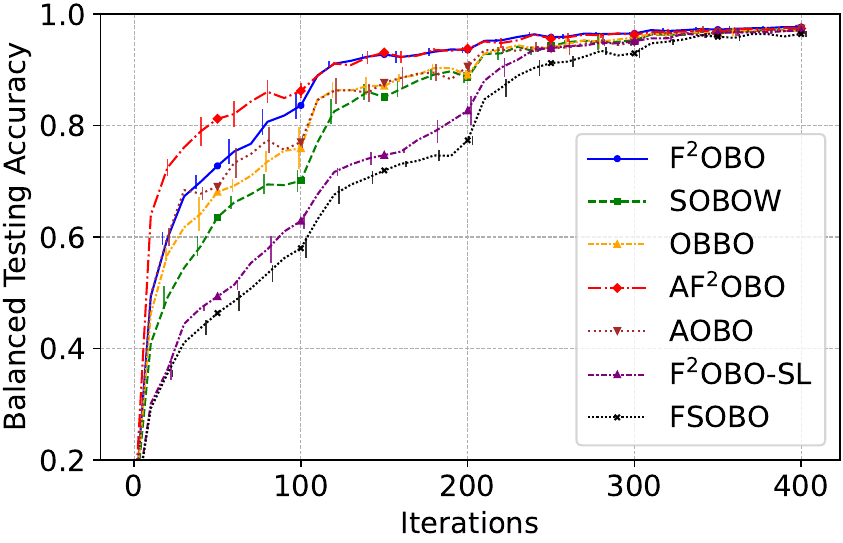} \\
(a) regret & 
(b) training loss & 
(c) test accuracy
\end{tabular}
\caption{Performance of different OBO algorithms on the parametric loss tuning for imbalanced data over five runs. F$^2$OBO-SL denotes F$^2$OBO under the single-loop structure, i.e., $K_t\equiv1$.}
\label{fig:figure3}
\end{figure*}

Here $u_j$ is the inverse class frequency (i.e., the reciprocal of the proportion of class-$j$ samples) \cite{DBLP:journals/corr/abs-2201-01212}. We use a lightweight MNIST CNN with four $3\times3$ conv blocks (Conv-ReLU-$2\times2$ MaxPool–BN), followed by flattening and a $10$-way linear classifier. 

We evaluated performance over $400$ timesteps in four $100$-timestep phases, transitioning from $0.6^i$ to $0.9^i$ distribution for each class $i = 0, 1, \dots , 9$. We tune learning rates $\alpha,\beta,\gamma \in \{ 0.001,0.005,0.01,0.05,0.1,0.5\}$, set $\lambda_1=5.0$ and $\tau=1$. In Figure~\ref{fig:figure3}, we compare the performance of different algorithms, where AF$^2$OBO achieves the best balanced testing accuracy. SF$^2$OBO achieves better empirical performance than FSOBO. However, it exhibits a larger local regret, which also corroborates our theoretical results.

\subsection{Online Hyperparameter Optimization}

We formulate online hyperparameter optimization as a bilevel problem, where the inner problem updates the model parameters and the upper problem optimizes hyperparameters according to validation performance. Under a dynamic data stream, both parameters and hyperparameters are updated at each time step to adapt to distribution shifts. At $t$, the optimizer receives $\mathcal{D}_t=\{\mathcal{D}_t^{\mathrm{tr}},\mathcal{D}_t^{\mathrm{val}}\}$, model parameters $\theta_t$, and hyperparameters $\lambda_t$. We optimize the L2 regularization coefficient with the box constraint $\lambda_t\in\Lambda=[0,0.1]^{d_1}$ for a linear model, and define the objective as follows:
\begin{align*}
    \min_{\lambda\in\Lambda} F_t(\lambda) &= \frac{1}{\left| \mathcal{D}_t^{\mathrm{val}} \right|}\sum_{(x_i,y_i)\in \mathcal{D}_t^{\mathrm{val}} } \mathcal{L}(f(\theta_t^*(\lambda);x_i),y_i) \quad \text{s.t.} \quad \theta_t^*(\lambda) = \mathop{\rm argmin}_{\theta\in\BR^{d_2}} g_t(\lambda,\theta), \\
    \text{where}& \quad g_t(\lambda,\theta) = \frac{1}{\left|\mathcal{D}_t^{\mathrm{tr}}\right|}\sum_{(x_i,y_i)\in \mathcal{D}_t^{\mathrm{tr}}}\mathcal{L}\left( f(\theta;x_i),y_i \right) + \Omega(\lambda,\theta).
\end{align*}
Here, $\Omega(\lambda,\theta)= \frac{1}{2} \sum_{i=0}^{d-1} \lambda_i \|\theta_i\|^2$. Given the optimal model parameter $\theta_t^*(\lambda_t)$, we update the hyperparameters using the validation loss. We evaluate stochastic OBO methods on the 20 Newsgroups text classification task \cite{lang1995newsweeder}. We repeat each experiment five times. Each run lasts 10,000 iterations, where 5\%, 10\%, and 15\% label noise is injected at iterations 2,500, 5,000, and 7,500 to simulate distribution shifts. Table~\ref{tab:3} compares the training loss, test accuracy, and running time of different methods. The results demonstrate that SF$^2$OBO achieves the best overall performance with competitive efficiency.

\begin{table}[t]
\caption{Performance of different stochastic OBO algorithms on online hyperparameter optimization.}
\label{tab:3}
\centering
\begin{tabular}{lccc}
\toprule
Algorithm & Training loss & Test accuracy & Running time \\
\midrule
SOBBO
& $1.3884 \pm 0.0519$
& $0.7979 \pm 0.0005$
& $150.80 \pm 1.84$ \\
SOGD
& $2.0910 \pm 0.0477$
& $0.7607 \pm 0.0005$
& $\mathbf{41.28 \pm 1.51}$ \\
SF$^2$OBO
& $\mathbf{1.2914 \pm 0.0664}$
& $\mathbf{0.8003 \pm 0.0016}$
& $103.02 \pm 1.77$ \\
\bottomrule
\end{tabular}
\end{table}

% SOBBO：0.8077 \pm 0.0008 273.89 \pm 16.78 
% \begin{table}[H]
% \caption{Performance of different stochastic OBO algorithms on online hyperparameter optimization.}
% \centering
% \small
% \setlength{\tabcolsep}{4pt}
% \begin{tabular}{lccccc}
% \toprule
%  & Training Loss & Test Accuracy & Runing Time \\
% \midrule
% SOBBO &  & $0.8077 {\pm} 0.0008$ & $273.89 {\pm} 16.78 $ \\
% SOGD \\
% SF$^2$OBO \\
% \bottomrule
% \end{tabular}
% \end{table}

% We conduct text classification on the 20 Newsgroups dataset \cite{lang1995newsweeder} to evaluate the performance of different stochastic OBO methods. We run the experiment for 10,000 iterations and introduce 5\%, 10\%, and 15\% label noise (by randomly shuffling a specified proportion of data labels while ensuring that the new labels differ from the original ones) at the 2,500th, 5,000th, and 7,500th iterations respectively to simulate distribution shift and evaluate the robustness of the algorithms. In Table~\ref{tab:3}, we compare the training loss and runing time of different algorithms, as well as their accuracy on the test set. The result shows that our SF$^2$OBO achieves the best performance while maintains a efficient time.

\section{Conclusion and Discussion}

% This work presents the implementation and regret analysis of a fully first-order algorithm for non-convex-strongly convex OBO problems. Our proposed fully first-order online bilevel optimizer (F$^2$OBO) achieves the same regret upper bound as the hypergradient-based method under optimal conditions, while maintaining the same computational complexity. We also design an adaptive-iteration F$^2$OBO (AF$^2$OBO), which eliminates the dependence of the regret upper bound on the drift variation of the inner-level optimal solution through a dynamic inner-iteration method, although this introduces additional iterations and regret terms.
% This work develops a fully first-order algorithm for non-convex-strongly-convex OBO and provides its regret analysis. The proposed fully first-order online bilevel optimizer (F$^2$OBO) attains the same regret upper bound as hypergradient-based methods under favorable conditions, while preserving comparable computational complexity. 
% We also propose a single-loop variant, SF$^2$OBO, by introducing additional measures of environmental variation, it can still guarantee sublinear regret.
% We further propose an adaptive-iteration variant, AF$^2$OBO, which removes the dependence of the regret bound on $V_T$ by dynamically adjusting the inner-iteration numbers, at the cost of additional iterations.
This work studies first-order algorithms for non-convex-strongly-convex OBO problem. 
We propose F$^2$OBO, a first-order algorithm, and prove that, under suitable parameter choices, it achieves regret and computational complexity comparable to those of second-order algorithms. 
We also establish a regret upper bound under the more efficient single-loop structure, which involves an additional gradient-variation term.
We further develop an adaptive-iteration variant, AF$^2$OBO, which removes the dependence on $H_{2,T}$ by dynamically adjusting the number of inner iterations, at the cost of additional computation. 
Finally, we extend our analysis to the stochastic OBO setting, propose SF$^2$OBO, and prove its regret bound. 

One limitation of this work is that our theoretical analysis is still restricted to OBO problems with strongly convex inner-level objectives, which limits the applicability of our algorithm. 
In future work, we will extend the analysis of first-order algorithms to OBO problems whose inner-level objectives are non-convex but satisfy the PL condition. 
Developing an appropriate regret notion and designing efficient algorithms for this setting is an interesting and meaningful direction.

\bibliographystyle{plainnat}
\bibliography{references}

@article{chen2025near,
  title={Near-optimal nonconvex-strongly-convex bilevel optimization with fully first-order oracles},
  author={Chen, Lesi and Ma, Yaohua and Zhang, Jingzhao},
  journal={Journal of Machine Learning Research},
  volume={26},
  number={109},
  pages={1--56},
  year={2025}
}

@article{sow2022primal,
  title={A primal-dual approach to bilevel optimization with multiple inner minima},
  author={Sow, Daouda and Ji, Kaiyi and Guan, Ziwei and Liang, Yingbin},
  journal={arXiv preprint arXiv:2203.01123},
  year={2022}
}

@inproceedings{tarzanagh2024online,
  title={Online bilevel optimization: Regret analysis of online alternating gradient methods},
  author={Tarzanagh, Davoud Ataee and Nazari, Parvin and Hou, Bojian and Shen, Li and Balzano, Laura},
  booktitle={International Conference on Artificial Intelligence and Statistics},
  pages={2854--2862},
  year={2024},
  organization={PMLR}
}

@article{lin2023non,
  title={Non-convex bilevel optimization with time-varying objective functions},
  author={Lin, Sen and Sow, Daouda and Ji, Kaiyi and Liang, Yingbin and Shroff, Ness},
  journal={Advances in Neural Information Processing Systems},
  volume={36},
  pages={29692--29717},
  year={2023}
}

@article{bohne2024online,
  title={Online nonconvex bilevel optimization with bregman divergences},
  author={Bohne, Jason and Rosenberg, David and Kazantsev, Gary and Polak, Pawel},
  journal={arXiv preprint arXiv:2409.10470},
  year={2024}
}

@inproceedings{grazzi2020iteration,
  title={On the iteration complexity of hypergradient computation},
  author={Grazzi, Riccardo and Franceschi, Luca and Pontil, Massimiliano and Salzo, Saverio},
  booktitle={International Conference on Machine Learning},
  pages={3748--3758},
  year={2020},
  organization={PMLR}
}

@inproceedings{pedregosa2016hyperparameter,
  title={Hyperparameter optimization with approximate gradient},
  author={Pedregosa, Fabian},
  booktitle={International conference on machine learning},
  pages={737--746},
  year={2016},
  organization={PMLR}
}

@inproceedings{maclaurin2015gradient,
  title={Gradient-based hyperparameter optimization through reversible learning},
  author={Maclaurin, Dougal and Duvenaud, David and Adams, Ryan},
  booktitle={International conference on machine learning},
  pages={2113--2122},
  year={2015},
  organization={PMLR}
}

@inproceedings{ji2021bilevel,
  title={Bilevel optimization: Convergence analysis and enhanced design},
  author={Ji, Kaiyi and Yang, Junjie and Liang, Yingbin},
  booktitle={International conference on machine learning},
  pages={4882--4892},
  year={2021},
  organization={PMLR}
}

@article{mackay2019self,
  title={Self-tuning networks: Bilevel optimization of hyperparameters using structured best-response functions},
  author={MacKay, Matthew and Vicol, Paul and Lorraine, Jon and Duvenaud, David and Grosse, Roger},
  journal={arXiv preprint arXiv:1903.03088},
  year={2019}
}

@inproceedings{kwon2023fully,
  title={A fully first-order method for stochastic bilevel optimization},
  author={Kwon, Jeongyeol and Kwon, Dohyun and Wright, Stephen and Nowak, Robert D},
  booktitle={International Conference on Machine Learning},
  pages={18083--18113},
  year={2023},
  organization={PMLR}
}

@article{liu2022bome,
  title={Bome! bilevel optimization made easy: A simple first-order approach},
  author={Liu, Bo and Ye, Mao and Wright, Stephen and Stone, Peter and Liu, Qiang},
  journal={Advances in neural information processing systems},
  volume={35},
  pages={17248--17262},
  year={2022}
}

@article{sow2022convergence,
  title={On the convergence theory for hessian-free bilevel algorithms},
  author={Sow, Daouda and Ji, Kaiyi and Liang, Yingbin},
  journal={Advances in Neural Information Processing Systems},
  volume={35},
  pages={4136--4149},
  year={2022}
}

@article{yang2023achieving,
  title={Achieving ${O}(\epsilon^{-1.5})$ Complexity in Hessian/Jacobian-free Stochastic Bilevel Optimization},
  author={Yang, Yifan and Xiao, Peiyao and Ji, Kaiyi},
  journal={Advances in Neural Information Processing Systems},
  volume={36},
  pages={39491--39503},
  year={2023}
}

@article{ji2022will,
  title={Will bilevel optimizers benefit from loops},
  author={Ji, Kaiyi and Liu, Mingrui and Liang, Yingbin and Ying, Lei},
  journal={Advances in Neural Information Processing Systems},
  volume={35},
  pages={3011--3023},
  year={2022}
}

@article{bracken1973mathematical,
  title={Mathematical programs with optimization problems in the constraints},
  author={Bracken, Jerome and McGill, James T},
  journal={Operations research},
  volume={21},
  number={1},
  pages={37--44},
  year={1973},
  publisher={INFORMS}
}

@article{vuorio2019multimodal,
  title={Multimodal model-agnostic meta-learning via task-aware modulation},
  author={Vuorio, Risto and Sun, Shao-Hua and Hu, Hexiang and Lim, Joseph J},
  journal={Advances in neural information processing systems},
  volume={32},
  year={2019}
}

@article{nazari2025stochastic,
  title={Stochastic Regret Guarantees for Online Zeroth-and First-Order Bilevel Optimization},
  author={Nazari, Parvin and Hou, Bojian and Tarzanagh, Davoud Ataee and Shen, Li and Michailidis, George},
  journal={arXiv preprint arXiv:2511.01126},
  year={2025}
}

@InProceedings{pmlr-v48-pedregosa16,
  title = 	 {Hyperparameter optimization with approximate gradient},
  author = 	 {Pedregosa, Fabian},
  booktitle = 	 {Proceedings of The 33rd International Conference on Machine Learning},
  pages = 	 {737--746},
  year = 	 {2016},
  editor = 	 {Balcan, Maria Florina and Weinberger, Kilian Q.},
  volume = 	 {48},
  series = 	 {Proceedings of Machine Learning Research},
  address = 	 {New York, New York, USA},
  month = 	 {20--22 Jun},
  publisher =    {PMLR},
  url = 	 {https://proceedings.mlr.press/v48/pedregosa16.html}
}

@article{gould2016differentiating,
  title={On differentiating parameterized argmin and argmax problems with application to bi-level optimization},
  author={Gould, Stephen and Fernando, Basura and Cherian, Anoop and Anderson, Peter and Cruz, Rodrigo Santa and Guo, Edison},
  journal={arXiv preprint arXiv:1607.05447},
  year={2016}
}

@inproceedings{franceschi2017forward,
  title={Forward and reverse gradient-based hyperparameter optimization},
  author={Franceschi, Luca and Donini, Michele and Frasconi, Paolo and Pontil, Massimiliano},
  booktitle={International conference on machine learning},
  pages={1165--1173},
  year={2017},
  organization={PMLR}
}

@inproceedings{shaban2019truncated,
  title={Truncated back-propagation for bilevel optimization},
  author={Shaban, Amirreza and Cheng, Ching-An and Hatch, Nathan and Boots, Byron},
  booktitle={The 22nd international conference on artificial intelligence and statistics},
  pages={1723--1732},
  year={2019},
  organization={PMLR}
}

@misc{lang1995newsweeder,
  title={NewsWeeder: Learning to filter netnews},
  author={Lang, Ken},
  howpublished={Technical report, Carnegie Mellon University},
  year={1995}
}

@article{ghadimi2018approximation,
  title={Approximation methods for bilevel programming},
  author={Ghadimi, Saeed and Wang, Mengdi},
  journal={arXiv preprint arXiv:1802.02246},
  year={2018}
}

@article{DBLP:journals/corr/abs-2201-01212,
  author       = {Mingchen Li and
                  Xuechen Zhang and
                  Christos Thrampoulidis and
                  Jiasi Chen and
                  Samet Oymak},
  title        = {AutoBalance: Optimized Loss Functions for Imbalanced Data},
  journal      = {CoRR},
  volume       = {abs/2201.01212},
  year         = {2022},
  url          = {https://arxiv.org/abs/2201.01212},
  eprinttype    = {arXiv},
  eprint       = {2201.01212},
  bibsource    = {dblp computer science bibliography, https://dblp.org}
}

@article{bohne2026non,
  title={Non-Stationary Functional Bilevel Optimization},
  author={Bohne, Jason and Petrulionyte, Ieva and Arbel, Michael and Mairal, Julien and Polak, Pawe{\l}},
  journal={arXiv preprint arXiv:2601.15363},
  year={2026}
}

@article{song2019maml,
  title={Es-maml: Simple hessian-free meta learning},
  author={Song, Xingyou and Gao, Wenbo and Yang, Yuxiang and Choromanski, Krzysztof and Pacchiano, Aldo and Tang, Yunhao},
  journal={arXiv preprint arXiv:1910.01215},
  year={2019}
}

@article{nichol2018first,
  title={On first-order meta-learning algorithms},
  author={Nichol, Alex and Achiam, Joshua and Schulman, John},
  journal={arXiv preprint arXiv:1803.02999},
  year={2018}
}

@inproceedings{shen2023penalty,
  title={On penalty-based bilevel gradient descent method},
  author={Shen, Han and Chen, Tianyi},
  booktitle={International conference on machine learning},
  pages={30992--31015},
  year={2023},
  organization={PMLR}
}

@article{huang2023momentum,
  title={On momentum-based gradient methods for bilevel optimization with nonconvex lower-level},
  author={Huang, Feihu},
  journal={arXiv preprint arXiv:2303.03944},
  year={2023}
}

@article{arbel2022non,
  title={Non-convex bilevel games with critical point selection maps},
  author={Arbel, Michael and Mairal, Julien},
  journal={Advances in Neural Information Processing Systems},
  volume={35},
  pages={8013--8026},
  year={2022}
}

@article{liu2021towards,
  title={Towards gradient-based bilevel optimization with non-convex followers and beyond},
  author={Liu, Risheng and Liu, Yaohua and Zeng, Shangzhi and Zhang, Jin},
  journal={Advances in Neural Information Processing Systems},
  volume={34},
  pages={8662--8675},
  year={2021}
}

@inproceedings{chen2024finding,
  title={On finding small hyper-gradients in bilevel optimization: Hardness results and improved analysis},
  author={Chen, Lesi and Xu, Jing and Zhang, Jingzhao},
  booktitle={The Thirty Seventh Annual Conference on Learning Theory},
  pages={947--980},
  year={2024},
  organization={PMLR}
}

@article{jia2026achieving,
  title={Achieving Better Local Regret Bound for Online Non-Convex Bilevel Optimization},
  author={Jia, Tingkai and Wang, Haiguang and Chen, Cheng},
  journal={arXiv preprint arXiv:2602.06457},
  year={2026}
}

\clearpage
\appendix

\section{Experimental Details}

In this section, we provide the detailed parameter settings for the experiments in Section~\ref{sec:experitments}.

In the parametric loss tuning experiment for imbalanced data, all algorithms use an inner learning rate $\alpha=0.1$ and an outer learning rate $\gamma=0.001$. For OBBO and F$^2$OBO, we set the number of inner iterations to $K=5$. For AOBO, we set $\delta=0.1$; for AF$^2$OBO, we set $\delta_y=0.5$ and $\delta_z=0.1$. The number of linear-system solving steps for AOBO and SOBOW is set to $5$. For all our first-order algorithms, we use $\lambda_1=5$ with $\tau=1$.

In the online hyperparameter optimization experiment, each algorithm receives a dataset with batch size $256$ at every time step. We use an inner learning rate $\alpha=0.1$ and an outer learning rate $\gamma=0.001$. For SOGD, we set the momentum coefficient to $\eta=0.9$ and the linear-system learning rate to $\beta=0.1$. For SOBBO, we use a window size of $w=50$ with $K=5$ inner iterations. For our SF$^2$OBO, we also set $K=5$ and use $\lambda_1=5$ with $\tau=1$.

\section{Preliminary Lemmas}
We list all the lemmas used in the proof below.

\begin{lemma}[Lemma 16 in \citet{tarzanagh2024online}]\label{y*-y*}
Under Assumptions~\ref{asm1},~\ref{asm2}, for all $t\in[T]$, any $\vx_1$, $\vx_2\in\mathcal{X}$, we have
\begin{align*}
    \left\|\vy_t^*(\vx_1) - \vy_t^*(\vx_2)\right\| \leq \frac{L_{g,1}}{\mu_g} \|\vx_1 - \vx_2\|.
\end{align*}
\end{lemma}

\begin{lemma}[Lemma 4.1 in \citet{chen2025near}]\label{lem:chen}
Under Assumptions~\ref{asm1} and~\ref{asm2}, let $\lambda > \frac{2L_{f,1}}{\mu_g}$, for all $t\in[T]$, it holds that 
\begin{align}
    & \left|f_t(\vx, \vy_t^*(\vx)) - \mathcal{L}_{\lambda,t}^*(\vx)\right| \leq \frac{D_0}{\lambda}, \label{F-L:lip}\\
    & \left\|\nabla f_t(\vx, \vy_t^*(\vx)) - \nabla\mathcal{L}_{\lambda,t}^*(\vx)\right\|\leq\frac{D_1}{\lambda}, \label{F-L:gradlip}\\
    & \mathcal{L}_{\lambda,t}^*(\vx) \text{ is } L_{\mathcal{L},1}\text{-smooth.} \label{lem:Lsmooth}
\end{align}
\end{lemma}

\begin{lemma}[Lemma 16 in \citet{tarzanagh2024online}]\label{lem:L_F}
Under Assumptions~\ref{asm1},~\ref{asm2}, for all $t\in[T]$, any $\vx_1$, $\vx_2\in\mathcal{X}$, we have
\begin{align*}
    \|\nabla f_t(\vx_1, \vy_t^*(\vx_1)) - \nabla f_t(\vx_2, \vy_t^*(\vx_2))\| \leq L_F \|\vx_1 - \vx_2\|,
\end{align*}
where 
\begin{align*}
    L_F := L_{f,1} + \frac{L_{f,0}L_{g,2}}{\mu_g} +\frac{2L_{g,1}L_{f,1}}{\mu_g} + \frac{L_{g,1}^2L_{f,1}}{\mu_g^2} + \frac{2L_{f,0}L_{g,1}L_{g,2}}{\mu_g^2} + \frac{L_{f,0}L_{g,1}^2L_{g,2}}{\mu_g^3}.
\end{align*}
\end{lemma}

\begin{lemma}[Lemma B.7 in \citet{jia2026achieving}]\label{lem:2M+V_T}
Under Assumption~\ref{asm2}, we have
\begin{align*}
    \sum_{t=1}^{T} \left( f_t\left(\vx_t, \vy^*_t(\vx_t) \right) - f_t\left(\vx_{t+1}, \vy^*_t(\vx_{t+1}) \right) \right) \leq 2M + V_T,
\end{align*}
where $V_T$ is defined in Eq.(\ref{pathV}).
\end{lemma}

% \begin{lemma}[Lemma B.8 in \citet{jia2026achieving}]\label{lem:begin}
% Under Assumptions~\ref{asm1},~\ref{asm2}, for all $t\in[T]$, we can obtain
% \begin{align*}
%     \sum_{t=1}^{T} \left\| \mathcal{G}_\mathcal{X}\left( \vx_t, \nabla F_t(\vx_t)), \gamma \right) \right\|^2 
%     \leq& \frac{2}{\gamma\theta} \sum_{t=1}^{T} \left( F_t(\vx_t) - F_t(\vx_{t+1}) \right) \\
%     & + \frac{2}{\gamma\theta} \sum_{t=1}^{T} \left\| \nabla F_t(\vx_t)) - \widetilde{\nabla} f_t(\vx_t, \vy_{t+1}, \mathbf{v}_{t+1}) \right\|^2,
% \end{align*}
% where \( \theta = 1 - \frac{\eta}{2} - \frac{\gamma L_F}{2} \) for some constant $\eta > 0$.
% \end{lemma}

\begin{lemma}[Based on Lemma 3.2 in \citet{kwon2023fully}]\label{y-y:lip}
    Let $\lambda > \frac{2L_{f,1}}{\mu_g}$, for all $t\in[T]$, it holds that $\|\vy_t^*(\vx) - \vy_{\lambda,t}^*(\vx)\|\leq\frac{2L_{f,0}}{\lambda\mu_g}$.
\end{lemma}
\begin{proof}
Lemma 3.2 \cite{kwon2023fully} shows that for any $\vx_1, \vx_2\in\mathcal{X}$ and $\lambda_2 \geq \lambda_1 \geq \frac{2L_{f,1}}{\mu_g}$, it holds that
\begin{align*}
    \left\|\vy_{\lambda_1,t}^*(\vx_1) - \vy_{\lambda_2,t}^*(\vx_2)\right\| \leq \frac{2(\lambda_2 - \lambda_1)}{\lambda_1\lambda_2}\frac{L_{f,0}}{\mu_g} + 3\kappa_g\left\|\vx_2 - \vx_1\right\|.
\end{align*}
Our result can be achieved if $\lambda = \lambda_1$ for any $\lambda\geq\frac{2L_{f,1}}{\mu_g}$, $\lambda_2\rightarrow+\infty$ and $\vx_1 = \vx_2$.
\end{proof}

% \begin{lemma}\label{y-y:gradlip}
%     Let $\lambda > \frac{2L_{f,1}}{\mu_g}$, for all $t\in[T]$, it holds that $\left\|\nabla\vy_t^*(\vx) - \nabla\vy_{\lambda,t}^*(\vx)\right\|\leq \frac{D_2}{\lambda}$, where
%     \begin{align*}
%         D_2 := \frac{L_{f,1}}{\mu_g} + \frac{L_{f,0}L_{g,2}}{\mu_g^2} + \frac{2L_{f,1}L_{g,1}}{\mu_g^2} + \frac{2L_{f,0}L_{g,1}L_{g,2}}{\mu_g^3}.
%     \end{align*}
% \end{lemma}

% \begin{lemma}\label{ygradbound}
%     Let $\lambda > \frac{2L_{f,1}}{\mu_g}$, for all $t\in[T]$, it holds that $\left\|\nabla\vy_{\lambda,t}^*(\vx)\right\| < \frac{4L_{g,1}}{\mu_g}$. 
% \end{lemma}

\begin{lemma}\label{lem:L_lip}
Under Assumptions~\ref{asm1} and \ref{asm2}, for any $\lambda > \frac{2L_{f,1}}{\mu_g}$ and $t\in[T]$, $\mathcal{L}_{\lambda,t}^*(\vx)$ is $L_{\mathcal{L},0}$-Lipschitz continuous, where
\begin{align*}
    L_{\mathcal{L},0} := L_{f,0} + \frac{2L_{f,0}L_{g,1}}{\mu_g}.
\end{align*}
\end{lemma}
\begin{proof}
For any $\lambda > \frac{2L_{f,1}}{\mu_g}$, it follows that
\begin{align*}
    \left\|\nabla\mathcal{L}_t^*(\vx)\right\| =& \left\|\nabla_\vx f_t(\vx, \vy_{\lambda,t}^*(\vx)) + \lambda\left(\nabla_\vx g_t(\vx, \vy_{\lambda,t}^*(\vx)) - \nabla_\vx g_t(\vx, \vy_t^*(\vx))\right)\right\| \\
    \leq& \left\|\nabla_\vx f_t(\vx, \vy_{\lambda,t}^*(\vx))\right\| + \lambda\left\|\nabla_\vx g_t(\vx, \vy_{\lambda,t}^*(\vx)) - \nabla_\vx g_t(\vx, \vy_t^*(\vx))\right\| \\
    \leq& L_{f,0} + \lambda L_{g,1}\left\|\vy_{\lambda,t}^*(\vx)) - \vy_t^*(\vx)\right\| \\
    \leq& L_{f,0} + \frac{2L_{f,0}L_{g,1}}{\mu_g} =: L_{\mathcal{L},0},
\end{align*}
where the last inequality comes from the $L_{f,0}$-Lipschitz continuity of $f_t$ under Assumption~\ref{asm2} and Lemma~\ref{y-y:lip}.
\end{proof}

\section{Proof of Section~\ref{sec:3}}

In this section, we provide the regret upper bound for Algorithm~\ref{alg:F$^2$OBO} and the regret upper bound under the single-loop structure in Sections~\ref{prf:sec:3.1} and~\ref{prf:sec:3.2}, respectively. The following lemma establishes a preliminary regret upper bound for the fully first-order algorithm applied to the approximation problem (\ref{P3}) in the deterministic OBO setting.

\begin{lemma}\label{prf:lem:begin}
    Under Assumptions~\ref{asm1},~\ref{asm2}, suppose that 
    $\gamma_t \leq \frac{1}{2L_F}$ for all $t\in[T]$. 
    Consider the update
    \begin{align*}
        \vx_{t+1}
        \leftarrow
        \vx_t - \gamma_t
        \mathcal{G}_\mathcal{X}
        \left(\vx_t, \widetilde\nabla \mathcal{L}_t(\vx_t), \gamma_t\right).
    \end{align*}
    Then, we obtain
    \begin{align*}
        \mathrm{BLReg}(T) \leq \frac{16M}{\gamma_T} + \frac{8}{\gamma_T}V_T + 16D_1^2\sum_{t=1}^T\frac{1}{\lambda_t^2} + 16\sum_{t=1}^T\left\|\nabla \mathcal{L}_t^*(\vx_t) - \widetilde\nabla \mathcal{L}_t^*(\vx_t)\right\|^2.
    \end{align*}
\end{lemma}
\begin{proof}
It begins with the smoothness of $F_t(\vx)$ shown in Lemma~\ref{lem:L_F}, we have
\begin{align*}
    &F_t(\vx_{t+1}) - F_t(\vx_t) \\
    \leq& \left\langle \nabla F_t(\vx_t)), \vx_{t+1} - \vx_t \right\rangle + \frac{L_F}{2}\left\|\vx_{t+1} - \vx_t\right\|^2 \\
    =& - \gamma_t\left\langle \nabla F_t(\vx_t)), \mathcal{G}_\mathcal{X}\left(\vx_t, \widetilde{\nabla}\mathcal{L}_t^*(\vx_t), \gamma_t\right) \right\rangle + \frac{\gamma_t^2L_F}{2}\left\|\mathcal{G}_\mathcal{X}\left(\vx_t, \widetilde{\nabla}\mathcal{L}_t^*(\vx_t), \gamma_t\right)\right\|^2 \\
    =& - \gamma_t\left\langle \widetilde{\nabla}\mathcal{L}_t^*(\vx_t), \mathcal{G}_\mathcal{X}\left(\vx_t, \widetilde{\nabla}\mathcal{L}_t^*(\vx_t), \gamma_t\right) \right\rangle \\
    & + \gamma_t \left\langle \widetilde\nabla \mathcal{L}_t^*(\vx_t) - \nabla F_t(\vx_t), \mathcal{G}_\mathcal{X}\left(\vx_t, \widetilde{\nabla}\mathcal{L}_t^*(\vx_t), \gamma_t\right) \right\rangle + \frac{\gamma_t^2L_F}{2}\left\|\mathcal{G}_\mathcal{X}\left(\vx_t, \widetilde{\nabla}\mathcal{L}_t^*(\vx_t), \gamma_t\right)\right\|^2,
\end{align*}
then with Young's inequality,
\begin{align}
    F_t(\vx_{t+1}) - F_t(\vx_t) 
    \leq& -\gamma_t\left\|\mathcal{G}_\mathcal{X}\left(\vx_t, \widetilde\nabla \mathcal{L}_t^*(\vx_t), \gamma_t\right)\right\|^2 + \frac{\gamma_t}{2}\left\|\mathcal{G}_\mathcal{X}\left(\vx_t, \widetilde{\nabla}\mathcal{L}_t^*(\vx_t), \gamma_t\right)\right\|^2 \nonumber\\
    & + \frac{\gamma_t}{2}\left\| \widetilde\nabla \mathcal{L}_t^*(\vx_t) - \nabla F_t(\vx_t)) \right\|^2 + \frac{\gamma_t^2L_F}{2}\left\|\mathcal{G}_\mathcal{X}\left(\vx_t, \widetilde{\nabla}\mathcal{L}_t^*(\vx_t), \gamma_t\right)\right\|^2 \nonumber\\
    \leq& -\frac{\gamma_t}{2}\left(1 - \gamma_t L_F\right)\left\|\mathcal{G}_\mathcal{X}\left(\vx_t, \widetilde{\nabla}\mathcal{L}_t^*(\vx_t), \gamma_t\right)\right\|^2 + \frac{\gamma_t}{2}\left\| \widetilde\nabla \mathcal{L}_t^*(\vx_t) - \nabla F_t(\vx_t)) \right\|^2. \label{A.10_1}
\end{align}
Moreover, by
\begin{align}
    -\left\|\mathcal{G}_\mathcal{X}\left(\vx_t, \widetilde{\nabla}\mathcal{L}_t^*(\vx_t), \gamma_t\right)\right\|^2
    \leq& \left\|\mathcal{G}_\mathcal{X}(\vx_t, \nabla F_t(\vx_t)), \gamma_t) - \mathcal{G}_\mathcal{X}\left(\vx_t, \widetilde{\nabla}\mathcal{L}_t^*(\vx_t), \gamma_t\right)\right\|^2 \nonumber\\
    & - \frac{1}{2} \left\|\mathcal{G}_\mathcal{X}(\vx_t, \nabla F_t(\vx_t)), \gamma_t)\right\|^2, \label{A.10_2}
\end{align}
substituting Eq.(\ref{A.10_2}) into Eq.(\ref{A.10_1}) and rearranging the terms, it follows that
\begin{align*}
    \left\|\mathcal{G}_\mathcal{X}\left(\vx_t, \nabla F_t(\vx_t), \gamma_t\right)\right\|^2 \leq& \frac{8}{\gamma_t}\left(F_t(\vx_t) - F_t(\vx_{t+1})\right) + 8\left\|\widetilde\nabla \mathcal{L}_t^*(\vx_t) - \nabla F_t(\vx_t)\right\|^2 \\
    \leq& \frac{8}{\gamma_t}\left(F_t(\vx_t) - F_t(\vx_{t+1})\right) + 16\left\|\nabla F_t(\vx_t) - \nabla \mathcal{L}_t^*(\vx_t)\right\|^2 \\
    & + 16\left\|\nabla \mathcal{L}_t^*(\vx_t) - \widetilde\nabla \mathcal{L}_t^*(\vx_t)\right\|^2,
\end{align*}
where we choose $\gamma_t \leq \frac{1}{2L_F}$ to ensure that $1-\gamma_t L_F \geq \frac{1}{2}$. Summing the above inequality over $t=1,\ldots,T$, and substituting Lemma~\ref{lem:2M+V_T} and Lemma~\ref{lem:chen}~\eqref{F-L:gradlip} in it, we obtain
\begin{align*}
    \mathrm{BLReg}(T) \leq \frac{16M}{\gamma_t} + \frac{8}{\gamma_t}V_T + 16D_1^2\sum_{t=1}^T\frac{1}{\lambda_t^2} + 16\sum_{t=1}^T\left\|\nabla \mathcal{L}_t^*(\vx_t) - \widetilde\nabla \mathcal{L}_t^*(\vx_t)\right\|^2.
\end{align*}
Thus we finish the proof.
\end{proof}

\subsection{Proof of Theorem~\ref{thm:F$^2$OBO}}\label{prf:sec:3.1}

In this section, we give the proof of Theorem~\ref{thm:F$^2$OBO}. Lemma~\ref{prf:lem:nabla_L} provides an upper bound on the approximate hypergradient error of Algorithm~\ref{alg:F$^2$OBO} by using the inner-loop iteration error bounds established in Lemmas~\ref{lem:y-y} and~\ref{lem:z-z}. Substituting this bound into Lemma~\ref{prf:lem:begin} yields the proof of Theorem~\ref{prf:thm:F$^2$OBO}.

\begin{lemma}\label{lem:y-y}
    Under Assumptions~\ref{asm1}, \ref{asm2}, let $\beta_t \leq \frac{1}{2\lambda_tL_{g,1}}$ and $K_t \equiv K \geq 1 - \frac{\ln2}{\ln\rho_\vy}$, where $\rho_\vy = 1 - \frac{\beta_t\lambda_t\mu_g}{2}$, Algorithm~\ref{alg:F$^2$OBO} can obtain
    \begin{align*}
        &\left\|\vy_{t+1} - \vy_{\lambda_t,t}^*(\vx_t)\right\|^2 \\
        \leq& \rho_\vy^{t-1}\left\|\vy_2 - \vy_{\lambda,1}^*(\vx_1)\right\|^2 + \frac{12L_{f,0}^2}{\mu_g^2}\sum_{j=0}^{t-2}\frac{1}{\lambda_{t-j}^2}\rho_\vy^{j+K} + 12\kappa_g^2\sum_{j=0}^{t-2}\rho_\vy^{j+K}\left\|\vx_{t-1-j} - \vx_{t-j}\right\|^2 \\
        & + 12\sum_{j=0}^{t-2}\rho_\vy^{j+K}\left\|\vy_{t-1-j}^*(\vx_{t-j}) - \vy_{t-j}^*(\vx_{t-j})\right\|^2.
    \end{align*}
\end{lemma}
\begin{proof}
It begins with
\begin{align}
    &\left\|\vy_t^{k+1} - \vy_{\lambda_t,t}^*(\vx_t)\right\|^2 \nonumber\\
    =& \left\|\vy_t^k - \beta_t\nabla_\vy \mathcal{L}_t(\vx_t, \vy_t^k, \lambda_t) - \vy_{\lambda_t,t}^*(\vx_t)\right\|^2 \nonumber\\
    =& \left\|\vy_t^k - \vy_{\lambda_t,t}^*(\vx_t)\right\|^2 - 2\beta_t\left\langle \nabla_\vy \mathcal{L}_t(\vx_t, \vy_t^k, \lambda_t), \vy_t^k - \vy_{\lambda_t,t}^*(\vx_t) \right\rangle + \beta_t^2\left\|\nabla_\vy \mathcal{L}_t(\vx_t, \vy_t^k, \lambda_t)\right\|^2 \nonumber\\
    \leq& \left\|\vy_t^k - \vy_{\lambda_t,t}^*(\vx_t)\right\|^2 - \frac{\beta_t\lambda_t\mu_g}{2}\left\|\vy_t^k - \vy_{\lambda_t,t}^*(\vx_t)\right\|^2 - \left(\frac{\beta_t}{2\lambda_tL_{g,1}} - \beta_t^2\right)\left\|\nabla_\vy \mathcal{L}_t(\vx_t, \vy_t^k, \lambda_t)\right\|^2 \nonumber\\
    =& \rho_\vy \left\|\vy_t^k - \vy_{\lambda_t,t}^*(\vx_t)\right\|^2, \label{eq:rho_y}
\end{align}
where we set $\beta_t \leq \frac{1}{2\lambda_t L_{g,1}}$ and $\rho_\vy := 1 - \frac{\beta_t\lambda_t\mu_g}{2}$. The first inequality follows from the fact that, if $\lambda_t \geq \frac{2L_{f,1}}{\mu_g}$ holds for all $t\in[T]$, then $\mathcal{L}_t$ is $(\lambda_t\mu_g/2)$-strongly convex and $(2\lambda_t L_{g,1})$-smooth, for any $\vx\in\mathcal{X}$, $\vy\in\mathbb{R}^{d_2}$,
\begin{align*}
    \frac{\lambda_t\mu_g}{2}\left\|\vy - \vy_{\lambda_t,t}^*(\vx)\right\|^2 \leq& \left\langle \nabla_\vy \mathcal{L}_t(\vx, \vy, \lambda_t), \vy - \vy_{\lambda_t,t}^*(\vx) \right\rangle \\
    \frac{1}{2\lambda_tL_{g,1}}\left\|\nabla_\vy \mathcal{L}_t(\vx, \vy, \lambda_t)\right\|^2 \leq& \left\langle \nabla_\vy \mathcal{L}_t(\vx, \vy, \lambda_t), \vy - \vy_{\lambda_t,t}^*(\vx) \right\rangle.
\end{align*}
Thus we natually have
\begin{align*}
    \left\|\vy_{t+1} - \vy_{\lambda_t,t}^*(\vx_t)\right\|^2 \leq& \rho_\vy^{K_t} \left\|\vy_t - \vy_{\lambda_t,t}^*(\vx_t)\right\|^2 \\
    \leq& 2\rho_\vy^{K_t} \left\|\vy_t - \vy_{\lambda_{t-1},t-1}^*(\vx_{t-1})\right\|^2 + 2\rho_\vy^{K_t} \left\|\vy_{\lambda_{t-1},t-1}^*(\vx_{t-1}) - \vy_{\lambda_t,t}^*(\vx_t)\right\|^2.
\end{align*}
Then it follows that
\begin{align}
    &\left\|\vy_{\lambda_{t-1},t-1}^*(\vx_{t-1}) - \vy_{\lambda_t,t}^*(\vx_t)\right\|^2 \nonumber\\
    \leq& 3\left\|\vy_{\lambda_{t-1},t-1}^*(\vx_{t-1}) - \vy_{t-1}^*(\vx_{t-1})\right\|^2 + 3\left\|\vy_{t-1}^*(\vx_{t-1}) - \vy_t^*(\vx_t)\right\|^2 \nonumber\\
    & + 3\left\|\vy_t^*(\vx_t) - \vy_{\lambda_t,t}^*(\vx_t)\right\|^2 \nonumber\\
    \leq& \frac{6L_{f,0}^2}{\lambda_t^2\mu_g^2} + 6\kappa_g^2 \left\|\vx_{t-1} - \vx_t\right\|^2 + 6\left\|\vy_{t-1}^*(\vx_t) - \vy_t^*(\vx_t)\right\|^2, \label{eq:y-y}
\end{align}
where the last ineqaulity comes from Lemma~\ref{y-y:lip} and~\ref{y*-y*}. Let $K_t \equiv K \geq 1 - \frac{\ln2}{\ln\rho_\vy}$ such that $2\rho_\vy^K \leq \rho_\vy$, we have
\begin{align}
    &\left\|\vy_{t+1} - \vy_{\lambda_t,t}^*(\vx_t)\right\|^2 \nonumber\\
    \leq& \rho_\vy\left\|\vy_t - \vy_{\lambda_{t-1},t-1}^*(\vx_{t-1})\right\|^2 + \frac{12L_{f,0}^2}{\lambda_t^2\mu_g^2}\rho_\vy^{K} + 12\kappa_g^2 \rho_\vy^{K}\left\|\vx_{t-1} - \vx_t\right\|^2  \nonumber\\
    & + 12\rho_\vy^{K}\left\|\vy_{t-1}^*(\vx_t) - \vy_t^*(\vx_t)\right\|^2 \nonumber\\
    \leq& \rho_\vy^{t-1}\left\|\vy_2 - \vy_{\lambda,1}^*(\vx_1)\right\|^2 + \frac{12L_{f,0}^2}{\mu_g^2}\sum_{j=0}^{t-2}\frac{1}{\lambda_{t-j}^2}\rho_\vy^{j+K} + 12\kappa_g^2 \sum_{j=0}^{t-2}\rho_\vy^{j+K}\left\|\vx_{t-1-j} - \vx_{t-j}\right\|^2 \nonumber\\
    & + 12\sum_{j=0}^{t-2}\rho_\vy^{j+K}\left\|\vy_{t-1-j}^*(\vx_{t-j}) - \vy_{t-j}^*(\vx_{t-j})\right\|^2. 
\end{align}
Thus we finish the proof.
\end{proof}

\begin{lemma}\label{lem:z-z}
    Under Assumptions~\ref{asm1}, \ref{asm2}, let $\alpha \leq \frac{1}{L_{g,1}}$ and $K_t \equiv K \geq 1 - \frac{\ln2}{\ln\rho_\vz}$, where $\rho_\vz = 1 - \alpha\mu_g$, Algorithm~\ref{alg:F$^2$OBO} can obtain
    \begin{align*}
        \left\|\vz_{t+1} - \vy_t^*(\vx_t)\right\|^2 
        \leq& \rho_\vz^{t-1}\left\|\vz_2 - \vy_1^*(\vx_1)\right\|^2 + 4\kappa_g^2\sum_{j=0}^{t-2}\rho_\vz^{j+K} \left\|\vx_{t-1-j} - \vx_{t-j}\right\|^2 \\
        & + 4\sum_{j=0}^{t-2}\rho_\vz^{j+K} \left\|\vy_{t-1-j}^*(\vx_{t-j}) - \vy_{t-j}^*(\vx_{t-j})\right\|^2.
    \end{align*}
\end{lemma}
\begin{proof}
It begins with
\begin{align*}
    &\left\|\vz_t^{k+1} - \vy_t^*(\vx_t)\right\|^2 \\
    =& \left\|\vz_t^k - \alpha\nabla_\vy g_t(\vx_t, \vz_t^k) - \vy_t^*(\vx_t)\right\|^2 \\
    =& \left\|\vz_t^k - \vy_t^*(\vx_t)\right\|^2 - 2\alpha\left\langle \nabla_\vy g_t(\vx_t, \vz_t^k), \vz_t^k - \vy_t^*(\vx_t) \right\rangle + \alpha^2\left\|\nabla_\vy g_t(\vx_t, \vz_t^k)\right\|^2 \\
    \leq& \left\|\vz_t^k - \vy_t^*(\vx_t)\right\|^2 - \alpha\mu_g\left\|\vz_t^k - \vy_t^*(\vx_t)\right\|^2 - \left(\frac{\alpha}{L_{g,1}} - \alpha^2\right)\left\|\nabla_\vy g_t(\vx_t, \vz_t^k)\right\|^2 \\
    =& \rho_\vz \left\|\vz_t^k - \vy_t^*(\vx_t)\right\|^2,
\end{align*}
where we set $\alpha \leq \frac{1}{L_{g,1}}$ and $\rho_\vz := 1 - \alpha\mu_g$. the first inequality follows from Assumption~\ref{asm1} and~\ref{asm2}. Then we complete the proof by
\begin{align*}
    &\left\|\vz_{t+1} - \vy_t^*(\vx_t)\right\|^2 \\
    \leq& \rho_\vz^{K_t}\left\|\vz_t - \vy_t^*(\vx_t)\right\|^2 \\
    \leq& 2\rho_\vz^{K_t}\left\|\vz_t - \vy_{t-1}^*(\vx_{t-1})\right\|^2 + 2\rho_\vz^{K_t}\left\|\vy_{t-1}^*(\vx_{t-1}) - \vy_t^*(\vx_t)\right\|^2 \\
    \overset{(i)}{\leq}& \rho_\vz\left\|\vz_t - \vy_{t-1}^*(\vx_{t-1})\right\|^2 + 4\kappa_g^2\rho_\vz^K\left\|\vx_{t-1} - \vx_t\right\|^2 + 4\rho_\vz^K\left\|\vy_{t-1}^*(\vx_t) - \vy_t^*(\vx_t)\right\|^2 \\
    \leq& \rho_\vz^{t-1}\left\|\vz_2 - \vy_1^*(\vx_1)\right\|^2 + 4\kappa_g^2\sum_{j=0}^{t-2}\rho_\vz^{j+K} \left\|\vx_{t-1-j} - \vx_{t-j}\right\|^2 \\
    & + 4\sum_{j=0}^{t-2}\rho_\vz^{j+K} \left\|\vy_{t-1-j}^*(\vx_{t-j}) - \vy_{t-j}^*(\vx_{t-j})\right\|^2,
\end{align*}
where in $(i)$, we set $K_t \equiv K\geq 1- \frac{\ln2}{\ln\rho_\vz}$ such that $2\rho_\vz^K \leq \rho_\vz$.
\end{proof}

\begin{lemma}\label{prf:lem:nabla_L}
Under Assumptions~\ref{asm1} and~\ref{asm2}, let $\Lambda(t, \lambda_t) = \left(1+\frac{1}{t}\right)^{\tau}\lambda_t$ for some $\tau>0$, $\lambda_1 > \frac{2L_{f,1}}{\mu_g}$, $\alpha\leq \frac{1}{L_{g,1}}$, $\beta_t \leq \frac{1}{2\lambda_tL_{g,1}}$, $\gamma_t \equiv \gamma \leq \frac{1}{4\kappa_g}\sqrt{\frac{1-\rho}{C_\lambda}}$ and $K_t \equiv K \geq \max\{1 - \frac{\ln2}{\ln\rho}, \frac{\ln c - 2\tau\ln T}{\ln\rho}\}$ for some $c > 0$, the approximation error of hypergradient in Algorithm~\ref{alg:F$^2$OBO} can be bounded as follows:
\begin{align*}
    \sum_{t=2}^T\left\|\widetilde\nabla \mathcal{L}_t^*(\vx_t) - \nabla \mathcal{L}_t^*(\vx_t)\right\|^2 
    \leq& 2c_\tau\Delta_1 + \frac{576c\kappa_g^3L_{f,0}^2}{\lambda_1^2}\sum_{t=1}^T\frac{1}{t^{2\tau}} + \frac{8\gamma^2C_\lambda\kappa_g^2D_1^2}{\lambda_1^2(1-\rho)}\sum_{t=1}^T\frac{1}{t^{2\tau}} \\
    & + \frac{4\gamma^2C_\lambda\kappa_g^2}{1-\rho}\mathrm{BLReg}(T) + \frac{2C_\lambda}{1-\rho}\sum_{t=2}^T\sup_{\vx\in\mathbb{R}^{d_1}}\left\|\vy_{t-1}^*(\vx) - \vy_{t}^*(\vx)\right\|^2.
\end{align*}
where $\rho$, $c$, $c_\tau$, $C_\lambda$ and $\Delta_1$ are some constants.
\end{lemma}
\begin{proof}
It begins with
\begin{align}
    \left\|\widetilde\nabla \mathcal{L}_t^*(\vx_t) - \nabla \mathcal{L}_t^*(\vx_t)\right\|^2 \leq& 3\left\|\nabla_\vx f_t(\vx_t, \vy_{t+1}) - \nabla_\vx f_t(\vx_t, \vy_{\lambda_t,t}^*(\vx))\right\|^2 \nonumber\\
    & + 3\lambda_t^2\left\|\nabla_\vx g_t(\vx_t, \vy_{t+1}) - \nabla_\vx g_t(\vx_t, \vy_{\lambda_t,t}^*(\vx_t))\right\|^2 \nonumber\\
    & + 3\lambda_t^2\left\|\nabla_\vx g_t(\vx_t, \vz_{t+1}) - \nabla_\vx g_t(\vx_t, \vy_t^*(\vx_t))\right\|^2 \nonumber\\
    \overset{(i)}{\leq}& 3\left(L_{f,1}^2 + \lambda_t^2L_{g,1}^2\right)\left\|\vy_{t+1} - \vy_{\lambda_t,t}^*(\vx_t)\right\|^2 + 3\lambda_t^2L_{g,1}^2\left\|\vz_{t+1} - \vy_t^*(\vx_t)\right\|^2 \nonumber\\
    \overset{(ii)}{\leq}& 6\lambda_t^2L_{g,1}^2\underbrace{\left\|\vy_{t+1} - \vy_{\lambda_t,t}^*(\vx_t)\right\|^2}_{\mathcal{L}_{t}\text{ inner iters bound: }(a)} + 3\lambda_t^2L_{g,1}^2\underbrace{\left\|\vz_{t+1} - \vy_t^*(\vx_t)\right\|^2}_{g_t\text{ inner iters bound: }(b)}, \label{a+b}
\end{align}
where (i) comes from the smoothness of $f_t$ and $g_t$ under Assumption~\ref{asm2}, and (ii) comes from $\lambda_t>\frac{2L_{f,1}}{\mu_g}$, thus $L_{f,1}<\lambda_tL_{g,1}$. 
Substituting Lemma~\ref{lem:y-y} and Lemma~\ref{lem:z-z} into~\eqref{a+b}, we have 
\begin{align*}
    &\left\|\widetilde\nabla \mathcal{L}_t^*(\vx_t) - \nabla \mathcal{L}_t^*(\vx_t)\right\|^2 \\
    \leq& 6\lambda_t^2L_{g,1}^2\rho_\vy^{t-1}\left\|\vy_2 - \vy_{\lambda,1}^*(\vx_1)\right\|^2 + 3\lambda_t^2L_{g,1}^2\rho_\vz^{t-1}\|\vz_2 - \vy_1^*(\vx_1)\|^2 + 72\kappa_g^2L_{f,0}^2\sum_{j=0}^{t-2}\frac{\lambda_t^2}{\lambda_{t-j}^2}\rho_\vy^{j+K} \\
    & + 72\lambda_t^2\kappa_g^2L_{g,1}^2\sum_{j=0}^{t-2}\rho_\vy^{j+K}\left\|\vx_{t-1-j} {-} \vx_{t-j}\right\|^2 + 72\lambda_t^2L_{g,1}^2\sum_{j=0}^{t-2}\rho_\vy^{j+K}\left\|\vy_{t-1-j}^*(\vx_{t-j}) {-} \vy_{t-j}^*(\vx_{t-j})\right\|^2 \\
    & + 12\lambda_t^2\kappa_g^2L_{g,1}^2\sum_{j=0}^{t-2}\rho_\vz^{j+K}\|\vx_{t-1-j} {-} \vx_{t-j}\|^2 + 12\lambda_t^2L_{g,1}^2\sum_{j=0}^{t-2}\rho_\vz^{j+K}\|\vy_{t-1-j}^*(\vx_{t-j}) {-} \vy_{t-j}^*(\vx_{t-j})\|^2.
\end{align*}
Suppose that, for a sufficiently large $K$, i.e., $K \geq \frac{\ln c - 2\tau\ln T}{\ln\rho}$, $\lambda_t^2 \rho^K \leq c$ holds for some constant $c > 0$, where $\rho := \max\{\rho_\vy,\rho_\vz\}$. Then we have
\begin{align}
    &\left\|\widetilde\nabla \mathcal{L}_t^*(\vx_t) - \nabla \mathcal{L}_t^*(\vx_t)\right\|^2 \nonumber\\
    \leq& 6\lambda_t^2L_{g,1}^2\rho_\vy^{t-1}\left\|\vy_2 - \vy_{\lambda,1}^*(\vx_1)\right\|^2 + 3\lambda_t^2L_{g,1}^2\rho_\vz^{t-1}\|\vz_2 - \vy_1^*(\vx_1)\|^2 + 72c\kappa_g^2L_{f,0}^2\sum_{j=0}^{t-2}\frac{\rho_\vy^{j}}{\lambda_{t-j}^2} \nonumber\\
    & + 72c\kappa_g^2L_{g,1}^2\sum_{j=0}^{t-2}\rho_\vy^{j}\left\|\vx_{t-1-j} - \vx_{t-j}\right\|^2 + 72cL_{g,1}^2\sum_{j=0}^{t-2}\rho_\vy^{j}\left\|\vy_{t-1-j}^*(\vx_{t-j}) - \vy_{t-j}^*(\vx_{t-j})\right\|^2 \nonumber\\
    & + 12c\kappa_g^2L_{g,1}^2\sum_{j=0}^{t-2}\rho_\vz^{j}\|\vx_{t-1-j} - \vx_{t-j}\|^2 + 12cL_{g,1}^2\sum_{j=0}^{t-2}\rho_\vz^{j}\|\vy_{t-1-j}^*(\vx_{t-j}) - \vy_{t-j}^*(\vx_{t-j})\|^2 \nonumber\\
    \overset{(i)}{\leq}& \lambda_t^2\rho^{t-1}\Delta_1' + 72c\kappa_g^2L_{f,0}^2\sum_{j=0}^{t-2}\frac{\rho_\vy^{j}}{\lambda_{t-j}^2} + C_\lambda\sum_{j=0}^{t-2}\rho^j\left\|\vy_{t-1-j}^*(\vx_{t-j}) - \vy_{t-j}^*(\vx_{t-j})\right\|^2 \nonumber\\
    & + C_\lambda\kappa_g^2\sum_{j=0}^{t-2}\rho^j\left\|\vx_{t-1-j} - \vx_{t-j}\right\|^2, \label{L-L}
\end{align}
where in (i), we define $\Delta_1' := 6L_{g,1}^2\|\vy_2 - \vy_{\lambda_1,1}^*(\vx_1)\|^2 + 3L_{g,1}^2\|\vz_2 - \vy_1^*(\vx_1)\|^2$, $C_\lambda :=84cL_{g,1}^2$ and $\rho := \max\{\rho_\vy, \rho_\vz\}$. After telescoping~\eqref{L-L}, we have
\begin{align}
    &\sum_{t=2}^T \left\|\widetilde\nabla \mathcal{L}_t^*(\vx_t) - \nabla \mathcal{L}_t^*(\vx_t)\right\|^2 \nonumber\\
    \leq& \Delta_1‘\sum_{t=2}^T\lambda_t^2\rho^{t-1} + 72c\kappa_g^2L_{f,0}^2\sum_{t=2}^T\sum_{j=0}^{t-2}\frac{\rho_\vy^{j}}{\lambda_{t-j}^2} + C_\lambda\sum_{t=2}^T\sum_{j=0}^{t-2}\rho^j\left\|\vy_{t-1-j}^*(\vx_{t-j}) - \vy_{t-j}^*(\vx_{t-j})\right\|^2 \nonumber\\
    & + C_\lambda\kappa_g^2\sum_{t=2}^T\sum_{j=0}^{t-2}\rho^j\left\|\vx_{t-1-j} - \vx_{t-j}\right\|^2. \label{sum_L-L}
\end{align}
Note that, for $\rho=e^{-\theta}$ and 
$\lambda_t=\prod_{i=1}^t(1 + \frac{1}{i})^\tau\lambda_1=\lambda_1 t^\tau$, we have
\begin{align*}
    \sum_{t=1}^T \lambda_t^2 \rho^{t-1} {=} \lambda_1^2 \sum_{t=1}^T t^{2\tau}\rho^{t-1} {\leq} \frac{\lambda_1^2}{\rho}
    \sum_{t=1}^{\infty} t^{2\tau}\rho^t {=} \lambda_1^2 e^\theta
    \sum_{t=1}^{\infty} t^{2\tau}e^{-\theta t} {\leq} \lambda_1^2 e^\theta
    \left(
    \int_0^\infty x^{2\tau}e^{-\theta x}\,\mathrm{d}x {+} 1
    \right).
\end{align*}
Using the standard Gamma integral
\begin{align*}
    \int_0^\infty x^{s-1}e^{-\theta x}\,\mathrm{d}x
    = \frac{\Gamma(s)}{\theta^s},
    \qquad s>0,\ \theta>0,
\end{align*}
we obtain
\[
    \int_0^\infty x^{2\tau}e^{-\theta x}\,\mathrm{d}x
    = \frac{\Gamma(2\tau+1)}{\theta^{2\tau+1}}.
\]
Therefore,
\begin{align}
    \sum_{t=1}^T \lambda_t^2\rho^{t-1}
    \leq
    \lambda_1^2 e^\theta
    \left(
    \frac{\Gamma(2\tau+1)}{\theta^{2\tau+1}}+1
    \right)
    =: c_\tau .
    \label{c_tau}
\end{align}
Also, if we set $m=t-j$, then
\begin{align*}
    \sum_{t=2}^T\sum_{j=0}^{t-2}\frac{\rho_\vy^j}{\lambda_{t-j}^2} <& \sum_{t=1}^T\sum_{j=0}^{t-2}\frac{\rho_\vy^j}{\lambda_{t-j}^2} = \sum_{t=1}^T\sum_{j=0}^{t-2}\frac{\rho_\vy^j}{(t-j)^{2\tau}\lambda_1^2} = \frac{1}{\lambda_1^2}\sum_{m=2}^T\frac{1}{m^{2\tau}}\sum_{j=0}^{T-m}\rho_\vy^j \leq \frac{4L_{g,1}}{\lambda_1^2\mu_g} \sum_{m=2}^T\frac{1}{m^{2\tau}}.
\end{align*}
Substituting the above inequality and~\eqref{c_tau} into~\eqref{sum_L-L}, and adopting a fixed step-size policy where $\gamma_t \equiv \gamma$, it holds that
\begin{align*}
    &\sum_{t=2}^T\left\|\widetilde\nabla \mathcal{L}_t^*(\vx_t) - \nabla \mathcal{L}_t^*(\vx_t)\right\|^2 \\
    \leq& 
    c_\tau\Delta_1’ + \frac{288c\kappa_g^3L_{f,0}^2}{\lambda_1^2}\sum_{t=1}^T\frac{1}{t^{2\tau}} + \frac{C_\lambda\kappa_g^2}{1-\rho}\sum_{t=2}^T\left\|\vx_{t-1} - \vx_{t}\right\|^2 + \frac{C_\lambda}{1-\rho}\sum_{t=2}^T\left\|\vy_{t-1}^*(\vx_{t}) - \vy_{t}^*(\vx_{t})\right\|^2 \\
    \leq& c_\tau\Delta_1' + \frac{288c\kappa_g^3L_{f,0}^2}{\lambda_1^2}\sum_{t=1}^T\frac{1}{t^{2\tau}} + \frac{2\gamma^2C_\lambda\kappa_g^2}{1-\rho}\sum_{t=2}^T\left\|\widetilde\nabla \mathcal{L}_t^*(\vx_t) - \nabla F_t(\vx_t)\right\|^2 + \frac{2\gamma^2C_\lambda\kappa_g^2}{1-\rho}\mathrm{BLReg}(T) \\
    & + \frac{C_\lambda}{1-\rho}\sum_{t=2}^T\sup_{\vx\in\mathbb{R}^{d_1}}\left\|\vy_{t-1}^*(\vx) - \vy_{t}^*(\vx)\right\|^2 \\
    \leq& c_\tau\Delta_1' + \frac{288c\kappa_g^3L_{f,0}^2}{\lambda_1^2}\sum_{t=1}^T\frac{1}{t^{2\tau}} + \frac{4\gamma^2C_\lambda\kappa_g^2D_1^2}{\lambda_1^2(1-\rho)}\sum_{t=1}^T\frac{1}{t^{2\tau}} + \frac{4\gamma^2C_\lambda\kappa_g^2}{1-\rho}\sum_{t=1}^T\left\|\widetilde\nabla \mathcal{L}_t^*(\vx_t) - \nabla \mathcal{L}_t^*(\vx_t)\right\|^2 \\
    & + \frac{2\gamma^2C_\lambda\kappa_g^2}{1-\rho}\mathrm{BLReg}(T) + \frac{C_\lambda}{1-\rho}\sum_{t=2}^T\sup_{\vx\in\mathbb{R}^{d_1}}\left\|\vy_{t-1}^*(\vx) - \vy_{t}^*(\vx)\right\|^2,
\end{align*}
where the first inequality comes from the fact that $\sum_{j=0}^{t-2}\rho^j < \sum_{j=0}^\infty\rho^j = \frac{1}{1-\rho}$ for $\forall\rho \in(0,1)$, and the last inequality follows Lemma~\ref{lem:chen} Eq.(\ref{F-L:gradlip}).
Let $\gamma \leq \frac{1}{4\kappa_g}\sqrt{\frac{1-\rho}{C_\lambda}}$ such that $1 - \frac{4\gamma^2C_\lambda\kappa_g^2}{1-\rho} \geq \frac{1}{2}$, thus we complete the proof by
\begin{align*}
    \sum_{t=1}^T\left\|\widetilde\nabla \mathcal{L}_t^*(\vx_t) - \nabla \mathcal{L}_t^*(\vx_t)\right\|^2 
    \leq& 2c_\tau\Delta_1 + \frac{576c\kappa_g^3L_{f,0}^2}{\lambda_1^2}\sum_{t=1}^T\frac{1}{t^{2\tau}} + \frac{8\gamma^2C_\lambda\kappa_g^2D_1^2}{\lambda_1^2(1-\rho)}\sum_{t=1}^T\frac{1}{t^{2\tau}} \\
    & + \frac{4\gamma^2C_\lambda\kappa_g^2}{1-\rho}\mathrm{BLReg}(T) + \frac{2C_\lambda}{1-\rho}\sum_{t=2}^T\sup_{\vx\in\mathbb{R}^{d_1}}\left\|\vy_{t-1}^*(\vx) - \vy_{t}^*(\vx)\right\|^2,
\end{align*}
where $\Delta_1 := \Delta_1' + \|\widetilde\nabla \mathcal{L}_1^*(\vx_1) - \nabla \mathcal{L}_1^*(\vx_1)\|^2$.
\end{proof}

\begin{theorem}[Restatement of Theorem~\ref{thm:F$^2$OBO}]\label{prf:thm:F$^2$OBO}
    Under Assumptions~\ref{asm1},~\ref{asm2}, let $\Lambda(t, \lambda_t) = \left(1+\frac{1}{t}\right)^{\tau}\lambda_t$ for some $\tau>0$, $\lambda_1 > \frac{2L_{f,1}}{\mu_g}$, $\alpha\leq \frac{1}{L_{g,1}}$, $\beta_t \leq \frac{1}{2\lambda_tL_{g,1}}$, $\gamma_t \equiv \gamma \leq \min\{\frac{1}{2L_F}, \frac{1}{16\kappa_g}\sqrt{\frac{1-\rho}{C_\lambda}}\}$ and $K_t \equiv K \geq \max\{1 - \frac{\ln2}{\ln\rho}, \frac{\ln c - 2\tau\ln T}{\ln\rho}\}$ for some $c > 0$, Algorithm~\ref{alg:F$^2$OBO} can guarantee
    \begin{align*}
        \mathrm{BLReg}(T) \leq& \frac{32M}{\gamma} + \frac{16}{\gamma}V_T + \frac{32D_1^2}{\lambda_1^2}\sum_{t=1}^T\frac{1}{t^{2\tau}} + 64c_\tau\Delta_1 + \frac{18432c\kappa_g^3L_{f,0}^2}{\lambda_1^2}\sum_{t=1}^T\frac{1}{t^{2\tau}} \\
        & + \frac{256\gamma^2C_\lambda\kappa_g^2D_1^2}{\lambda_1^2(1-\rho)}\sum_{t=1}^T\frac{1}{t^{2\tau}} + \frac{64C_\lambda}{1-\rho}H_{2,T} \\
        =& O\left(\sum_{t=1}^T \frac{1}{t^{2\tau}} + V_T + H_{2,T}\right),
    \end{align*}
    where $\rho$, $c$, $c_\tau$, $C_\lambda$ and $\Delta_1$ are some constants.
\end{theorem}
\begin{proof}
We substitute Lemma~\ref{prf:lem:nabla_L} into Lemma~\ref{prf:lem:begin} to get
\begin{align*}
    \mathrm{BLReg}(T) \leq& \frac{16M}{\gamma} + \frac{8}{\gamma}V_T + \frac{16D_1^2}{\lambda_1^2}\sum_{t=1}^T\frac{1}{t^{2\tau}} + 32c_\tau\Delta_1 + \frac{9216c\kappa_g^3L_{f,0}^2}{\lambda_1^2}\sum_{t=1}^T\frac{1}{t^{2\tau}} + \frac{32C_\lambda}{1-\rho}H_{2,T} \\
    & + \frac{64\gamma^2C_\lambda\kappa_g^2}{1-\rho}\mathrm{BLReg}(T) + \frac{128\gamma^2C_\lambda\kappa_g^2D_1^2}{\lambda_1^2(1-\rho)}\sum_{t=1}^T\frac{1}{t^{2\tau}}.
\end{align*}
Let $\gamma \leq \frac{1}{16\kappa_g}\sqrt{\frac{1-\rho}{C_\lambda}}$ such that $1 - \frac{64\gamma^2C_\lambda\kappa_g^2}{1-\rho} \geq \frac{1}{2}$, thus finally have
\begin{align*}
    \mathrm{BLReg}(T) \leq& \frac{32M}{\gamma} + \frac{16}{\gamma}V_T + \frac{32D_1^2}{\lambda_1^2}\sum_{t=1}^T\frac{1}{t^{2\tau}} + 64c_\tau\Delta_1 + \frac{18432c\kappa_g^3L_{f,0}^2}{\lambda_1^2}\sum_{t=1}^T\frac{1}{t^{2\tau}} \\
    & + \frac{256\gamma^2C_\lambda\kappa_g^2D_1^2}{\lambda_1^2(1-\rho)}\sum_{t=1}^T\frac{1}{t^{2\tau}} + \frac{64C_\lambda}{1-\rho}H_{2,T} \\
    =& O\left(\sum_{t=1}^T \frac{1}{t^{2\tau}} + V_T + H_{2,T}\right).
\end{align*}
Our final bound depends on the choice of $\tau$:
\begin{align*}
\begin{cases}
    \mathrm{BLReg}(T) \leq O\left(1+V_T+H_{2,T}\right) &\quad \text{if set } \tau\in(\frac{1}{2},\infty) \\
    \mathrm{BLReg}(T) \leq O\left(\log T+V_T+H_{2,T}\right) &\quad \text{if set } \tau = \frac{1}{2} \\
    \mathrm{BLReg}(T) \leq O\left(T^{1-\tau}+V_T+H_{2,T}\right) &\quad \text{if set } \tau\in(0,\frac{1}{2})
\end{cases}.
\end{align*}
\end{proof}

\subsection{Proof of Theorem~\ref{thm:SF$^2$OBO}}\label{prf:sec:3.2}

In this section, we prove the regret upper bound of F$^2$OBO under $K_t\equiv1$. We present single-loop counterparts of Lemmas~\ref{lem:y-y} and~\ref{lem:z-z} in Lemmas~\ref{lem:sl_y-y} and~\ref{lem:sl_z-z}, respectively. Then, in Lemma~\ref{lem:sl_nabla_L-L}, we establish the final upper bound on the hypergradient approximation error, which completes the proof of Theorem~\ref{prf:thm:SF$^2$OBO}. 

\begin{lemma}\label{lem:sl_y-y}
    Under Assumptions~\ref{asm1}, \ref{asm2}, let $\beta_t = \frac{1}{2\lambda_tL_{g,1}}$ and $K_t\equiv1$, Algorithm~\ref{alg:F$^2$OBO} can obtain
    \begin{align*}
        &\left\|\vy_{t+1} - \vy_{\lambda_t,t}^*(\vx_t)\right\|^2 \\
        \leq& \rho_\vy^{t-1}\left\|\vy_2 {-} \vy_{\lambda_1,1}^*(\vx_1)\right\|^2 {+} 288\kappa_g^3\sum_{j=0}^{t-2}\rho_\vy^j\left\|\vx_{t-1-j} {-} \vx_{t-j}\right\|^2 {+} \frac{1536L_{f,0}^2\kappa_g^3}{\mu_g^2}\sum_{j=0}^{t-2}\frac{\rho_\vy^j(\lambda_{t-j} {-} \lambda_{t-1-j})^2}{\lambda_{t-j}^2\lambda_{t-1-j}^2} \\
        & + \frac{384\kappa_g}{\mu_g^2}\sum_{j=0}^{t-2}\frac{\rho_\vy^j}{\lambda_{t-j}^2}\sup_{\vx,\vy} \left\|\nabla_\vy f_{t-j}(\vx, \vy) - \nabla_\vy f_{t-1-j}(\vx, \vy) \right\|^2 \\
        & + \frac{384\kappa_g}{\mu_g^2}\sum_{j=0}^{t-2}\rho_\vy^j\sup_{\vx,\vy} \left\|\nabla_\vy g_{t-j}(\vx, \vy) - \nabla_\vy g_{t-1-j}(\vx, \vy) \right\|^2,
    \end{align*}
    where $\rho_\vy = 1 - \frac{1}{8\kappa_g}$.
\end{lemma}
\begin{proof}
We begin with~\eqref{eq:rho_y} in Lemma~\ref{lem:y-y}, it follows that
\begin{align}
    \left\|\vy_{t+1} - \vy_{\lambda_t,t}^*(\vx_t)\right\|^2 \leq& \left(1-\frac{\beta_t\lambda_t\mu_g}{2}\right)\left\|\vy_t - \vy_{\lambda_t,t}^*(\vx_t)\right\|^2 \nonumber\\
    \leq& \left(1-\frac{\beta_t\lambda_t\mu_g}{2}\right)(1+\theta)\left\|\vy_t - \vy_{\lambda_{t-1}, t-1}^*(\vx_{t-1})\right\|^2 \nonumber\\
    & + \left(1-\frac{\beta_t\lambda_t\mu_g}{2}\right) \left(1+\frac{1}{\theta}\right)\left\|\vy_{\lambda_{t-1},t-1}^*(\vx_{t-1}) - \vy_{\lambda_t,t}^*(\vx_t)\right\|^2 \nonumber\\
    \overset{(i)}{\leq}& \left(1-\frac{1}{8\kappa_g}\right)\left\|\vy_t - \vy_{\lambda_{t-1}, t-1}^*(\vx_{t-1})\right\|^2 \nonumber\\
    & + 16\kappa_g\left\|\vy_{\lambda_{t-1},t-1}^*(\vx_{t-1}) - \vy_{\lambda_t,t}^*(\vx_t)\right\|^2, \label{B.6_1}
\end{align}
where in (i), we set $\theta=\frac{\beta_t\lambda_t\mu_g}{4}$ such that $(1 - \frac{\beta_t\lambda_t\mu_g}{2})(1+\frac{\beta_t\lambda_t\mu_g}{4}) = 1 - \frac{\beta_t\lambda_t\mu_g}{4} - \frac{\beta_t^2\lambda_t^2\mu_g^2}{8} \leq 1 - \frac{\beta_t\lambda_t\mu_g}{4}$, and $\beta_t = \frac{1}{2\lambda_tL_{g,1}}$. Then it begins with
\begin{align}
    \left\|\vy_{\lambda_{t-1},t-1}^*(\vx_{t-1}) - \vy_{\lambda_t,t}^*(\vx_t)\right\|^2 \leq& 2\left\|\vy_{\lambda_{t-1},t-1}^*(\vx_{t-1}) - \vy_{\lambda_{t-1},t-1}^*(\vx_t)\right\|^2 \nonumber\\
    & + 2\left\|\vy_{\lambda_{t-1},t-1}^*(\vx_t) - \vy_{\lambda_t,t}^*(\vx_t)\right\|^2 \nonumber\\
    \overset{(i)}{\leq}& 18\kappa_g^2\left\|\vx_{t-1} - \vx_t\right\|^2 + 2\left\|\vy_{\lambda_{t-1},t-1}^*(\vx_t) - \vy_{\lambda_t,t}^*(\vx_t)\right\|^2, \label{B.6_2} 
\end{align}
where $(i)$ comes from Lemma 3.2 \cite{kwon2023fully}. Remind that $\mathcal{L}_{t}(\vx, \vy, \lambda)$ is $(\lambda_t\mu_g/2)$-strongly convex in $\vy$, the following formula can be obtained immediately
\begin{align}
    \left\|\vy_{\lambda_t,t}^*(\vx) - \vy_{\lambda_{t-1},t-1}^*(\vx)\right\| \leq \frac{2}{\lambda_t\mu_g}\left\|\nabla_\vy \mathcal{L}_t(\vx, \vy_{\lambda_{t-1},t-1}^*(\vx), \lambda_t)\right\| \label{B.6_3}
\end{align}
We also noticed that
\begin{align}
    \nabla_\vy\mathcal{L}_{t}(\vx, \vy_{\lambda_{t-1},t-1}^*(\vx), \lambda_t) {=}& \nabla_\vy f_t(\vx, \vy_{\lambda_{t-1},t-1}^*(\vx)) + \lambda_t\nabla_\vy g_t(\vx, \vy_{\lambda_{t-1},t-1}^*(\vx)) \label{B.6_4}\\
    \nabla_\vy\mathcal{L}_{t-1}(\vx,\vy_{\lambda_{t-1},t-1}^*(\vx), \lambda_{t-1}) {=}& \nabla_\vy f_{t-1}(\vx, \vy_{\lambda_{t-1},t-1}^*(\vx)) {+} \lambda_{t-1}\nabla_\vy g_{t-1}(\vx, \vy_{\lambda_{t-1},t-1}^*(\vx)) {=} 0, \label{B.6_5}
\end{align}
then substitute Eq.(\ref{B.6_5}) into Eq.(\ref{B.6_4}), we obtain
\begin{align}
    \nabla_\vy \mathcal{L}_{t}(\vx, \vy_{\lambda_{t-1},t-1}^*(\vx), \lambda_t) =& \nabla_\vy f_t(\vx, \vy_{\lambda_{t-1},t-1}^*(\vx)) - \nabla_\vy f_{t-1}(\vx, \vy_{\lambda_{t-1},t-1}^*(\vx)) \nonumber\\
    & + \lambda_t\left(\nabla_\vy g_t(\vx, \vy_{\lambda_{t-1},t-1}^*(\vx)) - \nabla_\vy g_{t-1}(\vx, \vy_{\lambda_{t-1},t-1}^*(\vx))\right) \nonumber\\
    & + \left(\lambda_{t} - \lambda_{t-1}\right)\nabla_\vy g_{t-1}(\vx, \vy_{\lambda_{t-1},t-1}^*(\vx)). \nonumber
\end{align}
With
\begin{align*}
    \left\|\nabla_\vy g_{t-1}(\vx, \vy_{\lambda_{t-1},t-1}^*(\vx))\right\| =& \left\|\nabla_\vy g_{t-1}(\vx, \vy_{\lambda_{t-1},t-1}^*(\vx)) - \nabla_\vy g_{t-1}(\vx, \vy_{t-1}^*(\vx))\right\| \\
    \leq& L_{g,1}\left\|\vy_{\lambda_{t-1},t-1}^*(\vx) - \vy_{t-1}^*(\vx)\right\| \leq \frac{2L_{f,0}L_{g,1}}{\lambda_{t-1}\mu_g},
\end{align*}
it naturally holds that
\begin{align}
    \left\|\nabla_\vy \mathcal{L}_{t}(\vx, \vy_{\lambda_{t-1},t-1}^*(\vx), \lambda_t)\right\|
    \leq& \left\|\nabla_\vy f_t(\vx, \vy_{\lambda_{t-1},t-1}^*(\vx)) - \nabla_\vy f_{t-1}(\vx, \vy_{\lambda_{t-1},t-1}^*(\vx))\right\| \nonumber\\
    & + \lambda_t\left\|\nabla_\vy g_t(\vx, \vy_{\lambda_{t-1},t-1}^*(\vx)) - \nabla_\vy g_{t-1}(\vx, \vy_{\lambda_{t-1},t-1}^*(\vx))\right\| \nonumber\\
    & + \frac{2L_{f,0}L_{g,1}}{\mu_g}\frac{\lambda_{t} - \lambda_{t-1}}{\lambda_{t-1}}. \label{B.6_6}
\end{align}
Finally we substitute Eq.(\ref{B.6_6}) into Eq.(\ref{B.6_3}) to get
\begin{align*}
    \left\|\vy_{\lambda_t,t}^*(\vx) - \vy_{\lambda_{t-1},t-1}^*(\vx)\right\| \leq& \frac{2}{\lambda_t\mu_g}\sup_{\vx,\vy} \left\|\nabla_\vy f_t(\vx, \vy) - \nabla_\vy f_{t-1}(\vx, \vy) \right\| \nonumber\\
    & + \frac{2}{\mu_g}\sup_{\vx,\vy} \left\|\nabla_\vy g_t(\vx, \vy) - \nabla_\vy g_{t-1}(\vx, \vy) \right\| \nonumber\\
    & + \frac{4L_{f,0}L_{g,1}}{\mu_g^2}\frac{\lambda_{t} - \lambda_{t-1}}{\lambda_t\lambda_{t-1}} 
\end{align*}
and natually have
\begin{align}
    \left\|\vy_{\lambda_t,t}^*(\vx)) - \vy_{\lambda_{t-1},t-1}^*(\vx))\right\|^2 \leq& \frac{12}{\lambda_t^2\mu_g^2}\sup_{\vx,\vy} \left\|\nabla_\vy f_t(\vx, \vy) - \nabla_\vy f_{t-1}(\vx, \vy) \right\|^2 \nonumber\\
    & + \frac{12}{\mu_g^2}\sup_{\vx,\vy} \left\|\nabla_\vy g_t(\vx, \vy) - \nabla_\vy g_{t-1}(\vx, \vy) \right\|^2 \nonumber\\
    & + \frac{48L_{f,0}^2L_{g,1}^2}{\mu_g^4}\frac{(\lambda_t - \lambda_{t-1})^2}{\lambda_t^2\lambda_{t-1}^2}. \label{B.6_7}
\end{align}
Substitute Eq.(\ref{B.6_7}) into Eq.(\ref{B.6_2}) to obtain
\begin{align}
    &\left\|\vy_{\lambda_{t-1},t-1}^*(\vx_{t-1}) - \vy_{\lambda_t,t}^*(\vx_t)\right\|^2 \nonumber\\
    \leq& 18\kappa_g^2\left\|\vx_{t-1} - \vx_t\right\|^2 + \frac{24}{\lambda_t^2\mu_g^2}\sup_{\vx,\vy} \left\|\nabla_\vy f_t(\vx, \vy) - \nabla_\vy f_{t-1}(\vx, \vy) \right\|^2 \nonumber\\
    & + \frac{24}{\mu_g^2}\sup_{\vx,\vy} \left\|\nabla_\vy g_t(\vx, \vy) - \nabla_\vy g_{t-1}(\vx, \vy) \right\|^2 + \frac{96L_{f,0}^2L_{g,1}^2}{\mu_g^4}\frac{(\lambda_t - \lambda_{t-1})^2}{\lambda_t^2\lambda_{t-1}^2}. \label{B.6_8}
\end{align}
Finally, substitute Eq.(\ref{B.6_8}) into Eq.(\ref{B.6_1}), we finish the proof by
\begin{align*}
    &\left\|\vy_{t+1} - \vy_{\lambda_t,t}^*(\vx_t)\right\|^2 \\
    \leq& \rho_\vy \left\|\vy_t - \vy_{\lambda_{t-1},t-1}^*(\vx_{t-1})\right\|^2 + 288\kappa_g^3\left\|\vx_{t-1} - \vx_t\right\|^2 + \frac{1536L_{f,0}^2\kappa_g^3}{\mu_g^2}\frac{(\lambda_t - \lambda_{t-1})^2}{\lambda_t^2\lambda_{t-1}^2} \\
    & + \frac{384\kappa_g}{\lambda_t^2\mu_g^2}\sup_{\vx,\vy} \left\|\nabla_\vy f_t(\vx, \vy) - \nabla_\vy f_{t-1}(\vx, \vy) \right\|^2+ \frac{384\kappa_g}{\mu_g^2}\sup_{\vx,\vy} \left\|\nabla_\vy g_t(\vx, \vy) - \nabla_\vy g_{t-1}(\vx, \vy) \right\|^2 \\
    \leq& \rho_\vy^{t-1}\left\|\vy_2 {-} \vy_{\lambda_1,1}^*(\vx_1)\right\|^2 {+} 288\kappa_g^3\sum_{j=0}^{t-2}\rho_\vy^j\left\|\vx_{t-1-j} {-} \vx_{t-j}\right\|^2 {+} \frac{1536L_{f,0}^2\kappa_g^3}{\mu_g^2}\sum_{j=0}^{t-2}\frac{\rho_\vy^j(\lambda_{t-j} {-} \lambda_{t-1-j})^2}{\lambda_{t-j}^2\lambda_{t-1-j}^2} \\
    & + \frac{384\kappa_g}{\mu_g^2}\sum_{j=0}^{t-2}\frac{\rho_\vy^j}{\lambda_{t-j}^2}\sup_{\vx,\vy} \left\|\nabla_\vy f_{t-j}(\vx, \vy) - \nabla_\vy f_{t-1-j}(\vx, \vy) \right\|^2 \\
    & + \frac{384\kappa_g}{\mu_g^2}\sum_{j=0}^{t-2}\rho_\vy^j\sup_{\vx,\vy} \left\|\nabla_\vy g_{t-j}(\vx, \vy) - \nabla_\vy g_{t-1-j}(\vx, \vy) \right\|^2,
\end{align*}
where $\rho_\vy := 1 - \frac{1}{8\kappa_g}$.
\end{proof}

\begin{lemma}\label{lem:sl_z-z}
    Under Assumptions~\ref{asm1}, \ref{asm2}, let $\alpha \leq \frac{1}{L_{g,1}}$ and $K_t\equiv1$, Algorithm~\ref{alg:F$^2$OBO} can obtain
    \begin{align*}
        \left\|\vz_{t+1} - \vy_t^*(\vx_t)\right\|^2 \leq& \rho_\vz^{t-1}\left\|\vz_2 - \vy_1^*(\vx_1)\right\|^2 + \frac{8\kappa_g^2}{\alpha\mu_g}\sum_{j=0}^{t-2}\rho_\vz^j\left\|\vx_{t-1-j} - \vx_{t-j}\right\|^2 \\
        & + \frac{8}{\alpha\mu_g}\sum_{j=0}^{t-2}\rho_\vz^j\sup_{\vx}\left\|\vy_{t-1-j}^*(\vx) - \vy_{t-j}^*(\vx)\right\|^2,
    \end{align*}
    where $\rho_\vz = 1 - \frac{\alpha\mu_g}{2}$.
\end{lemma}
\begin{proof}
It follows that
\begin{align*}
    &\left\|\vz_{t+1} - \vy_t^*(\vx_t)\right\|^2 \leq \left(1 - \alpha\mu_g\right)\left\|\vz_t - \vy_t^*(\vx_t)\right\|^2 \\
    \leq& \left(1 - \alpha\mu_g\right)(1+\theta)\left\|\vz_t - \vy_{t-1}^*(\vx_{t-1})\right\|^2 + \left(1 - \alpha\mu_g\right)\left(1+\frac{1}{\theta}\right)\left\|\vy_{t-1}^*(\vx_{t-1}) - \vy_t^*(\vx_t)\right\|^2 \\
    \leq& \left(1 - \frac{\alpha\mu_g}{2}\right)\left\|\vz_t - \vy_{t-1}^*(\vx_{t-1})\right\|^2 + \frac{8\kappa_g^2}{\alpha\mu_g}\left\|\vx_{t-1} - \vx_t\right\|^2 + \frac{8}{\alpha\mu_g}\sup_{\vx}\left\|\vy_{t-1}^*(\vx) - \vy_t^*(\vx)\right\|^2 \\
    \leq& \rho_\vz^{t-1}\left\|\vz_2 - \vy_1^*(\vx_1)\right\|^2 + \frac{8\kappa_g^2}{\alpha\mu_g}\sum_{j=0}^{t-2}\rho_\vz^j\left\|\vx_{t-1-j} - \vx_{t-j}\right\|^2 \\
    & + \frac{8}{\alpha\mu_g}\sum_{j=0}^{t-2}\rho_\vz^j\sup_{\vx}\left\|\vy_{t-1-j}^*(\vx) - \vy_{t-j}^*(\vx)\right\|^2,
\end{align*}
where in the second inequality, we set $\theta = \frac{\alpha\mu_g}{2}$ such that $(1 - \alpha\mu_g)(1 + \frac{\alpha\mu_g}{2}) \leq 1 - \frac{\alpha\mu_g}{2}$, and $\rho_\vz := 1 - \frac{\alpha\mu_g}{2}$.
\end{proof}

\begin{lemma}\label{lem:sl_nabla_L-L}
Under Assumptions~\ref{asm1} and~\ref{asm2}, let $\Lambda(t,\lambda_t) = \left(1+\frac{1}{t}\right)^\tau\lambda_t$ for some $\tau>0$, $\lambda_1>\frac{2L_{f,1}}{\mu_g}$, $\alpha\leq\frac{1}{L_{g,1}}$, $\beta_t = \frac{1}{2\lambda_tL_{g,1}}$, $\gamma_t \leq \frac{1}{\lambda_1t^\tau}\sqrt{\frac{1-\rho}{6C_g}}$ and $K_t\equiv1$, Algorithm~\ref{alg:F$^2$OBO} can obtain
\begin{align*}
    \sum_{t=1}^T\left\|\widetilde\nabla \mathcal{L}_t^*(\vx_t) - \nabla \mathcal{L}_t^*(\vx_t)\right\|^2 \leq& 2c_\tau\Delta_1 + 2c_\lambda C_S + \frac{6c_\gamma C_g}{1-\rho}\mathrm{BLReg}(T) + \frac{6c_\gamma C_gD_1^2}{\lambda_1^2(1-\rho)}\sum_{t=1}^T\frac{1}{t^{2\tau}} \\
    & + \frac{2c_g'\lambda_T^2}{1-\rho_\vz}H_{2,T} + \frac{2c_g\lambda_T^2}{1-\rho_\vy}E_{\vy,T}^f + \frac{2c_g\lambda_T^2}{1-\rho_\vy}E_{\vy,T}^g,
\end{align*}
where $\rho = \max\{\rho_\vy, \rho_\vz\}$, $\rho_\vy = 1 - \frac{1}{8\kappa_g}$, $\rho_\vz = 1 - \frac{\alpha\mu_g}{2}$, $\Delta_1$, $c_\tau$, $c_\lambda$ $C_S$, $c_g$, $c_g'$, $C_g$ are some constants.
\end{lemma}
\begin{proof}
We begin with~\eqref{a+b} in Lemma~\ref{prf:lem:nabla_L}, which gives
\begin{align}
    \left\|\widetilde\nabla \mathcal{L}_t^*(\vx_t) - \nabla \mathcal{L}_t^*(\vx_t)\right\|^2 \leq& 6\lambda_t^2L_{g,1}^2\left\|\vy_{t+1} - \vy_{\lambda_t,t}^*(\vx_t)\right\|^2 + 3\lambda_t^2L_{g,1}^2\left\|\vz_{t+1} - \vy_t^*(\vx_t)\right\|^2. \nonumber
\end{align}
Substituting the bounds from Lemma~\ref{lem:sl_y-y} and Lemma~\ref{lem:sl_z-z} into the above inequality yields
\begin{align}
    \left\|\widetilde\nabla \mathcal{L}_t^*(\vx_t) - \nabla \mathcal{L}_t^*(\vx_t)\right\|^2 
    \leq& 6\lambda_t^2\rho_\vy^{t-1}L_{g,1}^2\left\|\vy_2 {-} \vy_{\lambda,1}^*(\vx_1)\right\|^2 + 3\lambda_t^2\rho_\vz^{t-1}L_{g,1}^2\left\|\vz_2 - \vy_1^*(\vx_1)\right\|^2 \nonumber\\
    & + c_\lambda\lambda_t^2\sum_{j=0}^{t-2}\frac{\rho_\vy^j(\lambda_{t-j} - \lambda_{t-1-j})^2}{\lambda_{t-j}^2\lambda_{t-1-j}^2} + C_g\lambda_t^2\sum_{j=0}^{t-2}\rho^j \left\|\vx_{t-1-j} - \vx_{t-j}\right\|^2 \nonumber\\
    & + c_g\lambda_t^2\sum_{j=0}^{t-2}\rho_\vy^j \sup_{\vx,\vy} \left\|\nabla_\vy f_{t-j}(\vx, \vy) - \nabla_\vy f_{t-1-j}(\vx, \vy) \right\|^2 \nonumber\\
    & + c_g\lambda_t^2\sum_{j=0}^{t-2}\rho_\vy^j \sup_{\vx,\vy} \left\|\nabla_\vy g_{t-j}(\vx, \vy) - \nabla_\vy g_{t-1-j}(\vx, \vy) \right\|^2 \nonumber\\
    & + c_g'\lambda_t^2\sum_{j=0}^{t-2}\rho_\vz^j \sup_\vx\|\vy_{t-1-j}^*(\vx) - \vy_{t-j}^*(\vx)\|^2, \nonumber
\end{align}
where, for notational simplicity, we define
\[
\begin{aligned}
c_\lambda &:= 6L_{g,1}^2\cdot\frac{1536L_{f,0}^2\kappa_g^3}{\mu_g^2},\\
C_g &:= 6L_{g,1}^2\cdot288\kappa_g^3 
    + 3L_{g,1}^2\cdot\frac{8\kappa_g^2}{\alpha\mu_g},\\
c_g &:= 6L_{g,1}^2\cdot\frac{384\kappa_g}{\mu_g},\\
c_g' &:= 3L_{g,1}^2\cdot\frac{8}{\alpha\mu_g},\\
\rho &:= \max\{\rho_\vy,\rho_\vz\},
\end{aligned}
\]
and have
\begin{align}
    &\sum_{t=1}^T\left\|\widetilde\nabla \mathcal{L}_t^*(\vx_t) - \nabla \mathcal{L}_t^*(\vx_t)\right\|^2 \nonumber\\
    \leq& c_\tau\Delta_1 + c_\lambda\sum_{t=1}^T\lambda_t^2\sum_{j=0}^{t-2}\frac{\rho_\vy^j(\lambda_{t-j} - \lambda_{t-1-j})^2}{\lambda_{t-j}^2\lambda_{t-1-j}^2} + C_g\sum_{t=1}^T\lambda_t^2\sum_{j=0}^{t-2}\rho^j\left\|\vx_{t-1-j} - \vx_{t-j}\right\|^2 \nonumber\\
    & + c_g\sum_{t=1}^T\lambda_t^2\sum_{j=0}^{t-2}\rho_\vy^j\sup_{\vx,\vy} \left\|\nabla_\vy f_{t-j}(\vx, \vy) {-} \nabla_\vy f_{t-1-j}(\vx, \vy) \right\|^2 \nonumber\\
    & + c_g\sum_{t=1}^T\lambda_t^2\sum_{j=0}^{t-2}\rho_\vy^j\sup_{\vx,\vy} \left\|\nabla_\vy g_{t-j}(\vx, \vy) - \nabla_\vy g_{t-1-j}(\vx, \vy) \right\|^2  \nonumber\\
    & + c_g'\sum_{t=1}^T\lambda_t^2\sum_{j=0}^{t-2}\rho_\vz^j\sup_\vx\|\vy_{t-1-j}^*(\vx) - \vy_{t-j}^*(\vx)\|^2, \label{B.8_1}
\end{align}
where $\Delta_1 := 6L_{g,1}^2\|\vy_2 - \vy_{\lambda_1,1}^*(\vx_1)\|^2 + 3L_{g,1}^2\|\vz_2 - \vy_1^*(\vx_1)\|^2 + \|\widetilde\nabla \mathcal{L}_1^*(\vx_1) - \nabla \mathcal{L}_1^*(\vx_1)\|^2$. For the second term in Eq.(\ref{B.8_1}), we define $S(T)$ and let $m=t-j$, it follows that
\begin{align*}
    S(T) := \sum_{t=2}^T\lambda_t^2\sum_{j=0}^{t-2}\rho_\vy^j\frac{(\lambda_{t-j} - \lambda_{t-1-j})^2}{\lambda_{t-j}^2\lambda_{t-1-j}^2} = \sum_{t=2}^T\lambda_t^2\sum_{m=2}^t\rho_\vy^{t-m}\frac{(\lambda_m - \lambda_{m-1})^2}{\lambda_m^2\lambda_{m-1}^2},
\end{align*}
with $\lambda_m - \lambda_{m-1} = \lambda_1(m^\tau - (m-1)^\tau) = \lambda_1\tau\xi^{\tau-1} \leq \lambda_1\tau(m-1)^{\tau-1}$ for some $\xi\in[m-1,m]$, it holds that
\begin{align*}
    \frac{(\lambda_m - \lambda_{m-1})^2}{\lambda_m^2\lambda_{m-1}^2} \leq \frac{\lambda_1^2\tau^2(m-1)^{2\tau-2}}{\lambda_1^4m^{2\tau}(m-1)^{2\tau}} \leq \frac{\tau^2}{\lambda_1^2}\frac{1}{(m-1)^2m^{2\tau}} \approx \frac{\tau^2}{\lambda_1^2}\frac{1}{m^{2\tau+2}}.
\end{align*}
Thus we have
\begin{align*}
    S(T) = \frac{\tau^2}{\lambda_1^2}\sum_{t=2}^T\lambda_t^2\sum_{m=2}^t\frac{\rho_\vy^{t-m}}{m^{2\tau+2}} = \tau^2\sum_{t=2}^Tt^{2\tau}\sum_{m=2}^t\frac{\rho_\vy^{t-m}}{m^{2\tau+2}} = \tau^2\sum_{t=2}^Tt^{2\tau}\sum_{j=0}^{t-2}\frac{\rho^j}{(t-j)^{2\tau+2}},
\end{align*}
for $0 \leq j \leq \left\lfloor \frac{t}{2} \right\rfloor$, 
\begin{align*}
    \sum_{j=0}^{\left\lfloor \frac{t}{2} \right\rfloor}\frac{\rho_\vy^j}{(t-j)^{2\tau+2}} \leq \left(\frac{2}{t}\right)^{2\tau+2}\sum_{j=0}^\infty \rho_\vy^j \leq  \left(\frac{2}{t}\right)^{2\tau+2}\frac{1}{1-\rho_\vy},
\end{align*}
for $\left\lfloor \frac{t}{2} \right\rfloor + 1 \leq j \leq t-2$, 
\begin{align*}
    \sum_{j=\left\lfloor \frac{t}{2} \right\rfloor + 1}^{t-2}\frac{\rho_\vy^j}{(t-j)^{2\tau+2}} \leq \frac{1}{2^{2\tau+2}}\sum_{j=\left\lfloor \frac{t}{2} \right\rfloor + 1}^{t-2}\rho_\vy^j \leq \frac{1}{2^{2\tau+2}} \cdot \frac{\rho_\vy^{\left\lfloor \frac{t}{2} \right\rfloor + 1}}{1-\rho_\vy}.
\end{align*}
Finally, we have
\begin{align*}
    S(T) \leq \tau^2\sum_{t=2}^T\left(\frac{2^{2\tau+2}}{1-\rho_\vy}\cdot\frac{1}{t^2} + \frac{\rho_\vy^{\left\lfloor \frac{t}{2} \right\rfloor + 1}t^{2\tau}}{2^{2\tau+2}(1-\rho_\vy)}\right) =: C_S,
\end{align*}
for any $\tau>0$, $C_S = O(1)$. Thus it holds that
\begin{align*}
    &\sum_{t=1}^T\left\|\widetilde\nabla \mathcal{L}_t^*(\vx_t) - \nabla \mathcal{L}_t^*(\vx_t)\right\|^2 \\
    \leq& c_\tau\Delta_1 + c_\lambda C_S + \frac{C_g}{1-\rho}\sum_{t=1}^T\lambda_t^2\left\|\vx_{t-1} - \vx_t\right\|^2 + \frac{c_g'\lambda_T^2}{1-\rho_\vz} \sum_{t=1}^T\sup_{\vx\in\mathcal{X}}\left\|\vy_{t-1}^*(\vx) - \vy_t^*(\vx)\right\|^2 \\
    & + \frac{c_g\lambda_T^2}{1-\rho_\vy}\sum_{t=1}^T \sup_{\vx,\vy} \left\|\nabla_\vy f_t(\vx, \vy) - \nabla_\vy f_{t-1}(\vx, \vy) \right\|^2 \\
    & + \frac{c_g\lambda_T^2}{1-\rho_\vy}\sum_{t=1}^T \sup_{\vx,\vy} \left\|\nabla_\vy g_t(\vx, \vy) - \nabla_\vy g_{t-1}(\vx, \vy) \right\|^2. \\
    \leq& c_\tau\Delta_1 + c_\lambda C_S + \frac{C_g}{1-\rho}\sum_{t=1}^T\lambda_t^2\gamma_t^2\left\|\mathcal{G}_\mathcal{X}\left(\vx_t, \widetilde\nabla \mathcal{L}_t^*(\vx_t), \gamma_t\right)\right\|^2 \\
    & + \frac{c_g'\lambda_T^2}{1-\rho_\vz} \sum_{t=1}^T\sup_{\vx\in\mathcal{X}}\left\|\vy_{t-1}^*(\vx) - \vy_t^*(\vx)\right\|^2 \\
    & + \frac{c_g\lambda_T^2}{1-\rho_\vy}\sum_{t=1}^T \sup_{\vx,\vy} \left\|\nabla_\vy f_t(\vx, \vy) - \nabla_\vy f_{t-1}(\vx, \vy) \right\|^2 \\
    & + \frac{c_g\lambda_T^2}{1-\rho_\vy}\sum_{t=1}^T \sup_{\vx,\vy} \left\|\nabla_\vy g_t(\vx, \vy) - \nabla_\vy g_{t-1}(\vx, \vy) \right\|^2 \\
    \leq& c_\tau\Delta_1 + c_\lambda C_S + \frac{C_g}{1-\rho}\sum_{t=1}^T\lambda_t^2\gamma_t^2\left(3\left\|\mathcal{G}_\mathcal{X}\left(\vx_t, \nabla F_t(\vx_t), \gamma_t\right)\right\|^2 + 3\left\|\nabla F_t(\vx_t) - \nabla\mathcal{L}_t^*(\vx_t)\right\|^2 \right. \\
    &\left. + \left\|\nabla\mathcal{L}_t^*(\vx_t) - \widetilde\nabla \mathcal{L}_t^*(\vx_t)\right\|^2\right) + \frac{c_g'\lambda_T^2}{1-\rho_\vz} \sum_{t=1}^T\sup_{\vx\in\mathcal{X}}\left\|\vy_{t-1}^*(\vx) - \vy_t^*(\vx)\right\|^2 \\
    & + \frac{c_g\lambda_T^2}{1-\rho_\vy}\sum_{t=1}^T \sup_{\vx,\vy} \left\|\nabla_\vy f_t(\vx, \vy) - \nabla_\vy f_{t-1}(\vx, \vy) \right\|^2 \\
    & + \frac{c_g\lambda_T^2}{1-\rho_\vy}\sum_{t=1}^T \sup_{\vx,\vy} \left\|\nabla_\vy g_t(\vx, \vy) - \nabla_\vy g_{t-1}(\vx, \vy) \right\|^2.
\end{align*}
Now suppose that for sufficiently small $\gamma_t$ such that $\lambda_t^2\gamma_t^2 \leq c_\gamma$ for some constant $c_\gamma > 0$, we obtain 
\begin{align*}
    &\sum_{t=1}^T\left\|\widetilde\nabla \mathcal{L}_t^*(\vx_t) - \nabla \mathcal{L}_t^*(\vx_t)\right\|^2 \\
    \leq& c_\tau\Delta_1 + c_\lambda C_S + \frac{3c_\gamma C_g}{1-\rho}\mathrm{BLReg}(T) + \frac{3c_\gamma C_gD_1^2}{\lambda_1^2(1-\rho)}\sum_{t=1}^T\frac{1}{t^{2\tau}} + \frac{3c_\gamma C_g}{1-\rho}\sum_{t=1}^T\left\|\widetilde\nabla \mathcal{L}_t^*(\vx_t) {-} \nabla \mathcal{L}_t^*(\vx_t)\right\|^2 \\
    & + \frac{c_g'\lambda_T^2}{1-\rho_\vz}H_{2,T} + \frac{c_g\lambda_T^2}{1-\rho_\vy}E_{\vy,T}^f + \frac{c_g\lambda_T^2}{1-\rho_\vy}E_{\vy,T}^g.
\end{align*}
Let $1 - \frac{3c_\gamma C_g}{1-\rho} \geq \frac{1}{2}$, which means
\begin{align*}
    c_\gamma \leq& \frac{1-\rho}{6C_g} \\
    \gamma \leq& \sqrt{\frac{1-\rho}{6\lambda_t^2C_g}} = \frac{1}{\lambda_1t^\tau}\sqrt{\frac{1-\rho}{6C_g}},
\end{align*}
finally we finish the proof that 
\begin{align*}
    \sum_{t=1}^T\left\|\widetilde\nabla \mathcal{L}_t^*(\vx_t) - \nabla \mathcal{L}_t^*(\vx_t)\right\|^2 \leq& 2c_\tau\Delta_1 + 2c_\lambda C_S + \frac{6c_\gamma C_g}{1-\rho}\mathrm{BLReg}(T) + \frac{6c_\gamma C_gD_1^2}{\lambda_1^2(1-\rho)}\sum_{t=1}^T\frac{1}{t^{2\tau}} \\
    & + \frac{2c_g'\lambda_T^2}{1-\rho_\vz}H_{2,T} + \frac{2c_g\lambda_T^2}{1-\rho_\vy}E_{\vy,T}^f + \frac{2c_g\lambda_T^2}{1-\rho_\vy}E_{\vy,T}^g.
\end{align*}
\end{proof}

\begin{theorem}[Restatement of Theorem~\ref{thm:SF$^2$OBO}]\label{prf:thm:SF$^2$OBO}
    Under Assumptions~\ref{asm1},~\ref{asm2}, let $\Lambda(t, \lambda_t) = \left(1+\frac{1}{t}\right)^{\tau}\lambda_t$ for some $\tau\in(0,\frac{1}{2})$, $\lambda_1 > \frac{2L_{f,1}}{\mu_g}$, $\alpha\leq \frac{1}{L_{g,1}}$, $\beta_t = \frac{1}{2\lambda_tL_{g,1}}$, $\gamma_t < \min\{\frac{1}{2L_F}, \frac{1}{8\lambda_1t^\tau}\sqrt{\frac{1-\rho}{3C_g}}\}$ and $K_t\equiv1$, Algorithm~\ref{alg:F$^2$OBO} can guarantee
    \begin{align*}
        \mathrm{BLReg}(T) 
        \leq O\left(T^\tau(1 + V_T) + T^{1-2\tau} + T^{2\tau}(H_{2,T} + E_{\vy,T}^f + E_{\vy,T}^g)\right),
    \end{align*}
    where $\rho$, $C_g$ are some constants and
    \begin{align*}
        &E_{\vy,T}^f = \sum_{t=2}^T\sup_{\vx,\vy}\left\|\nabla_\vy f_t(\vx, \vy) - \nabla_\vy f_{t-1}(\vx, \vy)\right\|^2, \\
        &E_{\vy,T}^g = \sum_{t=2}^T\sup_{\vx, \vy}\left\|\nabla_\vy g_t(\vx, \vy) - \nabla_\vy g_{t-1}(\vx, \vy)\right\|^2.
    \end{align*}
\end{theorem}
\begin{proof}
After substituting Lemma~\ref{lem:sl_nabla_L-L} into Lemma~\ref{prf:lem:begin}, we obtain
\begin{align*}
    \mathrm{BLReg}(T) \leq& \frac{16M}{\gamma_T} + \frac{8}{\gamma_T}V_T + 16D_1^2\sum_{t=1}^T\frac{1}{\lambda_t^2} + 32c_\tau\Delta_1 + 32c_\lambda C_S + \frac{96c_\gamma C_g}{1-\rho}\mathrm{BLReg}(T) \\
    & + \frac{96c_\gamma C_gD_1^2}{\lambda_1^2(1-\rho)}\sum_{t=1}^T\frac{1}{t^{2\tau}} + \frac{32c_g'\lambda_T^2}{1-\rho_\vz}H_{2,T} + \frac{32c_g\lambda_T^2}{1-\rho_\vy}E_{\vy,T}^f + \frac{32c_g\lambda_T^2}{1-\rho_\vy}E_{\vy,T}^g.
\end{align*}
Let $1 - \frac{96c_\gamma C_g}{1-\rho} \geq \frac{1}{2}$, $c_\gamma \leq \frac{1-\rho}{192C_g}$, which means
\begin{align*}
    \lambda_t^2\gamma_t^2 \leq& \frac{1-\rho}{192C_g} \\
    \gamma_t \leq& \frac{1}{8\lambda_1t^\tau}\sqrt{\frac{1-\rho}{3C_g}}.
\end{align*}
Thus we finally have
\begin{align*}
    \mathrm{BLReg}(T) \leq& \frac{16M}{\gamma_T} + \frac{8}{\gamma_T}V_T + \frac{16D_1^2}{\lambda_1^2}\sum_{t=1}^T\frac{1}{t^{2\tau}} + 32c_\tau\Delta_1 + 32c_\lambda C_S + \frac{96c_\gamma C_gD_1^2}{\lambda_1^2(1-\rho)}\sum_{t=1}^T\frac{1}{t^{2\tau}} \\
    & + \frac{32c_g'\lambda_T^2}{1-\rho_\vz}H_{2,T} + \frac{32c_g\lambda_T^2}{1-\rho_\vy}E_{\vy,T}^f + \frac{32c_g\lambda_T^2}{1-\rho_\vy}E_{\vy,T}^g \\
    =& O\left(T^\tau(1 + V_T) + T^{1-2\tau} + T^{2\tau}(H_{2,T} + E_{\vy,T}^f + E_{\vy,T}^g\right)).
\end{align*}
\end{proof}

\section{Proof of Section~\ref{sec:5}}

In this section, we present a regret upper bound proof for AF$^2$OBO. AF$^2$OBO improves the upper bound of the hypergradient error for the approximation problem (\ref{P3}) by using a dynamic inner iteration technique, which eliminates the dependence on $H_{2,T}$, as shown in Lemma~\ref{prf:lem:nabla_L_3} below. Finally, following the same proof process as F$^2$OBO, we give the proof in Theorem~\ref{prf:thm:AF$^2$OBO}.

\begin{lemma}\label{prf:lem:nabla_L_3}
Under Assumptions~\ref{asm1} and~\ref{asm2}, let $\Lambda(t, \lambda_t) = \left(1+\frac{1}{t}\right)^{\tau}\lambda_t$ for some $\tau>0$, $\lambda_1 > \frac{2L_{f,1}}{\mu_g}$, $\delta_\vy = \frac{1}{\sqrt{T}}$ and $\delta_\vz = \frac{1}{\sqrt{T^{1+2\tau}}}$, the approximation error of hypergradient in Algorithm~\ref{alg:AF$^2$OBO} can be bounded as
\begin{align*}
    \sum_{t=1}^T
    \left\|\widetilde\nabla \mathcal{L}_t^*(\vx_t)
    - \nabla \mathcal{L}_t^*(\vx_t)\right\|^2
    \leq 24\kappa_g^2 + 3\lambda_1^2\kappa_g^2.
\end{align*}
\end{lemma}
\begin{proof}
We first bound the hypergradient approximation error as follows:
\begin{align}
    \left\|\widetilde\nabla \mathcal{L}_t^*(\vx_t)
    - \nabla \mathcal{L}_t^*(\vx_t)\right\|^2
    \leq {}& 6\lambda_t^2L_{g,1}^2
    \left\|\vy_{t+1} - \vy_{\lambda_t,t}^*(\vx_t)\right\|^2
    + 3\lambda_t^2L_{g,1}^2
    \left\|\vz_{t+1} - \vy_t^*(\vx_t)\right\|^2 \nonumber\\
    \leq {}& 6\lambda_t^2L_{g,1}^2
    \frac{4\delta_\vy^2}{\lambda_t^2\mu_g^2}
    + 3\lambda_t^2L_{g,1}^2
    \frac{\delta_\vz^2}{\mu_g^2} \nonumber\\
    ={}& 24\kappa_g^2\delta_\vy^2
    + 3\lambda_t^2\kappa_g^2\delta_\vz^2 .
\end{align}
The second inequality follows from the error criteria
\begin{align*}
    \frac{\lambda_t\mu_g}{2}
    \left\|\vy_{t+1} - \vy_{\lambda_t,t}^*(\vx_t)\right\|
    \leq
    \left\|\nabla_\vy\mathcal{L}_t(\vx_t,\vy_{t+1},\lambda_t)\right\|
    \leq \delta_\vy,
\end{align*}
and
\begin{align*}
    \mu_g
    \left\|\vz_{t+1} - \vy_t^*(\vx_t)\right\|
    \leq
    \left\|\nabla_\vy g_t(\vx_t,\vz_{t+1})\right\|
    \leq \delta_\vz,
\end{align*}
which are implied by the $\mu_g$-strong convexity of $g_t$ and the
$(\lambda_t\mu_g/2)$-strong convexity of $\mathcal{L}_t$, respectively.
Therefore, summing over $t=1,\ldots,T$ gives
\begin{align*}
    \sum_{t=1}^T
    \left\|\widetilde\nabla \mathcal{L}_t^*(\vx_t)
    - \nabla \mathcal{L}_t^*(\vx_t)\right\|^2
    \leq
    24\kappa_g^2\delta_\vy^2T
    + 3\lambda_T^2\kappa_g^2\delta_\vz^2T .
\end{align*}
By choosing $\delta_\vy = \frac{1}{\sqrt{T}}$ and $\delta_\vz = \frac{1}{\sqrt{T^{1+2\tau}}}$, we obtain
\begin{align*}
    \sum_{t=1}^T
    \left\|\widetilde\nabla \mathcal{L}_t^*(\vx_t)
    - \nabla \mathcal{L}_t^*(\vx_t)\right\|^2
    \leq 24\kappa_g^2 + 3\lambda_1^2\kappa_g^2.
\end{align*}
\end{proof}

\begin{theorem}[Restatement of Theorem~\ref{thm:AF$^2$OBO}]\label{prf:thm:AF$^2$OBO}
    Under Assumptions~\ref{asm1},~\ref{asm2}, let $\Lambda(t, \lambda_t) = \left(1+\frac{1}{t}\right)^{\tau}\lambda_t$ for some $\tau>0$, $\lambda_1 > \frac{2L_{f,1}}{\mu_g}$, $\alpha\leq \frac{1}{L_{g,1}}$, $\beta_t \leq \frac{1}{2\lambda_tL_{g,1}}$, $\gamma \leq \frac{1}{2L_F}$ for all $t\in[T]$, $\delta_\vy = \frac{1}{\sqrt{T}}$ and $\delta_\vz = \frac{1}{\sqrt{T^{1+2\tau}}}$, Algorithm~\ref{alg:AF$^2$OBO} can guarantee
    \begin{align*}
        \mathrm{BLReg}(T) \leq \frac{16M}{\gamma} + \frac{8}{\gamma}V_T + \frac{16D_1^2}{\lambda_1^2}\sum_{t=1}^T\frac{1}{t^{2\tau}} + 384\kappa_g^2 + 48\lambda_1^2\kappa_g^2 = O\left(\sum_{t=1}^T\frac{1}{t^{2\tau}} + V_T\right).
    \end{align*}
    The total number of inner iterations $\mathcal{I}_T$ satisfies
    \begin{align*}
        \mathcal{I}_T \leq O\left(T^{1+2\tau} + T^{2\tau}H_{2,T}\right).
    \end{align*}
\end{theorem}
\begin{proof}
We begins with Lemma~\ref{prf:lem:begin} and substitute Lemma~\ref{prf:lem:nabla_L_3} into it to obtain
\begin{align*}
    \mathrm{BLReg}(T) \leq& \frac{16M}{\gamma} + \frac{8}{\gamma}V_T + \frac{16D_1^2}{\lambda_1^2}\sum_{t=1}^T\frac{1}{t^{2\tau}} + 384\kappa_g^2 + 48\lambda_1^2\kappa_g^2 = O\left(\sum_{t=1}^T\frac{1}{t^{2\tau}} + V_T\right),
\end{align*}
where $\gamma_t \equiv \gamma \leq \frac{1}{2L_F}$, $\delta_\vy = \frac{1}{\sqrt{T}}$ and $\delta_\vz = \frac{1}{\sqrt{T^{1+2\tau}}}$. Now consider the two inner-loop iteration process. The sequences $\{\vy_t^k\}_{k=1}^{K_t}$ and $\{\vz_t^m\}_{m=1}^{M_t}$ are generated by
\begin{align*}
    \vy_t^{k+1} \leftarrow& \vy_t^k - \beta_t\left(\nabla_\vy f_t(\vx_t, \vy_t^k) + \lambda_t\nabla_\vy g_t(\vx_t, \vy_t^k)\right), \quad \text{and} \\
    \vz_t^{m+1} \leftarrow& \vz_t^m - \alpha\nabla_\vy g_t(\vx_t, \vz_t^m),
\end{align*}
respectively. 
For $\{\vy_t^k\}_{k=1}^{K_t}$, it holds that
\begin{align}
    \frac{\delta_\vy^2}{4\lambda_t^2L_{g,1}^2} \leq \left\|\vy_t^{K_t} - \vy_{\lambda_t,t}^*(\vx_t)\right\|^2 \leq \rho_\vy^{K_t-1}\left\|\vy_t^1 - \vy_{\lambda_t,t}^*(\vx_t)\right\|^2, \label{B.2_1}
\end{align}
where $\rho_\vy\in(0,1)$. The second inequality follows from~\eqref{eq:rho_y} in Lemma~\ref{lem:y-y} under the stepsize condition $\beta_t \leq 1/(2\lambda_t L_{g,1})$ for all $t\in[T]$. The first inequality follows from the fact that $\mathcal{L}_t(\vx_t,\cdot,\lambda_t)$ is $(2\lambda_t L_{g,1})$-smooth with respect to $\vy$. Hence, 
\begin{align*}
    2\lambda_tL_{g,1}\left\|\vy_t^{K_t} - \vy_{\lambda_t,t}^*(\vx_t)\right\| \geq \left\|\nabla_\vy\mathcal{L}_t(\vx_t, \vy_t^{K_t}, \lambda_t)\right\| > \delta_\vy.
\end{align*}
After rearranging Eq.(\ref{B.2_1}), we have
\begin{align*}
    \rho_\vy^{K_t - 1} \geq& \frac{\delta_\vy^2}{4\lambda_t^2L_{g,1}^2\left\|\vy_t - \vy_{\lambda_t,t}^*(\vx_t)\right\|^2} \\
    K_t \leq& 1 + \frac{4\lambda_t^2L_{g,1}^2}{\ln\rho_\vy^{-1}}\left\|\vy_t - \vy_{\lambda_t,t}^*(\vx_t)\right\|^2 + \frac{2\ln\delta_\vy}{\ln\rho_\vy}.
\end{align*}
Let $\mathcal{I}_T^\vy$ denotes the total iteration number of $\vy$ and $\mathcal{I}_T^\vz$ denote the total iteration number of $\vz$, it immediately follows that
\begin{align}
    \mathcal{I}_T^\vy := \sum_{t=1}^TK_t \leq \left(1+\frac{2\ln\delta_\vy}{\ln\rho_\vy}\right)T + \frac{4L_{g,1}^2}{\ln\rho_\vy^{-1}}\sum_{t=1}^T\lambda_t^2\left\|\vy_t - \vy_{\lambda_t,t}^*(\vx_t)\right\|^2. \label{eq:I_T_y}
\end{align}
% In the same way, we obtain
% \begin{align}
%     \mathcal{I}_T^\vz := \sum_{t=1}^T M_t \leq \left(1+\frac{2\ln\delta_\vz}{\ln\rho_\vz}\right)T + \frac{L_{g,1}^2}{\ln\rho_\vz^{-1}}\sum_{t=1}^T \left\|\vz_t - \vy_{t}^*(\vx_t)\right\|^2. \label{eq:I_T_z}
% \end{align}

Then it begins with
\begin{align*}
    &\sum_{t=1}^T\lambda_t^2\left\|\vy_t - \vy_{\lambda_t,t}^*(\vx_t)\right\|^2 \\
    \leq& 2\sum_{t=1}^T\lambda_t^2\left\|\vy_t - \vy_{\lambda_{t-1},t-1}^*(\vx_{t-1})\right\|^2 + 2\sum_{t=1}^T\lambda_t^2\left\|\vy_{\lambda_{t-1},t-1}^*(\vx_{t-1}) - \vy_{\lambda_t,t}^*(\vx_t)\right\|^2 \\
    \leq& 2\sum_{t=1}^T\lambda_t^2 \frac{4\delta_\vy^2}{\lambda_t^2\mu_g^2} + 2\sum_{t=1}^T\lambda_t^2\left\|\vy_{\lambda_{t-1},t-1}^*(\vx_{t-1}) - \vy_{\lambda_t,t}^*(\vx_t)\right\|^2.
\end{align*}
Remind of~\eqref{eq:y-y}, we have
\begin{align*}
    \left\|\vy_{\lambda_{t-1},t-1}^*(\vx_{t-1}) - \vy_{\lambda_t,t}^*(\vx_t)\right\|^2 \leq& \frac{6L_{f,0}^2}{\lambda_t^2\mu_g^2} + 6\kappa_g^2\left\|\vx_{t-1} - \vx_t\right\|^2 + 6\left\|\vy_{t-1}^*(\vx_t) - \vy_t^*(\vx_t)\right\|^2 \\
    \leq& \frac{6L_{f,0}^2}{\lambda_t^2\mu_g^2} + 6\gamma^2\kappa_g^2L_{\mathcal{L},0}^2 + 6\left\|\vy_{t-1}^*(\vx_t) - \vy_t^*(\vx_t)\right\|^2,
\end{align*}
where the last inequality comes from Lemma~\ref{lem:L_lip}. Thus we obtain
\begin{align*}
    \sum_{t=1}^T\lambda_t^2\left\|\vy_t - \vy_{\lambda_t,t}^*(\vx_t)\right\|^2  \leq \frac{8\delta_\vy^2}{\mu_g^2}T + \frac{12L_{f,0}^2}{\mu_g^2}T + 12\gamma^2\lambda_1^2\kappa_g^2L_{\mathcal{L},0}^2T^{1+2\tau} + 6\lambda_1^2T^{2\tau}H_{2,T}.
\end{align*}
After substituting it into Eq.(\ref{eq:I_T_y}), we finally obtain
\begin{align*}
    \mathcal{I}_T^\vz \leq& \left(1+\frac{2\ln\delta_\vy}{\ln\rho_\vy}\right)T + \frac{4L_{g,1}^2}{\ln\rho_\vy^{-1}}\left(\frac{8\delta_\vy^2}{\mu_g^2}T + \frac{12L_{f,0}^2}{\mu_g^2}T + 12\gamma^2\lambda_1^2\kappa_g^2L_{\mathcal{L},0}^2T^{1+2\tau} + 6\lambda_1^2T^{2\tau}H_{2,T}\right) \\
    =& O\left(T^{1+2\tau} + T^{2\tau}H_{2,T}\right).
\end{align*}
Based on the similar analysis process for $\vy$, we have the following inequality
\begin{align*}
    \rho_\vz^{M_t-1} \geq& \frac{\delta_\vz^2}{L_{g,1}^2\left\|\vz_t - \vy_t^*(\vx_t)\right\|^2} \\
    M_t \leq& 1 + \frac{L_{g,1}^2}{\ln\rho_\vz^{-1}}\left\|\vz_t - \vy_t^*(\vx_t)\right\|^2 + \frac{2\ln\delta_\vz}{\ln\rho_\vz}.
\end{align*}
Thus we have
\begin{align*}
    \mathcal{I}_T^\vz \leq& \left(1 + \frac{2\ln\delta_\vz}{\ln\rho_\vz}\right)T + \frac{L_{g,1}^2}{\ln\rho_\vz^{-1}}\sum_{t=1}^T\left\|\vz_t - \vy_t^*(\vx_t)\right\|^2 \\
    \leq& \left(1 + \frac{2\ln\delta_\vz}{\ln\rho_\vz}\right)T + \frac{2L_{g,1}^2}{\ln\rho_\vz^{-1}}\left(\sum_{t=1}^T\left\|\vz_t - \vy_{t-1}^*(\vx_{t-1})\right\|^2 + \sum_{t=1}^T\left\|\vy_{t-1}^*(\vx_{t-1}) - \vy_t^*(\vx_t)\right\|^2\right) \\
    \leq& \left(1 + \frac{2\ln\delta_\vz}{\ln\rho_\vz}\right)T + \frac{2\kappa_g^2\delta_\vz^2}{\ln\rho_\vz^{-1}}T + \frac{4\gamma^2\kappa_g^2L_{\mathcal{L},0}^2L_{g,1}^2}{\ln\rho_\vz^{-1}}T + \frac{4L_{g,1}^2}{\ln\rho_\vz^{-1}}H_{2,T} = O\left(T\ln T + H_{2,T}\right).
\end{align*}
It is true that
\begin{align*}
    \mathcal{I}_T = \mathcal{I}_T^\vy + \mathcal{I}_T^\vz \leq O\left(T^{1+2\tau} + T^{2\tau}H_{2,T}\right).
\end{align*}
\end{proof}

\section{Proof of Section~\ref{sec:6}}

We first present Lemma~\ref{lem:stochas_begin}, which establishes a preliminary regret upper bound for the fully first-order algorithm applied to the approximation problem~\eqref{P3} in the stochastic OBO setting. Lemmas~\ref{lem:stochas_y-y} and~\ref{lem:stochas_z-z} provide the inner-loop iteration error bounds under stochastic noise, respectively. These bounds further lead to the approximate stochastic hypergradient error bound in Lemma~\ref{lem:stochas_nabla_L}, which ultimately yields the proof of Theorem~\ref{prf:thm:SF$^2$OBO}.

\begin{lemma}\label{lem:stochas_begin}
    Under Assumptions~\ref{asm1},~\ref{asm2}, suppose that 
    $\gamma_t \leq \frac{1}{2L_F}$ for all $t\in[T]$. 
    Consider the update
    \begin{align*}
        \vx_{t+1}
        \leftarrow
        \vx_t - \gamma_t
        \mathcal{G}_\mathcal{X}
        \left(\vx_t, \widetilde\nabla \mathcal{L}_t(\vx_t;\mathcal{S}_t), \gamma_t\right).
    \end{align*}
    Then, we obtain
    \begin{align*}
        \mathbb{E}\left[\mathrm{BLReg}(T)\right] \leq& \sum_{t=1}^T\frac{8}{\gamma_t}\mathbb{E}\left[ F_t(\vx_t) - F_t(\vx_{t+1})\right] + 24D_1^2\sum_{t=1}^T\frac{1}{\lambda_t^2} + 72\sigma^2\sum_{t=1}^T\frac{\lambda_t^2}{\widehat{B}_t} \\
        & + 24\sum_{t=1}^T \mathbb{E}\left[ \left\|\nabla \mathcal{L}_t^*(\vx_t) - \widetilde\nabla \mathcal{L}_t^*(\vx_t)\right\|^2\right].
\end{align*}
\end{lemma}
\begin{proof}
We begin with the $L_F$-smoothness of $F_t$. By Lemma~\ref{lem:L_F}, we have
\begin{align*}
    &F_t(\vx_{t+1}) - F_t(\vx_t) \\
    \leq& \left\langle \nabla F_t(\vx_t), \vx_{t+1} - \vx_t \right\rangle + \frac{L_F}{2} \left\|\vx_{t+1} - \vx_t\right\|^2 \\
    =& -\gamma_t \left\langle \nabla F_t(\vx_t), \mathcal{G}_\mathcal{X} \left(\vx_t, \widetilde\nabla \mathcal{L}_t^*(\vx_t;\mathcal{S}_t), \gamma_t\right) \right\rangle + \frac{\gamma_t^2L_F}{2} \left\|\mathcal{G}_\mathcal{X} \left(\vx_t, \widetilde\nabla \mathcal{L}_t^*(\vx_t;\mathcal{S}_t), \gamma_t\right)\right\|^2 \\
    =& -\gamma_t \left\langle \widetilde\nabla \mathcal{L}_t^*(\vx_t;\mathcal{S}_t), \mathcal{G}_\mathcal{X} \left(\vx_t, \widetilde\nabla \mathcal{L}_t^*(\vx_t;\mathcal{S}_t), \gamma_t\right)\right\rangle + \frac{\gamma_t^2L_F}{2} \left\|\mathcal{G}_\mathcal{X} \left(\vx_t, \widetilde\nabla \mathcal{L}_t^*(\vx_t;\mathcal{S}_t), \gamma_t\right)\right\|^2 \\
    & - \gamma_t\left\langle \nabla F_t(\vx_t) - \widetilde\nabla \mathcal{L}_t^*(\vx_t;\mathcal{S}_t),\mathcal{G}_\mathcal{X} \left(\vx_t, \widetilde\nabla \mathcal{L}_t^*(\vx_t;\mathcal{S}_t), \gamma_t\right) \right\rangle \\
    \leq& -\gamma_t\left(1 - \frac{\gamma_t L_F}{2}\right) \left\|\mathcal{G}_\mathcal{X} \left(\vx_t, \widetilde\nabla \mathcal{L}_t^*(\vx_t;\mathcal{S}_t), \gamma_t\right)\right\|^2 + \frac{\gamma_t}{2}\left\|\nabla F_t(\vx_t) - \widetilde\nabla \mathcal{L}_t^*(\vx_t;\mathcal{S}_t)\right\|^2 \\
    & + \frac{\gamma_t}{2}\left\|\mathcal{G}_\mathcal{X} \left(\vx_t, \widetilde\nabla \mathcal{L}_t^*(\vx_t;\mathcal{S}_t), \gamma_t\right)\right\|^2 \\
    =& -\frac{\gamma_t}{2}\left(1 - \gamma_t L_F\right) \left\|\mathcal{G}_\mathcal{X} \left(\vx_t, \widetilde\nabla \mathcal{L}_t^*(\vx_t;\mathcal{S}_t), \gamma_t\right)\right\|^2 + \frac{\gamma_t}{2} \left\|\nabla F_t(\vx_t) - \widetilde\nabla \mathcal{L}_t^*(\vx_t;\mathcal{S}_t)\right\|^2.
\end{align*}
Let $\gamma_t \leq \frac{1}{2L_F}$, rearranging the above inequality gives
\begin{align*}
    \left\|\mathcal{G}_\mathcal{X} \left(\vx_t, \widetilde\nabla \mathcal{L}_t^*(\vx_t;\mathcal{S}_t), \gamma_t\right)\right\|^2 \leq& \frac{4}{\gamma_t} \left(F_t(\vx_t) - F_t(\vx_{t+1})\right) + 2\left\|\nabla F_t(\vx_t) - \widetilde\nabla \mathcal{L}_t^*(\vx_t;\mathcal{S}_t)\right\|^2. 
\end{align*}
Moreover, by
\begin{align*}
    \frac{1}{2}\left\|\mathcal{G}_\mathcal{X} \left(\vx_t, \nabla F_t(\vx_t), \gamma_t\right)\right\|^2 \leq \left\|\mathcal{G}_\mathcal{X} \left(\vx_t, \widetilde\nabla \mathcal{L}_t^*(\vx_t;\mathcal{S}_t), \gamma_t\right)\right\|^2 + \left\|\nabla F_t(\vx_t) - \widetilde\nabla \mathcal{L}_t^*(\vx_t;\mathcal{S}_t)\right\|^2,
\end{align*}
it follows that
\begin{align}
    \left\|\mathcal{G}_\mathcal{X} \left(\vx_t, \nabla F_t(\vx_t), \gamma_t\right)\right\|^2 \leq& \frac{8}{\gamma_t} \left(F_t(\vx_t) - F_t(\vx_{t+1})\right) + 8\left\|\nabla F_t(\vx_t) - \widetilde\nabla \mathcal{L}_t^*(\vx_t;\mathcal{S}_t)\right\|^2 \nonumber\\
    \leq& \frac{8}{\gamma_t} \left(F_t(\vx_t) - F_t(\vx_{t+1})\right) + 24\left\|\nabla F_t(\vx_t) - \nabla \mathcal{L}_t^*(\vx_t)\right\|^2 \nonumber\\
    & + 24\left\|\nabla \mathcal{L}_t^*(\vx_t) - \widetilde\nabla \mathcal{L}_t^*(\vx_t)\right\|^2 + 24\left\|\widetilde\nabla \mathcal{L}_t^*(\vx_t) - \widetilde\nabla \mathcal{L}_t^*(\vx_t;\mathcal{S}_t)\right\|^2. \label{E.1_1}
\end{align}
Note that, under Assumption~\ref{asm:stochas}, we naturally have
\begin{align*}
    \mathbb{E}_{\mathcal{S}_t}\left[\left\|\widetilde\nabla \mathcal{L}_t^*(\vx_t) - \widetilde\nabla \mathcal{L}_t^*(\vx_t;\mathcal{S}_t)\right\|^2\right] =& \mathbb{E}_{\mathcal{S}_t}\left[\left\|\nabla_\vx f_t(\vx_t, \vy_{t+1}) - \nabla_\vx f_t(\vx_t, \vy_{t+1};\xi_t) \right.\right. \\
    & \left.\left. + \lambda_t\left(\nabla_\vx g_t(\vx_t, \vy_{t+1}) - \nabla_\vx g_t(\vx_t, \vy_{t+1};\zeta_t)\right) \right.\right. \\
    & \left.\left. + \lambda_t\left(\nabla_\vx g_t(\vx_t, \vz_{t+1}) - \nabla_\vx g_t(\vx_t, \vz_{t+1})\right)\right\|^2\right] \\
    =& \mathbb{E}_{\xi_t}\left[\left\|\nabla_\vx f_t(\vx_t, \vy_{t+1}) - \nabla_\vx f_t(\vx_t, \vy_{t+1};\xi_t)\right\|^2\right] \\
    & + \lambda_t^2\mathbb{E}_{\zeta_t}\left[\left\|\nabla_\vx g_t(\vx_t, \vy_{t+1}) - \nabla_\vx g_t(\vx_t, \vy_{t+1};\zeta_t)\right\|^2\right] \\
    & + \lambda_t^2\mathbb{E}_{\zeta_t}\left[\left\|\nabla_\vx g_t(\vx_t, \vz_{t+1}) - \nabla_\vx g_t(\vx_t, \vz_{t+1};\zeta_t)\right\|^2\right] \\
    \leq& (1+2\lambda_t^2)\frac{\sigma^2}{\widehat{B}_t} \leq \frac{3\lambda_t^2\sigma^2}{\widehat{B}_t},
\end{align*}
and substituting the above inequality into Eq.(\ref{E.1_1}), we complete the proof by
\begin{align*}
    \mathbb{E}\left[\mathrm{BLReg}(T)\right] \leq& \sum_{t=1}^T\frac{8}{\gamma_t}\mathbb{E}\left[ F_t(\vx_t) - F_t(\vx_{t+1})\right] + 24D_1^2\sum_{t=1}^T\frac{1}{\lambda_t^2} + 72\sigma^2\sum_{t=1}^T\frac{\lambda_t^2}{\widehat{B}_t} \\
        & + 24\sum_{t=1}^T \mathbb{E}\left[ \left\|\nabla \mathcal{L}_t^*(\vx_t) - \widetilde\nabla \mathcal{L}_t^*(\vx_t)\right\|^2\right].
\end{align*}
\end{proof}

\begin{lemma}\label{lem:stochas_y-y}
Under Assumptions~\ref{asm1}, \ref{asm2} and \ref{asm:stochas}, let $\beta_t \leq \frac{1}{2\lambda_tL_{g,1}}$, Algorithm~\ref{alg:StochasF$^2$OBO} can obtain
\begin{align*}
    &\mathbb{E}\left[\left\|\vy_{t+1} - \vy_{\lambda_t,t}^*(\vx_t)\right\|^2\right] \\
    \leq& \rho_\vy^{t-1} \mathbb{E}\left[ \left\|\mathbf{y}_2 - \mathbf{y}_{\lambda_1,1}^*(\mathbf{x}_1)\right\|^2\right] + \frac{4\sigma^2}{\mu_g} \sum_{j=0}^{t-2} \frac{\rho_\vy^j\beta_{t-j}\lambda_{t-j}}{B_{t-j}} + \frac{12L_{f,0}^2q_t^{K_t}}{\mu_g^2} \sum_{j=0}^{t-2} \frac{\rho_\mathbf{y}^j}{\lambda_{t-j}} \\
    & + 12\kappa_g^2q_t^{K_t} \sum_{j=0}^{t-2} \rho_\mathbf{y}^j \mathbb{E}\left[ \left\|\mathbf{x}_{t-j} - \mathbf{x}_{t-1-j}\right\|^2\right] + 12q_t^{K_t} \sum_{j=0}^{t-2} \rho_\mathbf{y}^j \mathbb{E}\left[ \left\|\mathbf{y}_{t-1-j}^*(\mathbf{x}_{t-j}) - \mathbf{y}_{t-j}^*(\mathbf{x}_{t-j})\right\|^2\right],
\end{align*}
where $q_t = 1 - \frac{\beta_t\lambda_t\mu_g}{2}$. For each $t\in[T]$, we choose $K_t>0$ such that $2q_t^{K_t} \leq \rho_\vy$ for some constant $\rho_\vy>0$.
\end{lemma}
\begin{proof}
Let $\widetilde{\vg}_\vy :=  \nabla_\vy\mathcal{L}_t(\vx_t, \vy_t^k, \lambda_t;\mathcal{S}_t)$, $\vg_\vy := \nabla_\vy\mathcal{L}_t(\vx_t, \vy_t^k, \lambda_t)$. With the update rule
\begin{align*}
    \vy_t^{k+1} \leftarrow \vy_t^k - \beta_t\widetilde{\vg}_\vy,
\end{align*}
we have
\begin{align*}
    &\mathbb{E}_{\mathcal{S}_t^k}\left[ \left\|\vy_t^{k+1} - \vy_{\lambda_t,t}^*(\vx_t)\right\|^2 \right] \\
    =& \mathbb{E}_{\mathcal{S}_t^k}\left[ \left\|\vy_t^k - \beta_t \widetilde{\vg}_\vy - \vy_{\lambda_t,t}^*(\vx_t) \right\|^2 \right] \\
    \leq& \left\|\vy_t^k - \vy_{\lambda_t,t}^*(\vx_t)\right\|^2 - 2\beta_t \mathbb{E}_{\mathcal{S}_t^k}\left[ \left\langle \widetilde{\vg}_\vy, \vy_t^k - \vy_{\lambda_t,t}^*((\vx_t) \right\rangle \right] + \beta_t^2\mathbb{E}_{\mathcal{S}_t^k}\left[ \left\|\widetilde{\vg}_\vy\right\|^2 \right] \\
    \overset{(i)}{=}& \left\|\vy_t^k - \vy_{\lambda_t,t}^*(\vx_t)\right\|^2 - 2\beta_t \left\langle \vg_\vy, \vy_t^k - \vy_{\lambda_t,t}^*(\vx_t) \right\rangle + \beta_t^2\left\|\vg_\vy\right\|^2 + \beta_t^2(1+\lambda_t^2)\frac{\sigma^2}{B_t} \\
    \overset{(ii)}{\leq}& \left\|\vy_t^k - \vy_{\lambda_t,t}^*(\vx_t)\right\|^2 - \beta_t\frac{\lambda_t\mu_g}{2}\left\|\vy_t^k - \vy_{\lambda_t,t}^*(\vx_t)\right\|^2 - \left( \frac{\beta_t}{2\lambda_tL_{g,1}} - \beta_t^2 \right)\left\|\vg_\vy\right\|^2 + \beta_t^2(1+\lambda_t^2)\frac{\sigma^2}{B_t} \\
    \leq& q_t\left\|\vy_t^k - \vy_{\lambda_t,t}^*(\vx_t)\right\|^2 + 2\beta_t^2\lambda_t^2\frac{\sigma^2}{B_t}
\end{align*}
where $(i)$ comes from the fact that
\begin{align*}
    \mathbb{E}_{\mathcal{S}_t^k}\left[ \left\|\vg_\vy - \widetilde{\vg}_\vy\right\|^2 \right] =& \mathbb{E}_{\mathcal{S}_t^k} \left[\left\|\nabla_\vy f_t(\vx_t, \vy_t^k) - \nabla_\vy f_t(\vx_t, \vy_t^k;\xi_t^k) \right.\right. \\
    & \left.\left. + \lambda_t\left(\nabla_\vy g_t(\vx_t, \vy_t^k) - \nabla_\vy g_t(\vx_t, \vy_t^k;\zeta_t^k)\right)\right\|^2\right] \leq (1+\lambda_t^2)\frac{\sigma^2}{B_t},
\end{align*}
and $(ii)$ comes from the $(\lambda_t\mu_g/2)$-strongly convexity and $(2\lambda_tL_{g,1})$-smoothness of $\mathcal{L}_t$ w.r.t. $\vy$, and the last inequality comes from $\beta_t\leq \frac{1}{2\lambda_tL_{g,1}}$, let $q_t := 1 - \frac{\beta_t\lambda_t\mu_g}{2}$, we then obtain
\begin{align}
    \mathbb{E}\left[ \left\|\vy_{t+1} - \vy_{\lambda_t,t}^*(\vx_t)\right\|^2\right] \leq& q_t^{K_t} \mathbb{E}\left[ \left\|\vy_t - \vy_{\lambda_t,t}^*(\vx_t)\right\|^2\right] + 2\beta_t^2\lambda_t^2\frac{\sigma^2}{B_t} \sum_{j=0}^{K_t-1}q_t^j \nonumber\\
    \leq& q_t^{K_t} \mathbb{E}\left[ \left\|\vy_t - \vy_{\lambda_t,t}^*(\vx_t)\right\|^2\right] + \frac{4\beta_t\lambda_t\sigma^2}{\mu_gB_t}. \label{E.2_1}
\end{align}
Recall from~\eqref{eq:y-y} that we also have
\begin{align*}
    &\left\|\vy_t - \vy_{\lambda_t,t}^*(\vx_t)\right\|^2 \\
    \leq& 2 \left\|\vy_t - \vy_{\lambda_{t-1},t-1}^*(\vx_{t-1})\right\|^2 + 2 \left\|\vy_{\lambda_{t-1},t-1}^*(\vx_{t-1}) - \vy_{\lambda_t,t}^*(\vx_t)\right\|^2 \\
    \leq& 2 \left\|\vy_t - \vy_{\lambda_{t-1},t-1}^*(\vx_{t-1})\right\|^2 + \frac{12L_{f,0}^2}{\lambda_t^2\mu_g^2} + 12\kappa_g^2\left\|\vx_{t-1} - \vx_t\right\|^2 + 12\left\|\vy_{t-1}^*(\vx_t) - \vy_t^*(\vx_t)\right\|^2,
\end{align*}
and substituting the above inequality into~\eqref{E.2_1}, we obtain
\begin{align*}
    \mathbb{E}\left[\left\|\vy_{t+1} - \vy_{\lambda_t,t}^*(\vx_t)\right\|^2\right] \leq& 2q_t^{K_t} \mathbb{E}\left[ \left\|\vy_t - \vy_{\lambda_{t-1},t-1}^*(\vx_{t-1})\right\|^2\right] + \frac{4\beta_t\lambda_t\sigma^2}{\mu_gB_t} + \frac{12q_t^{K_t}L_{f,0}^2}{\lambda_t^2\mu_g^2} \\
    & + 12q_t^{K_t}\kappa_g^2\mathbb{E}\left[ \left\|\vx_{t-1} - \vx_t\right\|^2\right] + 12q_t^{K_t}\mathbb{E}\left[ \left\|\vy_{t-1}^*(\vx_t) - \vy_t^*(\vx_t)\right\|^2\right].
\end{align*}
Telescoping the above inequality, we complete the proof by
\begin{align*}
    &\mathbb{E}\left[\left\|\vy_{t+1} - \vy_{\lambda_t,t}^*(\vx_t)\right\|^2\right] \\
    \leq& \mathbb{E}\left[\left\|\mathbf{y}_2 - \mathbf{y}_{\lambda_1,1}^*(\mathbf{x}_1)\right\|^2\right] \prod_{j=0}^{t-2} 2q_j^{K_j} + \frac{4\sigma^2}{\mu_g} \sum_{j=0}^{t-2} \frac{\beta_{t-j}\lambda_{t-j}}{B_{t-j}} \prod_{m=1}^{j} 2q_{t+1-m}^{K_{t+1-m}} \\
    & + \frac{6L_{f,0}^2}{\mu_g^2} \sum_{j=0}^{t-2} \frac{1}{\lambda_{t-j}^2} \prod_{m=0}^{j}2q_{t-m}^{K_{t-m}} + 6\kappa_g^2 \sum_{j=0}^{t-2} \mathbb{E}\left[ \left\|\mathbf{x}_{t-j} - \mathbf{x}_{t-1-j}\right\|^2\right] \prod_{m=0}^{j}2q_{t-m}^{K_{t-m}} \\
    & + 6 \sum_{j=0}^{t-2} \mathbb{E}\left[ \left\|\mathbf{y}_{t-1-j}^*(\mathbf{x}_{t-j})-\mathbf{y}_{t-j}^*(\mathbf{x}_{t-j})\right\|^2\right] \prod_{m=0}^{j} 2q_{t-m}^{K_{t-m}} \\
    \leq& \rho_\vy^{t-1} \mathbb{E}\left[\left\|\mathbf{y}_2 - \mathbf{y}_{\lambda_1,1}^*(\mathbf{x}_1)\right\|^2\right] + \frac{4\sigma^2}{\mu_g} \sum_{j=0}^{t-2} \frac{\rho_\vy^j\beta_{t-j}\lambda_{t-j}}{B_{t-j}} + \frac{12L_{f,0}^2q_t^{K_t}}{\mu_g^2} \sum_{j=0}^{t-2} \frac{\rho_\mathbf{y}^j}{\lambda_{t-j}^2} \\
    & + 12\kappa_g^2q_t^{K_t} \sum_{j=0}^{t-2} \rho_\mathbf{y}^j \mathbb{E}\left[ \left\|\mathbf{x}_{t-j} - \mathbf{x}_{t-1-j}\right\|^2\right] + 12q_t^{K_t} \sum_{j=0}^{t-2} \rho_\mathbf{y}^j \mathbb{E}\left[ \left\|\mathbf{y}_{t-1-j}^*(\mathbf{x}_{t-j}) - \mathbf{y}_{t-j}^*(\mathbf{x}_{t-j})\right\|^2\right],
\end{align*}
where the last inequality follows by assuming that, for all $t\in[T]$, $K_t$ is sufficiently large such that $2q_t^{K_t} \leq \rho_\vy$ for some constant $\rho_\vy$.
\end{proof}

\begin{lemma}\label{lem:stochas_z-z}
    Under Assumptions~\ref{asm1}, \ref{asm2} and \ref{asm:stochas}, let $\alpha_t \leq \frac{1}{L_{g,1}}$, Algorithm~\ref{alg:StochasF$^2$OBO} can obtain
    \begin{align*}
        &\mathbb{E}\left[ \left\|\mathbf{z}_{t+1} - \vy_t^*(\vx_t)\right\|^2 \right] \\
        \leq& \rho_\mathbf{z}^{t-1} \mathbb{E}\left[ \left\|\mathbf{z}_2 - \mathbf{y}_1^*(\mathbf{x}_1)\right\|^2\right] + \frac{\sigma^2}{\mu_g} \sum_{j=0}^{t-2}\frac{\rho_\mathbf{z}^j\alpha_{t-j}}{B_{t-j}} + 4p_t^{K_t}\kappa_g^2 \sum_{j=0}^{t-2}\rho_\mathbf{z}^j \mathbb{E}\left[ \left\|\mathbf{x}_{t-1-j} - \mathbf{x}_{t-j}\right\|^2\right] \\
        & + 4p_t^{K_t} \sum_{j=0}^{t-2}\rho_\mathbf{z}^j \mathbb{E}\left[ \left\|\mathbf{y}_{t-1-j}^*(\mathbf{x}_{t-j}) - \mathbf{y}_{t-j}^*(\mathbf{x}_{t-j})\right\|^2\right],
    \end{align*}
    where $p_t = 1 - \alpha_t\mu_g$. For each $t\in[T]$, we choose $K_t>0$ such that $2p_t^{K_t} \leq \rho_\vz$ for some constant $\rho_\vz>0$.
\end{lemma}
\begin{proof}
By applying a derivation similar to that in Lemma~\ref{lem:stochas_y-y}, we have
\begin{align*}
    &\mathbb{E}_{\zeta_t^k}\left[\left\|\vz_t^{k+1} - \vy_t^*(\vx_t)\right\|^2\right] \\
    =& \mathbb{E}_{\zeta_t^k}\left[\left\|\vz_t^k - \alpha_t\nabla_\vy g_t(\vx_t, \vz_t^k;\zeta_t^k) - \vy_t^*(\vx_t)\right\|^2\right] \\
    \leq& \left\|\vz_t^k - \vy_t^*(\vx_t)\right\|^2 - 2\alpha_t\left\langle \nabla_\vy g_t(\vx_t, \vz_t^k), \vz_t^k - \vy_t^*(\vx_t) \right\rangle + \alpha_t^2\mathbb{E}_{\zeta_t^k}\left[\left\|\nabla_\vy g_t(\vx_t, \vz_t^k;\zeta_t^k)\right\|^2\right] \\
    =& \left\|\vz_t^k - \vy_t^*(\vx_t)\right\|^2 - 2\alpha_t\left\langle \nabla_\vy g_t(\vx_t, \vz_t^k), \vz_t^k - \vy_t^*(\vx_t) \right\rangle + \alpha_t^2\left\|\nabla_\vy g_t(\vx_t, \vz_t^k)\right\|^2 + \frac{\alpha_t^2\sigma^2}{B_t} \\
    \leq& \left\|\vz_t^k - \vy_t^*(\vx_t)\right\|^2 - \alpha_t\mu_g\left\|\vz_t^k - \vy_t^*(\vx_t)\right\|^2 - \left(\frac{\alpha_t}{L_{g,1}} - \alpha_t^2\right)\left\|\nabla_\vy g_t(\vx_t, \vz_t^k)\right\|^2 + \frac{\alpha_t^2\sigma^2}{B_t} \\
    \leq& p_t\left\|\vz_t^k - \vy_t^*(\vx_t)\right\|^2 + \frac{\alpha_t^2\sigma^2}{B_t},
\end{align*}
where $p_t := 1-\alpha_t\mu_g$. Then, after telescoping over $K_t$ iterations, we have
\begin{align*}
    &\mathbb{E}\left[\left\|\vz_{t+1} - \vy_t^*(\vx_t)\right\|^2\right] \\
    \leq& p_t^{K_t} \mathbb{E}\left[ \left\|\vz_t - \vy_t^*(\vx_t)\right\|^2\right] + \frac{\alpha_t^2\sigma^2}{B_t}\sum_{j=0}^{K_t-1}p_t^j \\
    \leq& 2p_t^{K_t} \mathbb{E}\left[ \left\|\vz_t - \vy_{t-1}^*(\vx_{t-1})\right\|^2\right] + 2p_t^{K_t} \mathbb{E}\left[ \left\|\vy_{t-1}^*(\vx_{t-1}) - \vy_t^*(\vx_t)\right\|^2\right] + \frac{\alpha_t\sigma^2}{\mu_gB_t} \\
    \leq& 2p_t^{K_t} \mathbb{E}\left[ \left\|\vz_t - \vy_{t-1}^*(\vx_{t-1})\right\|^2\right] + 4p_t^{K_t}\kappa_g^2 \mathbb{E}\left[ \left\|\vx_{t-1} - \vx_t\right\|^2\right] \\
    & + 4p_t^{K_t} \mathbb{E}\left[ \left\|\vy_{t-1}^*(\vx_t) - \vy_t^*(\vx_t)\right\|^2\right] + \frac{\alpha_t\sigma^2}{\mu_gB_t}.
\end{align*}
Finally, telescoping the above inequality yields
\begin{align*}
    &\mathbb{E}\left[ \left\|\mathbf{z}_{t+1} - \vy_t^*(\vx_t)\right\|^2 \right] \\
    \leq& \rho_\mathbf{z}^{t-1} \mathbb{E}\left[ \left\|\mathbf{z}_2 - \mathbf{y}_1^*(\mathbf{x}_1)\right\|^2\right] + \frac{\sigma^2}{\mu_g} \sum_{j=0}^{t-2} \frac{\rho_\mathbf{z}^j\alpha_{t-j}}{B_{t-j}} + 4p_t^{K_t}\kappa_g^2 \sum_{j=0}^{t-2}\rho_\mathbf{z}^j \mathbb{E}\left[ \left\|\mathbf{x}_{t-1-j} - \mathbf{x}_{t-j}\right\|^2\right] \\
    & + 4p_t^{K_t} \sum_{j=0}^{t-2}\rho_\mathbf{z}^j \mathbb{E}\left[ \left\|\mathbf{y}_{t-1-j}^*(\mathbf{x}_{t-j}) - \mathbf{y}_{t-j}^*(\mathbf{x}_{t-j})\right\|^2\right],
\end{align*}
where the last inequality follows by assuming that, for all $t\in[T]$, $K_t$ is sufficiently large such that $2p_t^{K_t} \leq \rho_\vz$ for some constant $\rho_\vz$.
\end{proof}

\begin{lemma}\label{lem:stochas_nabla_L}
    Under Assumptions~\ref{asm1}, \ref{asm2} and \ref{asm:stochas}, let $\alpha_t \leq \frac{1}{L_{g,1}}$, $\beta_t \leq \frac{1}{2\lambda_tL_{g,1}}$ and $\gamma_t \equiv \gamma \leq \frac{1}{8\kappa_gL_{g,1}}\sqrt{\frac{1 - \rho}{21c_\lambda}}$, Algorithm~\ref{alg:StochasF$^2$OBO} can obtain
    \begin{align*}
        &\sum_{t=1}^T \mathbb{E}\left[\left\|\widetilde\nabla \mathcal{L}_t^*(\vx_t) - \nabla \mathcal{L}_t^*(\vx_t)\right\|^2\right] \\
        \leq& 2c_\tau\Delta_1 + \frac{54\lambda_T^2\kappa_gL_{g,1}\sigma^2}{1-\rho} \sum_{t=1}^T\frac{\beta_t\lambda_t + \alpha_t}{B_t} + \frac{144\kappa_g^2L_{f,0}^2c_\lambda}{1-\rho} \sum_{t=1}^T\frac{1}{\lambda_t^2} + \frac{168L_{g,1}^2c_\lambda}{1-\rho} H_{2,T} \\
        & + \frac{672\gamma^2\kappa_g^2L_{g,1}^2c_\lambda}{1-\rho} \mathbb{E}\left[ \mathrm{BLReg}(T) \right] + \frac{672\gamma^2\kappa_g^2L_{g,1}^2c_\lambda D_1^2}{1-\rho} \sum_{t=1}^T \frac{1}{\lambda_t^2} + \frac{2016\gamma^2\kappa_g^2L_{g,1}^2c_\lambda\sigma^2}{1-\rho} \sum_{t=1}^T\frac{\lambda_t^2}{\widehat{B}_t},
    \end{align*}
    where $c_\tau$, $\Delta_1$ are some constants. For each $t\in[T]$, we choose $K_t>0$ such that $\lambda_t^2q_t^{K_t} \leq c_\lambda$, $\lambda_t^2p_t^{K_t} \leq c_\lambda$ for some constant $c_\lambda$.
\end{lemma}
\begin{proof}
Recall from~\eqref{a+b} in Lemma~\ref{prf:lem:nabla_L} that
\begin{align*}
    \left\|\widetilde\nabla \mathcal{L}_t^*(\vx_t) - \nabla \mathcal{L}_t^*(\vx_t)\right\|^2 
    \leq& 6\lambda_t^2L_{g,1}^2\left\|\vy_{t+1} - \vy_{\lambda_t,t}^*(\vx_t)\right\|^2 + 3\lambda_t^2L_{g,1}^2\left\|\vz_{t+1} - \vy_t^*(\vx_t)\right\|^2.
\end{align*}
Then, substituting the results of Lemma~\ref{lem:stochas_y-y} and Lemma~\ref{lem:stochas_z-z} into the above inequality, we obtain
\begin{align*}
    &\sum_{t=2}^T \mathbb{E}\left[ \left\|\widetilde\nabla \mathcal{L}_t^*(\vx_t) - \nabla \mathcal{L}_t^*(\vx_t)\right\|^2\right] \\
    \leq& 6L_{g,1}^2\sum_{t=2}^T \lambda_t^2\mathbb{E}\left[ \left\|\vy_{t+1} - \vy_{\lambda_t,t}^*(\vx_t)\right\|^2 \right] + 3L_{g,1}^2 \sum_{t=2}^T \lambda_t^2\mathbb{E}\left[ \left\|\vz_{t+1} - \vy_t^*(\vx_t)\right\|^2\right] \\
    \leq& 6L_{g,1}^2 \mathbb{E}\left[ \left\|\vy_2 - \vy_{\lambda_1,1}^*(\vx_1)\right\|^2\right] \sum_{t=2}^T \lambda_t^2\rho_\vy^{t-1} + 24\kappa_gL_{g,1}\sigma^2\sum_{t=2}^T\lambda_t^2\sum_{j=0}^{t-2}\frac{\rho_\vy^j\beta_{t-j}\lambda_{t-j}}{B_{t-j}} \\
    & + 72\kappa_g^2L_{f,0}^2 \sum_{t=2}^T \lambda_t^2q_t^{K_t} \sum_{j=0}^{t-2}\frac{\rho_\vy^j}{\lambda_{t-j}^2} + 72\kappa_g^2L_{g,1}^2 \sum_{t=2}^T \lambda_t^2q_t^{K_t} \sum_{j=0}^{t-2} \rho_\vy^j \mathbb{E}\left[ \left\|\vx_{t-j} - \vx_{t-1-j}\right\|^2\right] \\
    & + 72L_{g,1}^2 \sum_{t=2}^T \lambda_t^2q_t^{K_t} \sum_{j=0}^{t-2}\rho_\vy^j \mathbb{E}\left[ \left\|\vy_{t-1-j}^*(\vx_{t-j}) - \vy_{t-j}^*(\vx_{t-j})\right\|^2\right] \\
    & + 3L_{g,1}^2 \mathbb{E}\left[ \left\|\vz_2 - \vy_1^*(\vx_1)\right\|^2\right] \sum_{t=2}^T \lambda_t^2\rho_\vz^{t-1} + 3\kappa_gL_{g,1}\sigma^2\sum_{t=2}^T\lambda_t^2\sum_{j=0}^{t-2}\frac{\rho_\vz^j\alpha_{t-j}}{B_{t-j}} \\
    & + 12\kappa_g^2L_{g,1}^2 \sum_{t=2}^T \lambda_t^2p_t^{K_t} \sum_{j=0}^{t-2} \rho_\vz^j \mathbb{E}\left[ \left\|\vx_{t-j} - \vx_{t-1-j}\right\|^2\right] \\
    & + 12L_{g,1}^2 \sum_{t=2}^T \lambda_t^2p_t^{K_t} \sum_{j=0}^{t-2} \rho_\vz^j \mathbb{E}\left[ \left\|\vy_{t-1-j}^*(\vx_{t-j}) - \vy_{t-j}^*(\vx_{t-j})\right\|^2\right].
\end{align*}
Now, assume that, for a sufficiently large $K_t$, $\lambda_t^2 q_t^{K_t} \leq c_\lambda$ and $\lambda_t^2 p_t^{K_t} \leq c_\lambda$. For simplicity, define $\rho := \max\{\rho_\vy,\rho_\vz\}$. Then, remind of~\eqref{c_tau}, it holds that
\begin{align*}
    &\sum_{t=1}^T \mathbb{E}\left[ \left\|\widetilde\nabla \mathcal{L}_t^*(\vx_t) - \nabla \mathcal{L}_t^*(\vx_t)\right\|^2\right] \\
    \leq& c_\tau\Delta_1 + 27\kappa_gL_{g,1}\sigma^2\sum_{t=2}^T\lambda_t^2 \sum_{j=0}^{t-2}\frac{\rho^j}{B_{t-j}}(\beta_{t-j}\lambda_{t-j} + \alpha_{t-j}) + 72\kappa_g^2L_{f,0}^2c_\lambda \sum_{t=2}^T\sum_{j=0}^{t-2}\frac{\rho_\vy^j}{\lambda_{t-j}^2} \\
    & + 84L_{g,1}^2c_\lambda \sum_{t=2}^T\sum_{j=0}^{t-2}\rho^j \mathbb{E}\left[ \left\|\vy_{t-1-j}^*(\vx_{t-j}) - \vy_{t-j}^*(\vx_{t-j})\right\|^2\right] \\
    & + 84\kappa_g^2L_{g,1}^2c_\lambda \sum_{t=2}^T\sum_{j=0}^{t-2}\rho^j \mathbb{E}\left[ \left\|\vx_{t-j} - \vx_{t-1-j}\right\|^2\right] \\
    \leq& c_\tau\Delta_1 + 27\kappa_gL_{g,1}\sigma^2\sum_{t=2}^T\lambda_t^2 \sum_{j=0}^{t-2}\frac{\rho^j}{B_{t-j}}(\beta_{t-j}\lambda_{t-j} + \alpha_{t-j}) + 72\kappa_g^2L_{f,0}^2c_\lambda \sum_{t=2}^T\sum_{j=0}^{t-2}\frac{\rho_\vy^j}{\lambda_{t-j}^2} \\
    & + \frac{84\kappa_g^2L_{g,1}^2c_\lambda}{1-\rho} \sum_{t=2}^T \mathbb{E}\left[ \left\|\vx_t - \vx_{t-1}\right\|^2\right] + \frac{84L_{g,1}^2c_\lambda}{1-\rho} H_{2,T},
\end{align*}
where $\Delta_1 := 6L_{g,1}^2 \mathbb{E}[ \|\vy_2 - \vy_{\lambda_1,1}^*(\vx_1)\|^2] + 3L_{g,1}^2 \mathbb{E}[ \|\vz_2 - \vy_1^*(\vx_1)\|^2]$. Then we have
\begin{align*}
    &\sum_{t=1}^T \mathbb{E}\left[ \left\|\widetilde\nabla \mathcal{L}_t^*(\vx_t) - \nabla \mathcal{L}_t^*(\vx_t)\right\|^2\right] \\
    \leq& c_\tau\Delta_1 + 27\kappa_gL_{g,1}\sigma^2\sum_{t=2}^T\lambda_t^2 \sum_{j=0}^{t-2}\frac{\rho^j}{B_{t-j}}(\beta_{t-j}\lambda_{t-j} + \alpha_{t-j}) + 72\kappa_g^2L_{f,0}^2c_\lambda \sum_{t=2}^T\sum_{j=0}^{t-2}\frac{\rho_\vy^j}{\lambda_{t-j}^2} \\
    & + \frac{336\kappa_g^2L_{g,1}^2c_\lambda}{1-\rho} \sum_{t=1}^T \gamma_t^2 \mathbb{E}\bigg[\left\|\mathcal{G}_\mathcal{X} \left(\vx_t, \nabla F_t(\vx_t), \gamma_t\right)\right\|^2 + \left\|\nabla F_t(\vx_t) - \nabla \mathcal{L}_t^*(\vx_t)\right\|^2  \\
    & + \left\|\nabla \mathcal{L}_t^*(\vx_t) - \widetilde\nabla \mathcal{L}_t^*(\vx_t)\right\|^2 + \left\|\widetilde\nabla \mathcal{L}_t^*(\vx_t) - \widetilde\nabla \mathcal{L}_t^*(\vx_t;\mathcal{S}_t)\right\|^2 \bigg] + \frac{84L_{g,1}^2c_\lambda}{1-\rho} H_{2,T}.
\end{align*}
Let $\gamma_t \equiv \gamma \leq \frac{1}{8\kappa_gL_{g,1}}\sqrt{\frac{1-\rho}{21c_\lambda}}$, we can gaurantee that $1 - \frac{336\kappa_g^2L_{g,1}^2c_\lambda}{1-\rho} \geq \frac{1}{2}$, thus it holds that
\begin{align}
    &\sum_{t=1}^T \mathbb{E}\left[ \left\|\widetilde\nabla \mathcal{L}_t^*(\vx_t) - \nabla \mathcal{L}_t^*(\vx_t)\right\|^2\right] \nonumber\\
    \leq& 2c_\tau\Delta_1 + 54\kappa_gL_{g,1}\sigma^2 \sum_{t=2}^T\lambda_t^2 \sum_{j=0}^{t-2}\frac{\rho^j}{B_{t-j}}(\beta_{t-j}\lambda_{t-j} {+} \alpha_{t-j}) + 144\kappa_g^2L_{f,0}^2c_\lambda \sum_{t=2}^T\sum_{j=0}^{t-2} \frac{\rho_\vy^j}{\lambda_{t-j}^2} \nonumber\\
    & + \frac{672\gamma^2\kappa_g^2L_{g,1}^2c_\lambda}{1-\rho} \mathbb{E}\left[ \mathrm{BLReg}(T) \right] + \frac{672\gamma^2\kappa_g^2L_{g,1}^2c_\lambda D_1^2}{C_1(1-\rho)} \sum_{t=1}^T \frac{1}{\lambda_t^2} + \frac{2016\gamma^2\kappa_g^2L_{g,1}^2c_\lambda\sigma^2}{1-\rho} \sum_{t=1}^T\frac{\lambda_t^2}{\widehat{B}_t} \nonumber\\ 
    & + \frac{168L_{g,1}^2c_\lambda}{1-\rho} H_{2,T}. \label{E.4_1}
\end{align}
Consider the second term. Letting $m=t-j$, we have
\begin{align*}
    \sum_{t=2}^T\lambda_t^2 \sum_{j=0}^{t-2}\frac{\rho^j}{B_{t-j}}(\beta_{t-j}\lambda_{t-j} + \alpha_{t-j}) = \sum_{m=2}^T \frac{\beta_m\lambda_m + \alpha_m}{B_m}\sum_{t=m}^T\lambda_t^2\rho^{t-m} \leq \frac{\lambda_T^2}{1-\rho} \sum_{t=1}^T \frac{\beta_t\lambda_t + \alpha_t}{B_t}.
\end{align*}
For the third term, we similarly have
\begin{align*}
    \sum_{t=2}^T\sum_{j=0}^{t-2} \frac{\rho_\vy^j}{\lambda_{t-j}^2} = \sum_{m=2}^T\frac{1}{\lambda_m^2}\sum_{t=m}^T\rho_\vy^{t-m} \leq \frac{1}{1-\rho}\sum_{t=1}^T \frac{1}{\lambda_t^2}.
\end{align*}
Substituting the above two inequalities into~\eqref{E.4_1}, we complete the proof by
\begin{align*}
    &\sum_{t=1}^T \mathbb{E}\left[\left\|\widetilde\nabla \mathcal{L}_t^*(\vx_t) - \nabla \mathcal{L}_t^*(\vx_t)\right\|^2\right] \\
    \leq& 2c_\tau\Delta_1 + \frac{54\lambda_T^2\kappa_gL_{g,1}\sigma^2}{1-\rho} \sum_{t=1}^T\frac{\beta_t\lambda_t + \alpha_t}{B_t} + \frac{144\kappa_g^2L_{f,0}^2c_\lambda}{1-\rho} \sum_{t=1}^T\frac{1}{\lambda_t^2} + \frac{168L_{g,1}^2c_\lambda}{1-\rho} H_{2,T} \\
    & + \frac{672\gamma^2\kappa_g^2L_{g,1}^2c_\lambda}{1-\rho} \mathbb{E}\left[ \mathrm{BLReg}(T) \right] + \frac{672\gamma^2\kappa_g^2L_{g,1}^2c_\lambda D_1^2}{1-\rho} \sum_{t=1}^T \frac{1}{\lambda_t^2} + \frac{2016\gamma^2\kappa_g^2L_{g,1}^2c_\lambda\sigma^2}{1-\rho} \sum_{t=1}^T\frac{\lambda_t^2}{\widehat{B}_t}.
\end{align*}
\end{proof}

\begin{theorem}[Restatement of Theorem~\ref{thm:StochasF$^2$OBO}]
    Under Assumption~\ref{asm1},~\ref{asm2} and \ref{asm:stochas}, Let $\Lambda(t, \lambda_t) = \left(1+\frac{1}{t}\right)^{\tau}\lambda_t$ for some $\tau \in (0,\frac{1}{2})$, $\lambda_1 > \frac{2L_{f,1}}{\mu_g}$, $\alpha_t = \frac{1}{L_{g,1}t^a}$, $\beta_t = \frac{1}{2\lambda_1L_{g,1}t^b}$, $\gamma_t \equiv \gamma \leq \min\{\frac{1}{2L_F}, \frac{1}{48\kappa_gL_{g,1}}\sqrt{\frac{1-\rho}{7c_\lambda}}\}$, $K_t = 4\kappa_g\max\{t^{b-\tau}, t^a\}\log t$, $\widehat{B}_t=t^{B_1}$ and $B_t=t^{B_2}$, Algorithm~\ref{alg:StochasF$^2$OBO} can guarantee
    \begin{align*}
        \mathbb{E}\left[\mathrm{BLReg}(T)\right] \leq O\left(T^{1-2\tau} + \sigma^2\max\{T^{1+2\tau-B_1}, T^{1+3\tau-b-B_2}, T^{1+2\tau-a-B_2}\} + V_T + H_{2,T}\right),
    \end{align*}
    where $\rho$, $c_\lambda$ are some constants.
\end{theorem}
\begin{proof}
Substituting Lemma~\ref{lem:stochas_nabla_L} into Lemma~\ref{lem:stochas_begin}, we obtain
\begin{align*}
    &\mathbb{E}\left[\mathrm{BLReg}(T)\right] \\
    \leq& \sum_{t=1}^T \frac{8}{\gamma_t}\left(F_t(\vx_t) - F_t(\vx_{t+1})\right) + 24D_1^2\sum_{t=1}^T\frac{1}{\lambda_t^2} + 72\sigma^2\sum_{t=1}^T\frac{\lambda_t^2}{\widehat{B}_t} + \frac{3456\kappa_g^2L_{f,0}^2c_\lambda}{1-\rho} \sum_{t=1}^T \frac{1}{\lambda_t^2} \\
    & + 48c_\tau\Delta_1 + \frac{1296\lambda_T^2\kappa_gL_{g,1}\sigma^2}{1-\rho} \sum_{t=1}^T\frac{\beta_t\lambda_t {+} \alpha_t}{B_t} + \frac{16128\gamma^2\kappa_g^2L_{g,1}^2c_\lambda}{1-\rho} \mathbb{E}\left[ \mathrm{BLReg}(T) \right] \\
    & + \frac{16128\gamma^2\kappa_g^2L_{g,1}^2c_\lambda D_1^2}{1-\rho} \sum_{t=1}^T \frac{1}{\lambda_t^2} + \frac{48384\gamma^2\kappa_g^2L_{g,1}^2c_\lambda\sigma^2}{1-\rho} \sum_{t=1}^T\frac{\lambda_t^2}{\widehat{B}_t} + \frac{4032L_{g,1}^2c_\lambda}{1-\rho} H_{2,T}.
\end{align*}
Choosing $\gamma \leq \frac{1}{48\kappa_gL_{g,1}}
\sqrt{\frac{1-\rho}{7c_\lambda}}$ ensures that $1 - \frac{4032\gamma^2\kappa_g^2L_{g,1}^2c_\lambda}{1-\rho}
\geq \frac{1}{2}$. Moreover, with
\begin{align*}
    \lambda_t = \lambda_1t^\tau, \quad \tau\in\left(0, \frac{1}{2}\right), \quad 
    \beta_t = \frac{1}{2\lambda_1L_{g,1}t^b}, \quad 
    \alpha_t = \frac{1}{L_{g,1}t^a}, \quad 
    \widehat{B}_t = t^{B_1}, \quad 
    B_t = t^{B_2},
\end{align*}
we finally have
\begin{align*}
    &\mathbb{E}\left[\mathrm{BLReg}(T)\right] \\\leq& \frac{32M}{\gamma} + \frac{16}{\gamma}V_T + \frac{48D_1^2}{\lambda_1^2}\sum_{t=1}^Tt^{-2\tau} + 144\lambda_1^2\sigma^2 \sum_{t=1}^T t^{2\tau - B_1} + \frac{6912\kappa_g^2L_{f,0}^2c_\lambda}{\lambda_1^2(1-\rho)} \sum_{t=1}^Tt^{-2\tau} \\
    & + \frac{2592\lambda_1^2\kappa_g\sigma^2T^{2\tau}}{1-\rho} \sum_{t=1}^T (t^{\tau-b-B_2} + t^{-a-B_2}) + \frac{32256\gamma^2\kappa_g^2L_{g,1}^2c_\lambda D_1^2}{\lambda_1^2(1-\rho)} \sum_{t=1}^T t^{-2\tau} \\
    & + 96c_\tau\Delta_1 + \frac{96768\gamma^2\lambda_1^2\kappa_g^2L_{g,1}^2c_\lambda\sigma^2}{1-\rho} \sum_{t=1}^T t^{2\tau-B_1} + \frac{8064L_{g,1}^2c_\lambda}{1-\rho} H_{2,T} \\
    =& O\left(T^{1-2\tau} + \sigma^2\max\{T^{1+2\tau-B_1}, T^{1+3\tau-b-B_2}, T^{1+2\tau-a-B_2}\} + V_T + H_{2,T}\right),
\end{align*}
which completes the proof.
\end{proof}

\end{document}